\documentclass{article}

    \PassOptionsToPackage{numbers}{natbib}

    \usepackage[final]{neurips_2024}

\usepackage[utf8]{inputenc} 
\usepackage[T1]{fontenc}
\usepackage{hyperref}       
\usepackage{url}            
\usepackage{booktabs}       
\usepackage{nicefrac}       
\usepackage{microtype}      
\usepackage{xcolor}         
\usepackage{enumitem}
\usepackage{wrapfig}

\usepackage{mymacros}

\usepackage{tikz}
\usetikzlibrary{arrows,automata}
\usetikzlibrary{positioning}
\usetikzlibrary{arrows.meta,calc,decorations.markings,math}

\title{How does Inverse RL Scale to Large
State Spaces? \\A Provably Efficient Approach}

\author{%
Filippo Lazzati\\
Politecnico di Milano\\
Milan, Italy\\
\texttt{filippo.lazzati@polimi.it} \\
\And
Mirco Mutti \\
Technion \\
Haifa, Israel\\
\And
Alberto Maria Metelli \\
Politecnico di Milano\\
Milan, Italy\\
}

\usepackage[ruled,linesnumbered]{algorithm2e}
\usepackage{subcaption}
\SetKwComment{Comment}{/* }{ */} 
\allowdisplaybreaks[4]

\let\oldnl\nl
\newcommand{\nonl}{\renewcommand{\nl}{\let\nl\oldnl}}

\begin{document}
\setlength{\abovedisplayskip}{4pt}
\setlength{\belowdisplayskip}{4pt}
\setlength{\textfloatsep}{10pt}

\renewcommand\thmcontinues[1]{Continued}

\maketitle

\begin{abstract}
    In online Inverse Reinforcement Learning (IRL), the learner can
    collect samples about the dynamics of the environment to improve its
    estimate of the reward function. Since IRL suffers from identifiability issues,
    many theoretical works on online IRL focus on estimating the entire set of
    rewards that explain the demonstrations, named the \emph{feasible reward
    set}. However, none of the algorithms available in the literature can scale to
    problems with large state spaces.
    In this paper, we focus on the online IRL problem in Linear Markov Decision
    Processes (MDPs). We show that the structure offered by Linear MDPs is not
    sufficient for efficiently estimating the feasible set when the state space
    is large. As a consequence, we introduce the novel framework of
    \emph{rewards compatibility}, which generalizes the notion of feasible set,
    and we develop \caty, a sample efficient algorithm whose complexity is independent of the cardinality of
    the state space in Linear MDPs.
    When restricted to the tabular setting, we demonstrate that \caty is minimax
    optimal up to logarithmic factors. As a by-product, we show that Reward-Free
    Exploration (RFE) enjoys the same worst-case rate, improving over the
    state-of-the-art lower bound. Finally, we devise a unifying framework for
    IRL and RFE that may be of independent interest.
\end{abstract}

\section{Introduction}\label{section: introduction}

Inverse Reinforcement Learning (IRL) is the problem of inferring the reward
function given demonstrations of an optimal behavior, i.e., from an \emph{expert} agent.
\cite{russell1998learning,ng2000algorithms}. Since its formulation, much of the
research effort has been put into the design of efficient algorithms for solving
the IRL problem \cite{arora2018survey,adams2022survey}. Indeed, the solution of
the IRL problem opens the door to a variety of interesting applications,
including Apprenticeship Learning (AL)
\cite{abbeel2004apprenticeship,abbeel2006helicopter}, reward design
\cite{hadfieldmenell2017inverserewarddesign}, interpretability of the expert's behavior
\cite{hadfieldmenell2016cooperativeIRL}, and transferability to new environments
\cite{Fu2017LearningRR}.

Nowadays, the factor that most negatively impacts the adoption of IRL
solutions in real-world applications is the intrinsic \emph{ill-posedness} of
its formulation. The IRL problem has been historically defined as the problem of
recovering \emph{the} reward function underlying the demonstrations
\cite{russell1998learning,ng2000algorithms}, even though mere demonstrations can
be equivalently explained by a \emph{variety} of rewards. In other words, the
IRL problem is underconstrained, even in the limit of infinite demonstrations
\cite{ng2000algorithms,metelli2021provably}.

To overcome this weakness and to come up with a \emph{single} reward
function, three main approaches are commonly adopted in the literature.
($i$) The first approach consists of the use of a \emph{heuristic} to select a
specific reward function from the set of all the rewards that explain the
demonstrations. Implicitly, these works re-define IRL as the problem of
recovering \emph{the} reward function explaining the demonstrations \emph{and}
complying with the heuristic. As an example,
\cite{ng2000algorithms,ratliff2006mmp} select the reward that maximizes some
notion of margin, and \cite{ziebart2008maximum} implicitly chooses the reward
returned by the optimization algorithm among those that maximize the likelihood.
However, these approaches may generate issues in applications
\cite{skalse2023invariance,Fu2017LearningRR}.
($ii$) In the second approach, additional \emph{constraints} beyond mere demonstrations
are enforced to guarantee the uniqueness of the reward function to recover. In
``reward identifiability'' works, the additional information commonly concerns some
structure of the environment \cite{kim2021rewardidentification}, or multiple
demonstrations across various environments
\cite{amin2016resolving,cao2021identifiability}. In Reward Learning
(ReL) works \cite{jeon2020rewardrational}, demonstrations of optimal behavior
are combined with other kinds of expert feedback, like comparisons
\cite{wirth2017surveyPbRL}.
($iii$) As a third approach, recently, \cite{metelli2021provably,metelli2023towards}
proposed the alternative formulation of IRL as the problem of recovering
\emph{all} the reward functions compatible with the demonstrations, i.e., the
\emph{feasible reward set}. In this manner, we are not
subject to the limitations of the first approach, and we do not depend on
additional information like in the second approach.

In practical applications, the chosen IRL formulation has to be tackled by
algorithms that use a \emph{finite} number of demonstrations and a limited
knowledge of the dynamics of the environment. In the common \emph{online} IRL
scenario, the learner explores the (unknown) environment, and exploits this
additional information to improve its performance on the IRL task
\citep[e.g.,][]{metelli2021provably,lindner2022active,metelli2023towards,zhao2023inverse,lazzati2024offline}.
On this basis, the IRL approach ($iii$) based on the \emph{feasible set}
\cite{metelli2021provably,metelli2023towards} displays desirable properties since
``postpones'' the choice of the heuristic and/or enforcement of
additional constraints, with the advantage of
analyzing the intrinsic complexity of the IRL problem only, without being
obfuscated by other factors. In other words, this recent formulation of the
IRL problem paves the way for the design and analysis of provably efficient IRL
algorithms, endowed with solid theoretical guarantees.

However, the algorithms designed for learning the feasible set currently
available in the literature
\citep[e.g.,][]{metelli2021provably,lindner2022active,metelli2023towards,zhao2023inverse,lazzati2024offline}
struggle when attempting to scale them to IRL problems with \emph{large state
spaces}. This is apparent because their sample complexity exhibits an explicit
dependence on the cardinality of the state space. This inevitably represents a
major limitation since most real-world scenarios concern problems with large, or
even continuous, state spaces
\cite[e.g.,][]{Fu2017LearningRR,barnes2024massively,michini2013scalable,finn2016guided}.

In this context, function approximation represents an essential tool to tackle
the curse of dimensionality and enforce generalization
\cite{silver2016go,mnih2013playing}. Linear Markov Decision Processes (MDPs)
\cite{jin2020provablyefficient,yang2019sampleoptimal} offer a simple but
powerful structure, in which we assume the reward function and the transition
model can be expressed as linear combinations of known features, that permits
theoretical analysis of the sample complexity. Even though many extensions have
been developed \cite{wang2020general,jin2021eluder,du2021bilinear}, the Linear
MDPs framework typically represents one of the first function approximation
settings to analyze when focusing on a novel problem, before moving to more
complex settings \cite[e.g.,][]{wang2020onrewardfree,viano2024imitation}.

In this paper, we aim to shed light on the challenges of scaling the feasible
reward set to large-scale problems. Motivated by its limitations when dealing
with large state spaces, we introduce the novel \emph{Rewards Compatibility}
framework. Being a generalization of the notion of feasible set, it allows us to
define the new \emph{IRL Classification Problem}, a fourth approach to cope with
the ill-posedness of the IRL formulation. This permits the development of \caty
(\catylong), a provably efficient IRL algorithm for Linear MDPs characterized by
large or even continuous state spaces.

\textbf{Original Contributions.}~~The main contributions of the current work
can be summarized as follows:
\begin{itemize}[leftmargin=*, noitemsep, topsep=-2pt]
    \item We prove that the notion of feasible set can \emph{not} be learned
    efficiently in MDPs with large/continuous state spaces, even under the
    structure enforced by Linear MDPs. Nevertheless, we show that this problem
    disappears under the \emph{assumption} that the expert's policy is known, by
    providing a sample efficient algorithm for such setting (Section
    \ref{section: extension framework past works}).
    \item To overcome the need for knowing the expert's policy exactly, we propose \emph{Rewards
    Compatibility}, a novel framework that formalizes the intuitive notion of
    \emph{compatibility} of a reward function with expert demonstrations. It
    generalizes the feasible set and allows us to define an original
    learning setting, \emph{IRL classification}, based on a new formulation of IRL \emph{classification} task (Section \ref{section: rewards compatibility}).
    \item For the newly-devised framework, we develop \caty (\catylong), a new
    sample and computationally efficient IRL algorithm for both
    tabular and Linear MDPs. Remarkably, this \caty does not require the
    additional assumption that the expert's policy is known (Section
    \ref{section: finally linear mdps}).
    \item In the tabular setting, we prove a tight minimax lower
    bound to the sample complexity of the  IRL classification problem of
    $\Omega\big(\frac{H^3SA}{\epsilon^2}(S+\log\frac{1}{\delta})\big)$ episodes,
    where $S$ and $A$ are the cardinalities of the state and action spaces, $H$ is the
    horizon, $\epsilon$ the accuracy and $\delta$ the failure probability. This
    bound is \emph{matched} by \caty, up to logarithmic factors. Exploiting a similar construction, we show that
    a lower bound with the same rate holds also for the Reward-Free Exploration
    (RFE) problem, improving by an $H$ factor over the RFE
    state-of-the-art lower bound \cite{jin2020RFE} (Section
    \ref{section: insights rfe and irl}).
    \item Finally, we formulate a novel \emph{Objective-Free Exploration} (OFE)
    setting that isolates the challenges of exploration beyond
    Reinforcement Learning (RL), by unifying RFE and IRL (Section
    \ref{section: ofe}).
\end{itemize}
Additional related works and the proofs of all the results are reported in Appendix \ref{section: additional
related works} and \ref{section: remove assumption expert policy}
-\ref{section: more lower bound}.

\section{Preliminaries}\label{section: preliminaries}


\textbf{Notation.}~~Given an integer $N \in \Nat$, we define
$\dsb{N}\coloneqq\{1,\dotsc,N\}$.
Given sets $\X$ and $\Y$, we denote $\mathcal{H}_d(\X,\Y)\coloneqq\max\{
\sup_{x\in\X}\inf_{y\in\Y}d(x,y),\sup_{y\in\Y}\inf_{x\in\X}d(y,x)\}$ their Hausdorff distance with inner distance $d$. We denote by $\Delta^\X$ the probability simplex over $\X$, and
by $\Delta_\Y^\X$ the set of functions from $\Y$ to $\Delta^\X$. Sometimes, we
denote the dot product between vectors $x,y$ as $\dotp{x,y}\coloneqq x^\intercal
y$. We employ $\mathcal{O},\Omega,\Theta$ for the common asymptotic
notation and
$\widetilde{\mathcal{O}},\widetilde{\Omega},\widetilde{\Theta}$ to omit
logarithmic terms.

\textbf{Markov Decision Processes.}~~A finite-horizon Markov Decision Process (MDP) without reward
\cite{puterman1994markov} is defined as a tuple $\M\coloneqq\tuple{\S,\A,H,
d_0,p}$, where $\S$ and $\A$ are the measurable state and action spaces, $H \in \Nat$ is
the horizon, $d_0\in\Delta^\S$ is the initial-state distribution, and
$p\in\P\coloneqq\Delta_{\SAH}^\S$ is the transition model. Given a (deterministic) reward
function $r\in\mathfrak{R}\coloneqq[-1,1]^{\SAH}$, we denote by $\overline{\M}\coloneqq\M\cup\{r\}$ the
MDP obtained by pairing $\M$ and $r$.
Each policy $\pi\in\Pi\coloneqq\Delta_{\SH}^\A$ induces in $\overline{\M}$ a state-action
probability distribution $d^{p,\pi}\coloneqq\{d^{p,\pi}_h\}_{h\in\dsb{H}}$ (we
omit $d_0$ for simplicity) that assigns, to each subset $\Z\subseteq\S\times\A$, the
probability of being in $\Z$ at stage $h \in \dsb{H}$ when playing $\pi$ in
$\overline{\M}$. We denote with $\S^{p,\pi}_h$ the set of states supported by
$d^{p,\pi}_h$ for any action at stage $h$, and with $\S^{p,\pi}$ the disjoint
union of sets $\{\S^{p,\pi}_h\}_{h \in \dsb{H}}$.
The $Q$-function of policy $\pi$ in MDP $\overline{\M}$ is defined at every
$(s,a,h) \in \S \times \A \times \dsb{H}$ as $Q^{\pi}_h(s,a;p,r)\coloneqq
\E_{p,\pi}[\sum_{t=h}^H r_{t}(s_t,a_t)|s_h=s,a_h=a]$, and the optimal
$Q$-function as $Q^*_h(s,a;p,r)\coloneqq\sup_{\pi \in \Pi} Q^{\pi}_h(s,a;p,r)$,
where the expectation $\E_{p,\pi}$ is computed over the stochastic process
generated by playing policy $\pi$ in the MDP $\overline{\M}$. Similarly, we
define the $V$-function of policy $\pi$ at $(s,h)$ as $V^\pi_h(s;p,r)\coloneqq
\E_{p,\pi}[\sum_{t=h}^H r_{t}(s_t,a_t)|s_h=s]$, and the optimal $V$-function as
$V^*_h(s;p,r)\coloneqq\sup_{\pi \in \Pi} V^{\pi}_h(s;p,r)$. We define the
utility of $\pi$ as $J^\pi(r;p)\coloneqq \E_{s\sim d_0}[V^\pi_1(s;p,r)]$, and
the optimal utility as $J^*(r;p)\coloneqq \E_{s\sim d_0}[V^*_1(s;p,r)]$.
A forward (sampling) model of the environment permits to collect samples
starting from $s\sim d_0$ and following some policy. A generative (sampling)
model consists in an oracle that, given an arbitrary state-action-stage triple
$s,a,h$ in input, returns a sampled next state $s'\sim p_h(\cdot|s,a)$.

\textbf{Linear MDPs.}~~
Based on \cite{jin2020provablyefficient}, we say that
an MDP $\overline{\M}=\tuple{\S,\A,H, d_0,p,r}$ is a \emph{Linear MDP} with
a (known) feature map $\phi:\SA\rightarrow  \RR^d$, if for every $h\in\dsb{H}$,
there exist $d \in \Nat$ unknown (signed) measures $\mu_h=[\mu_h^1,\dotsc,\mu_h^d]^\intercal$
over $\S$ and an unknown vector $\theta_h\in \RR^d$, such that for every
$(s,a)\in\SA$, we have $p_h(\cdot|s,a)=\dotp{\phi(s,a),\mu_h(\cdot)}$ and
$r_h(s,a)=\dotp{\phi(s,a),\theta_h}$. Without loss of generality, we assume
$\|\phi(s,a)\|_2\le 1$ for all $(s,a)\in\SA$, and
$\max\{\|\theta_h\|_2,\||\mu_h|(\S)\|_2\}\le\sqrt{d}$.\footnote{$|\mu_h|(\mathcal{B})$
denotes the vector containing the variation of each measure $\mu_h^i$ over the measurable set $\mathcal{B}$.}
$\M$ is a \emph{Linear MDP without reward} if its transition
model satisfies the assumption described above.

\textbf{BPI and RFE.}~~In both Best-Policy Identification (BPI) \cite{menard2021fast} and
Reward-Free Exploration (RFE) \cite{jin2020RFE}, the learner has to explore the \emph{unknown} MDP to
optimize a certain reward function. In BPI, the learner observes the reward
function $r$ during exploration, and its goal is to output a policy
$\widehat{\pi}$ such that, in the true MDP with transition model $p$ we have 
$\mathbb{P}\big(J^*(r;p)-J^{\widehat{\pi}}(r;p)\le\epsilon\big)\ge 1-\delta$ for every
$\epsilon,\delta\in(0,1)$.
RFE considers the setting in which the reward to optimize is revealed \emph{a
posteriori} of the exploration phase. Thus
the goal of the agent in RFE is to compute an estimate $\widehat{p}$ of the true dynamics $p$
so that $\mathbb{P}\big(\sup_{r\in\mathfrak{R}} \{J^*(r;p)-
J^{\widehat{\pi}_r}(r;p)\}\le\epsilon\big)\ge 1-\delta$ for every $\epsilon,\delta\in(0,1)$, where $\widehat{\pi}_r$
is the optimal policy in the MDP with $\widehat{p}$ as transition model and $r$
as reward function.

\textbf{Online IRL.}~~We consider the online\footnote{``Online'' refers to how we
estimate the transition model $p$, not to the expert's policy $\pie$, for
which we assume to have access to a batch dataset. This is justified by the fact that most of IRL real-world applications involve the presence of a fixed dataset of expert demonstrations previously collected and the agent can explore the environment in order to reconstruct one (or more) reward functions compatible with those demonstrations.} IRL setting
\cite{metelli2021provably,lindner2022active,zhao2023inverse,Xu2023ProvablyEA,Shani2021OnlineAL}
in which, similarly to the online AL setting
\cite{Shani2021OnlineAL,Xu2023ProvablyEA}, we are given a dataset
$\D^E=\{\tuple{s_1^i,a_1^i,\dotsc,s_{H-1}^i,a_{H-1}^i,s_H^i}\}_{i\in\dsb{\tau^E}}$
of $\tau^E \in \Nat$ trajectories collected by executing the expert's policy
$\pie$ in a certain (unknown) MDP $\overline{\M}=\M\cup\{r^E\}$. 
We make the assumption that $\pie$ is optimal under the true (unknown) reward
$r^E$ in $\overline{\M}$. Since the dynamics of $\overline{\M}$ is unknown, we
are allowed to actively explore the environment through a \emph{forward} model
to collect a new state-action dataset $\D$. The goal is to use the latter and demonstrations in
$\D^E$ to estimate a reward function that makes the expert's
policy $\pie$ optimal. Sometimes, we will denote an IRL instance as
$\M\cup\{\pie\}$, and a Linear IRL instance with recovered reward $r$ as an IRL
instance in which $\M\cup\{r\}$ is a Linear MDP. 

\section{Limitations of the Feasible Set}\label{section: extension framework past works}

In this section, after having characterized the feasible set formulation in Linear MDPs, we
show that it
suffers from \emph{statistical} (and \emph{computational}) inefficiency in
problems with large state spaces, even under the Linear MDP assumption. We will
provide a solution to these issues in Section \ref{section: rewards
compatibility}.

\textbf{The Feasible Set.}~~According to the standard definition~\cite[e.g.,][]{metelli2021provably,lindner2022active,metelli2023towards,zhao2023inverse,lazzati2024offline}, the feasible set contains the rewards that make the
expert's policy $\pie$ optimal, as defined below.

\begin{defi}[Feasible Set \cite{lazzati2024offline}]\label{def: fs
implicit}
Let $\M$ be an MDP without reward and let $\pie$ be the
expert's policy. The \emph{feasible set} $\R_{p,\pie}$ of rewards compatible
with $\pie$ in $\M$ is defined as: {
\begin{equation*}
    \R_{p,\pie}\coloneqq\{
        r\in\mathfrak{R}\,|\,J^{\pie}(r;p)=J^*(r;p)
    \}.
\end{equation*}}%
\end{defi}

Without function approximation, the feasible set contains a variety of rewards
for any deterministic policy. In Linear MDPs, due to the feature map, the
feasible set might exhibit some degeneracy.\footnote{We exemplify this
proposition in Appendix \ref{appendix: examples degenerate fs}. In Appendix
\ref{section: missing proofs} we generalize to infinite state spaces.}
Definition \ref{def: fs implicit} can be adapted to Linear MDPs with feature map
$\phi$ as: $\R_{\phi,p,\pie}\coloneqq\{
r\in\mathfrak{R}\,|\,J^{\pie}(r;p)=J^*(r;p) \wedge \exists
\theta:\dsb{H}\to\RR^d,\forall (s,a,h)\in\SAH:\,
r_h(s,a)=\dotp{\phi(s,a),\theta_h}\}$. We omit $\phi$ in $\R_{\phi,p,\pie}$ for
notational simplicity.

\begin{restatable}{prop}{degeneratefs}\label{prop: degeneratefs} Let $\M$ be a Linear
MDP without reward with a finite state space, and let $\phi$ be a feature
mapping. Let $\{\Phi_h^{\pie}\}_{h\in \dsb{H}}$ and $\{\overline{\Phi}_h\}_{h\in \dsb{H}}$ be the sets of
expert's and non-expert's features, defined for every $h \in \dsb{H}$ as:
\begin{align*}
    \Phi_h^{\pie}&\coloneqq
    \big\{\phi(s,a^E) \,|\, s\in\S^{p,\pie}_h,\,
     a^E\in\A^E_h(s)\big\},\qquad
    \overline{\Phi}_h\coloneqq
    \big\{\phi(s,a) \,|\, s\in\S^{p,\pie}_h,\, a\in\A\setminus\A^E_h(s)
    \big\},
\end{align*}
where $\A^E_h(s)\coloneqq\{a\in\A|\pi^E_h(\cdot|s)>0\}$ for every $s \in \S$. If
for none of the $H$ pairs of sets $\tuple{\Phi_h^{\pie},\overline{\Phi}_h}$
there exists a separating hyperplane, then $\R_{p,\pie}=\{\overline{r}\}$, with
$\overline{r}_h(s,a)=0$ $\forall (s,a,h)\in\SAH$ i.e., the feasible set with linear rewards
in $\phi$ contains only the reward function that assigns zero reward everywhere.
\end{restatable}
Intuitively, expert's actions must have the largest optimal $Q$-value among all
actions, and linearity imposes the ``separability'' requirement. The result
holds also for MDPs with linear rewards only. We exemplify Proposition
\ref{prop: degeneratefs} in Appendix \ref{appendix: examples degenerate fs}.

\textbf{Learning the Feasible Set.}~~In order to highlight the challenges of
learning the feasible set with large-scale MDPs, based on
\cite{metelli2023towards,lazzati2024offline}, we devise the following PAC
requirement.
\begin{defi}[PAC Algorithm]\label{def: pac framework bad} Let
$\epsilon,\delta\in(0,1)$, and let $\mathfrak{A}$ be an algorithm that
collects $\tau^E$ samples about $\pie$ using a generative model, and $\tau$
episodes from a Linear MDP without reward $\M=\tuple{\S,\A,H,
d_0,p}$ using a forward model. Let $\widehat{\R}$ be the estimate of the
feasible set $\R_{p,\pie}$ outputted by $\mathfrak{A}$. Then, $\mathfrak{A}$ is
$(\epsilon,\delta)$\emph{-PAC} for IRL if $ \mathbb{P}_{\M,\mathfrak{A}}
        \big(\mathcal{H}_d(\R_{p,\pie},\widehat{\R})\le\epsilon\big)
        \ge 1-\delta$, 
    where $\mathbb{P}_{\M,\mathfrak{A}}$ is the probability measure induced
    by $\mathfrak{A}$ in $\M$, and
    $d(r,\widehat{r})\propto\sup_{\pi\in\Pi}\sum_{h\in\dsb{H}} \E_{(s,a)\sim
    d^{p,\pi}_h(\cdot,\cdot)}|r_h(s,a)-\widehat{r}_h(s,a)|$.\footnote{For simplicity, we provide the full
    expression of distance $d$ in Appendix \ref{section: missing proofs}, Equation~\eqref{eq: inner distance d}.}
    The \emph{sample complexity} is the pair $(\tau^E,\tau)$.
\end{defi}
It is worth noting that in Definition~\ref{def: pac framework bad}, we are considering a generative model for collecting samples from the expert's policy, which represents the easiest learning scenario.
The following result shows that, even in this convenient setting, estimating the feasible set is statistically inefficient.
\begin{restatable}[Statistical Inefficiency]{thr}{propnoass}\label{prop: fs not learnable 0 assumptions}
    Let $\M\cup\{\pie\}$ be a Linear IRL instance with finite state space
    $\S$ and deterministic expert's policy, and let $\epsilon,\delta\in(0,1)$.
    If an algorithm $\mathfrak{A}$ is $(\epsilon,\delta)$-PAC, then
    $\tau^E=\Omega(S)$, where $S\coloneqq|\S|$ is the cardinality of the state space.
\end{restatable}
In other words, even under the easiest learning conditions (i.e., generative
model and deterministic expert), the sample complexity scales directly with the
cardinality of the state space $S$, thus, it is infeasible when $S$ is large or even
infinite. Observe that this result extends to any class of MDPs that contains
Linear MDPs. In Appendix \ref{section: additional regularity}, we analyze if
additional assumptions can drop the $\Omega(S)$ dependence.
Nevertheless, if $\pie$ is \textit{known}, it is
possible to construct sample efficient algorithms. Algorithm~\ref{alg: known pie} (whose pseudocode is presented in Appendix~\ref{section: algorithm linear fs appendix}), under the assumption that $\pie$ is known, makes use of an inner RFE routine (Algorithm 1 of
    \cite{wagenmaker2022noharder}) to recover the feasible set.

\begin{restatable}{thr}{upperboundlinearoldframework}\label{thr: upper bound
linear old framework} Assume that $\pie$ (along with its support
$\S^{p,\pie}$) is known. Then, for any $\epsilon,\delta\in(0,1)$, Algorithm
\ref{alg: known pie} is $(\epsilon,\delta)$-PAC for IRL with a number of
episodes $\tau$ upper bounded by:
\begin{align*}
    \tau\le\widetilde{\mathcal{O}}\Big(
            \frac{H^5d}{\epsilon^2}\Big(d+\log\frac{1}{\delta}\Big)    
        \Big).
\end{align*}
\end{restatable}

\textbf{Limitations of the Feasible Set.}~~We can now conclude that the feasible
set suffers from two main limitations. ($i$) \emph{Sample Inefficiency}: If
$\pie$ is unknown, it requires a number of samples that depends on the
cardinality of the state space (Theorem~\ref{prop: fs not learnable 0
assumptions}). ($ii$) \emph{Lack of Practical Implementability}: It contains a
continuum of rewards, thus, no practical algorithm can explicitly compute it. We
will discuss in the next section how to overcome both these issues.

\section{Rewards Compatibility}\label{section: rewards compatibility}

In this section, we present the main contribution of this work: \emph{Rewards
Compatibility}, a novel framework for IRL that allows us to conveniently rephrase the
learning from demonstrations problem as a classification task. We
anticipate that the presentation of the framework is completely general and
independent of structural assumptions of the MDP (e.g., Linear MDP).

\subsection{Compatible Rewards}

In the following, for ease of presentation, we consider the exact setting, i.e., when
$d_0$, $p$, and $\pie$ are known.
In addition, we will drop the dependence on $p$ when clear from the context.

In IRL, an expert agent demonstrates policy $\pie$ assumed
optimal under some (unknown) reward function $r^E$, i.e.,
$J^*(r^E)=J^{\pie}(r^E)$ . The task is to recover a reward $r$ such that
$J^*(r)=J^{\pie}(r)$.
By definition, IRL tells us that $r^E$ makes the demonstrated policy $\pie$
optimal, but what about other policies? We \emph{do not} and \emph{cannot}
know. Since there are (infinite) rewards making $\pie$ optimal (but they
differ in the performance attributed to other policies) we realize that there are
many rewards equally ``compatible'' with $\pie$.\footnote{See Appendix
\ref{section: visual explanation} for a visual intuition.} Clearly, wih no
additional information, we are unable to identify $r^E$.

The feasible set considers only these rewards, i.e., $r \in \mathfrak{R}$ for which
$J^*(r)=J^{\pie}(r)$, and it refuses all the others. This can be interpreted as
the feasible set carrying out a \emph{classification} of rewards based on 
a ``hard'' notion of \emph{compatibility} with demonstrations. In other words,
rewards $r$ satisfying condition $J^*(r)=J^{\pie}(r)$ are compatible with
$\pie$, and the others are not.
Nevertheless, our insight is that
some rewards are \emph{``more'' compatible} with $\pie$ than
others.

\begin{example}[label=exa:cont]\label{example: muffin cake soup} Consider an MDP with
one state and  $H=1$ in which the expert has three actions:
Eating a muffin (M), a cake (C), or some (bad) vegetable soup (S).
The true reward $r^E$ assigns $r^E(M)=+1,r^E(C)=+0.99$ and $r^E(S)=-1$, i.e., the
expert has a (weak) preference for the muffin over the cake, while she hates the
soup; thus, she will demonstrate $\pie=M$. Let $r_g,r_b$ be:
\begin{align*}
    r_g(M)=+0.99,r_g(C)=+1,r_g(S)=-1,\qquad r_b(M)=-1,r_b(C)=-1,r_b(S)=+1.
\end{align*}
Intuitively, $r_g$ is \emph{``more'' compatible} with $\pie$ than $r_b$, because it
establishes that M and C are much better than S, while reward $r_b$ reverses
the preferences. Clearly, we make a small error if we model the preferences of
the expert with $r_g$ instead of the true reward $r^E$. However, the notion of
feasible set is completely blind to the difference between $r_g$ and $r_b$ at
modeling $r^E$, and it refuses both of them.
\end{example}

We propose the following ``soft'' definition of (non)compatibility to capture
this intuition.\footnote{In Appendix \ref{section: multiplicative alternative},
a \emph{multiplicative} alternative definition is presented.}
\begin{defi}[Rewards (non)Compatibility]\label{def: reward compatibility} Let
    $\M\cup\{\pie\}$ be an IRL instance, and let $r \in \mathfrak{R}$ be
    any reward. We define the \emph{(non)compatibility}
    $\overline{\C}_{p,\pie}:\mathfrak{R}\to  \RR_{\ge 0}$ of reward $r$ w.r.t.
    $\M\cup\{\pie\}$ as:
    \begin{align*}
        \overline{\C}_{p,\pie}(r)\coloneqq
        J^*(r;p) - J^{\pie}(r;p).
    \end{align*}
    \end{defi}
In words, the (non)compatibility of reward $r$ w.r.t. policy $\pie$ in problem
$\M$ quantifies the \textit{sub-optimality} of $\pie$ in the MDP $\M\cup\{r\}$.
By definition, rewards $r$ belonging to the feasible set (i.e., $r\in\R_{p,\pie}$) satisfy
$\overline{\C}_{p,\pie}(r)=0$, i.e., they have zero non-compatibility
with $\pie$ in $\M$.\footnote{We use \emph{(non)compatibility} since a reward $r \in \mathfrak{R}$ is maximally compatible when $\overline{\C}_{p,\pie}(r)=0$. Thus, the larger $\overline{\C}_{p,\pie}(r)$, the more $r$ is non-compatible. In this sense, $\overline{\C}_{p,\pie}(r)$ quantifies the non-compatibility of $r$.}

\begin{example}[continues=exa:cont]
    (Non)compatibility discriminates
between $r_g$ and $r_b$. 
    Indeed, we have that $\overline{\C}_{p,\pie}(r^E)=0$,
    $\overline{\C}_{p,\pie}(r_g)=0.01$, and
    $\overline{\C}_{p,\pie}(r_b)=2$. In words, reward $r_g$ suffers from
    very small (non)compatibility, while $r_b$ suffers from large
    (non)compatibility, thus we say that reward $r_g$ is more compatible with
    $\pie$ than $r_b$, as expected.
\end{example}

By definition of IRL, the true reward $r^E$ makes the observed $\pie$ optimal,
but reveals no information about the other policies. Thus, it is meaningful that
$\overline{\C}_{p,\pie}$ considers the suboptimality of $\pie$ only, because
demonstrations from $\pi^E$ do not provide information about other policies, as illustrated below.

\begin{example}
    Let $r_b'$ be such that $r_b'(M)=+0.99,r_b'(C)=-1,r_b'(S)=+1$. Clearly, $r_b'$ is much
    worse than $r_g$ at modeling $r^E$, because it does not capture the fact
    that the expert appreciates the cake but she hates the soup. However,
    demonstrations from $\pi^E$ alone do not provide information about C or S, but only about
    $\pie=M$ (i.e., the expert always eats the muffin). Thus, we
    have that $\overline{\C}_{p,\pie}(r_g)=\overline{\C}_{p,\pie}(r_b')=0.01$,
    i.e., $r_g$ and $r_b'$ are equally compatible with the given demonstrations.
\end{example}

For a discussion on comparing the (non)compatibility of different rewards, see
Appendix \ref{apx: compatibility comparison}.

\subsection{The IRL Classification Formulation}

Our goal is to overcome the limitations of the feasible set highlighted in
Section \ref{section: extension framework past works}. Drawing inspiration from
the notion of ``membership checker'' algorithm in \cite{lazzati2024offline}, we
propose a novel formulation of IRL.

\begin{defi}[IRL Classification Problem and IRL Algorithm]\label{def: IRL
    problem} An \emph{IRL Classification Problem} instance is made of a tuple
    $\tuple{\M,\pie,\R,\Delta}$, where $\M$ is an MDP without reward, $\pie$
    is the expert's policy, $\R\subseteq\mathfrak{R}$ is a set of rewards to
    classify, and $\Delta\in\RR_{\ge0}$ is some threshold. The goal is to
    classify all and only the rewards $r\in\R$ based on their (non)compatibility
    with $\pie$ in $\M$ w.r.t. $\Delta$. In symbols:
    \begin{align*}
        \forall r\in\R:\;\textnormal{ \textbf{if} }\;\overline{\C}_{p,\pie}(r)
        \le\Delta\;\textnormal{ \textbf{then} } \;  \text{ \textnormal{\textbf{return}} \textnormal{True}, \; \textnormal{\textbf{else}} \textnormal{\textbf{return}} \textnormal{False}}.
    \end{align*}
    An \emph{IRL algorithm} takes in input a
    reward $r\in\R$ and outputs a boolean saying whether
    $\overline{\C}_{p,\pie}(r)\le\Delta$.
\end{defi}

Given $r\in\R$, we output whether it makes the expert's policy $\pie$
at most $\Delta$-suboptimal or not. Intuitively, we classify rewards in
$\R$ based on how good $\pie$ performs w.r.t. them. A $\Delta$-(non)compatible
reward guarantees that, among its $\Delta$-optimal policies, there is $\pie$,
but the optimal policy might be different from $\pi^E$ (see Appendix \ref{section: guarantees
forward rl} for how this relates to (forward) RL). Note that we allow for
$\R\neq\mathfrak{R}$ to manage scenarios in which we have some prior knowledge on $r^E$,
i.e., $r^E\in\R\subset\mathfrak{R}$.
 
\begin{remark}
    Permitting non-zero (non)compatibility is equivalent to enlarging the
    feasible set. Let $\R=\mathfrak{R}$, and define the set of rewards
    positively classified as $\R_\Delta$, i.e., $\mathcal{R}_{\Delta}\coloneqq
    \{r \in \mathcal{R} \,|\, \overline{\C}_{p,\pie}(r) \le\Delta \}$. For any
    $\Delta,\Delta'$ s.t. $0\le\Delta\le\Delta'\le2H$, we have: $        \R_{p,\pie}=\mathcal{R}_{0} \subseteq \mathcal{R}_{\Delta}
        \subseteq \mathcal{R}_{\Delta'} \subseteq
        \mathcal{R}_{2H} = \mathfrak{R} $.
\end{remark}

\textbf{Discussion on Reward Compatibility.}~~It should be remarked that:
\begin{itemize}[leftmargin=*, noitemsep, topsep=-1pt]
    \item \emph{The limits of the rewards compatibility framework are the same
    as the limits of the feasible set}. We cannot identify $r^E$ from the
    feasible set or among the rewards with small (non)compatibility. As
    aforementioned, this is an inherent limit of IRL and cannot be overcome with a more refined objective
    formulation, unless further information on $r^E$ is available (e.g., preferences).
    \item \emph{Rewards compatibility offers advantages over
    feasible set}. Differently from the feasible set, as we
    will see in Section \ref{section: finally linear mdps}, it is possible to
    \emph{practically} implement algorithms that
    solve the IRL classification problem, with guarantees of sample efficiency
    even when the state space is large.
\end{itemize}

\subsection{A Learning Framework for Online IRL Classification}

In this section, we combine the online IRL setting presented in Section \ref{section:
preliminaries} with the IRL classification problem of Definition \ref{def: IRL
problem}. Intuitively, the performance of an algorithm depends on its accuracy
at estimating the (non)compatibility of the rewards, as formalized by the
following PAC requirement.

\begin{defi}[PAC Framework]\label{def: new pac framework} Let
$\epsilon,\delta\in(0,1)$, and let $\D^E$ be a dataset of $\tau^E$ expert's
trajectories. An algorithm $\mathfrak{A}$ exploring for $\tau$ episodes is
$(\epsilon,\delta)$-PAC for the IRL classification problem if:
\begin{align*} 
    \mathop{\mathbb{P}}\limits_{\M,\pie,\mathfrak{A}}\Big(
        \sup\limits_{r\in\R} \Big|\overline{\C}_{p,\pie}(r)-
        \widehat{\C}(r)\Big|\le\epsilon\Big)\ge 1-\delta,
\end{align*}
where $\mathbb{P}_{\M,\pie,\mathfrak{A}}$ is the joint probability measure
induced by $\pie$ and $\mathfrak{A}$ in $\M$, and $\widehat{\C}$ is the estimate
of $\overline{\C}_{p,\pie}$ computed by $\mathfrak{A}$. The \emph{sample
complexity} is defined by the pair $(\tau^E,\tau)$.
\end{defi}

\begin{wrapfigure}{r}{.4\textwidth}
    \centering
    \vspace{-0.2cm}
    \includegraphics[width=.38\textwidth]{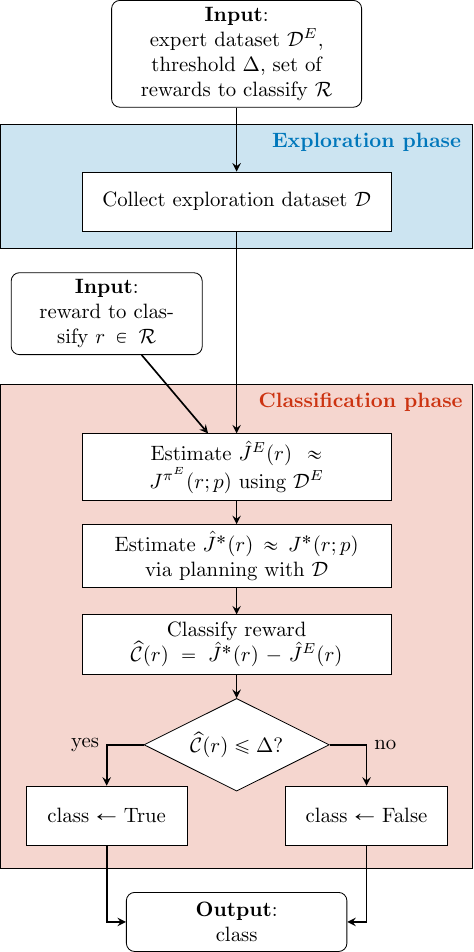}
    \caption{Flow-chart of \caty.}\label{fig:flow}
    \vspace{-1.3cm}
\end{wrapfigure}

Intuitively, our goal is to estimate the (non)compatibility of the rewards in
$\R$ with sufficient accuracy, so that, given a threshold $\Delta \ge 0$, we are able
to classify ``most'' of them correctly w.h.p. (with high probability).
The concept is exemplified in Figure \ref{fig: reward classification}.
Note that the estimation problem is independent of the threshold $\Delta$, which
can be appropriately selected to cope with noise in the demonstrations, (unknown) expert
suboptimality, or to manage the amount of ``false negatives'' and ``false positives''.

\textbf{Remark 4.2.}~~
    \textit{For $\eta\ge0$, let $\mathcal{R}_{\eta}\coloneqq \{r \in \mathcal{R} \,|\,
    \overline{\C}_{p,\pie}(r) \le\eta \}$ and  $\widehat{\R}_{\eta}\coloneqq \{r
    \in \mathcal{R} \,|\, \widehat{\C}(r) \le\eta \}$ denote the sets of rewards
    positively classified using, respectively, the true (non)compatibility
    $\overline{\C}_{p,\pie}$ and the estimate $\widehat{\C}$ constructed by an
    $(\epsilon,\delta)$-PAC algorithm. Then, with probability $1-\delta$, it
    holds that:
    $\widehat{\R}_{\Delta-\epsilon}\subseteq\R_{\Delta}\subseteq
    \widehat{\R}_{\Delta+\epsilon}$. Thus, we can trade-off the amount of
    ``false negatives'' (resp. ``false positives'') by, e.g., choosing the
    threshold $\Delta\gets\Delta+\epsilon$ (resp.
    $\Delta\gets\Delta-\epsilon$).}

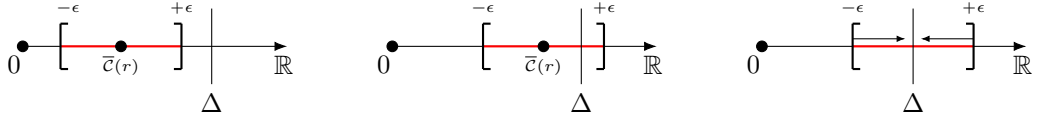
\begin{figure*}[!t]
\centering
\begin{subfigure}[b]{0.3\textwidth}
    \centering
    \begin{tikzpicture}
        \filldraw [black] (0,0) circle (2pt);
        \draw[-{Latex[scale=1.]}] (0,0) -- (3.5,0);
        \draw[thick,red] (0.5,0) -- (2.1,0);
        \filldraw [black] (1.3,0) circle (2pt);
        \draw (2.5,0.5) -- (2.5,-0.5);
        \draw[thick] (0.5,0.3) -- (0.5,-0.3);
        \draw[thick] (0.5,0.3) -- (0.6,0.3);
        \draw[thick] (0.5,-0.3) -- (0.6,-0.3);
        \draw[thick] (2.1,0.3) -- (2.1,-0.3);
        \draw[thick] (2,0.3) -- (2.1,0.3);
        \draw[thick] (2,-0.3) -- (2.1,-0.3);
        \draw (0.1,0) node[anchor=north east] {$0$};
        \draw (3.2,0) node[anchor=north west] {$ \RR$};
        \draw (2.5,-0.5) node[anchor=north] {$\Delta$};
        \draw (1.3,0) node[anchor=north] {\tiny$\overline{\C}(r)$};
        \draw (0.6,0.3) node[anchor=south] {\tiny$-\epsilon$};
        \draw (2.1,0.3) node[anchor=south] {\tiny$+\epsilon$};
    \end{tikzpicture}
    \caption
    {{ Reward $r$ is classified correctly.}}    
    \label{fig:r correct}
\end{subfigure}
\hfill
\begin{subfigure}[b]{0.3\textwidth}  
    \centering 
    \begin{tikzpicture}
        \filldraw [black] (0,0) circle (2pt);
        \draw[-{Latex[scale=1.]}] (0,0) -- (3.5,0);
        \draw[thick,red] (1.2,0) -- (2.8,0);
        \filldraw [black] (2,0) circle (2pt);
        \draw (2.5,0.5) -- (2.5,-0.5);
        \draw[thick] (1.2,0.3) -- (1.2,-0.3);
        \draw[thick] (1.2,0.3) -- (1.3,0.3);
        \draw[thick] (1.2,-0.3) -- (1.3,-0.3);
        \draw[thick] (2.8,0.3) -- (2.8,-0.3);
        \draw[thick] (2.7,0.3) -- (2.8,0.3);
        \draw[thick] (2.7,-0.3) -- (2.8,-0.3);
        \draw (0.1,0) node[anchor=north east] {$0$};
        \draw (3.2,0) node[anchor=north west] {$ \RR$};
        \draw (2.5,-0.5) node[anchor=north] {$\Delta$};
        \draw (2,0) node[anchor=north] {\tiny$\overline{\C}(r)$};
        \draw (1.2,0.3) node[anchor=south] {\tiny$-\epsilon$};
        \draw (2.8,0.3) node[anchor=south] {\tiny$+\epsilon$};
    \end{tikzpicture}
    \caption[]%
    {{ Reward $r$ can be mis-classified.}}
    \label{fig: r maybe wrong}
\end{subfigure}
\hfill
\begin{subfigure}[b]{0.3\textwidth}  
    \centering 
    \begin{tikzpicture}
        \filldraw [black] (0,0) circle (2pt);
        \draw[-{Latex[scale=1.]}] (0,0) -- (3.5,0);
        \draw[thick,red] (1.2,0) -- (2.8,0);
        \draw (2,0.5) -- (2,-0.5);
        \draw[thick] (1.2,0.3) -- (1.2,-0.3);
        \draw[thick] (1.2,0.3) -- (1.3,0.3);
        \draw[thick] (1.2,-0.3) -- (1.3,-0.3);
        \draw[thick] (2.8,0.3) -- (2.8,-0.3);
        \draw[thick] (2.7,0.3) -- (2.8,0.3);
        \draw[thick] (2.7,-0.3) -- (2.8,-0.3);
        \draw[-{Latex[scale=0.7]}] (1.2,0.1) -- (1.9,0.1);
        \draw[-{Latex[scale=0.7]}] (2.8,0.1) -- (2.1,0.1);
        \draw (0.1,0) node[anchor=north east] {$0$};
        \draw (3.2,0) node[anchor=north west] {$ \RR$};
        \draw (2,-0.5) node[anchor=north] {$\Delta$};
        \draw (1.2,0.3) node[anchor=south] {\tiny$-\epsilon$};
        \draw (2.8,0.3) node[anchor=south] {\tiny$+\epsilon$};
    \end{tikzpicture}
    \caption[]%
    {{ Range of uncertain (non)compatibility values.}}    
    \label{fig: range uncertain noncomp values}
\end{subfigure}
\caption
{The axis represents (estimated) (non)compatibility values. (a) Rewards
$r$ whose true (non)compatibility
$\overline{\C}(r)\coloneqq\overline{\C}_{p,\pie}(r)$ is far from threshold
$\Delta$ by at least $\epsilon$, are correctly classified, while (b) in the opposite case, rewards can be mis-classified. (c) The red
interval $[\Delta-\epsilon,\Delta+\epsilon]$ exemplifies the set of rewards
$\{r\in\R\,|\, |\overline{\C}(r)-\Delta|\le\epsilon\}$ that are (potentially)
mis-classified. The length of the interval reduces with $\epsilon$.}
\label{fig: reward classification}
\end{figure*}

\section{\caty: A Provably Efficient Algorithm for IRL}\label{section: finally linear mdps}

In this section, we present \caty (\catylong), a provably efficient algorithm
for solving the \emph{online} IRL \emph{classification} problem. We consider
three different kinds of structure for the MDPs: tabular MDPs, tabular MDPs with
linear rewards, and Linear MDPs. Similarly to RFE, our online IRL classification
setting is made of two phases: ($i$) an \emph{exploration} phase, in which the
algorithm explores the environment using the knowledge of $\R$ and of
the expert's dataset $\D^E$ to collect samples about the dynamics of the MDP,
and ($ii$) a \emph{classification} phase, in which it performs the classification of a reward $r\in\R$ without interactions with the environment.
A flow-chart is reported in Figure~\ref{fig:flow} (pseudocode in Appendix \ref{section: more on algorithm}).

\textbf{\textcolor{vibrantBlue}{Exploration phase.}}~~The \emph{exploration}
phase collects a dataset $\mathcal{D}$ in a way that depends on the
structure of the MDP and of the set of rewards $\R$ to be classified.
Specifically, for Linear MDPs, \caty executes RFLin
\cite{wagenmaker2022noharder}. Instead, for tabular MDPs (with or without linear
reward), \caty instantiates either BPI-UCBVI \cite{menard2021fast} for each
reward $r \in \R$ (when $|\R| = \Theta(1)$, i.e., a ``small'' constant w.r.t. to the size
of the MDP, where ``small'' depends on the size of the state space, see Appendix
\ref{appendix: proof main algorithm}) or RF-Express \cite{menard2021fast}. 
%
Note that \caty in this phase does not use the expert's dataset  $\D^E$.
%

\textbf{\textcolor{vibrantRed}{Classification phase.}}~~The
\emph{classification} performs the estimation $\widehat{\C}(r)$
of the (non)compatibility term $\overline{\C}_{p,\pie}(r)$ for the single input
reward $r\in\R$ by splitting it into two independent estimates:
$\widehat{J}^E(r)\approx J^{\pie}(r;p)$, which is computed with $\D^E$ only, and
$\widehat{J}^*(r)\approx J^*(r;p)$, which is computed with $\D$ only.
Concerning $\widehat{J}^E(r)$, when the reward is linear
$r_h(s,a)=\dotp{\phi(s,a), \theta_h}$, \caty uses $\D^E$ to construct an
empirical estimate $\widehat{\psi}^E\approx \psi^{p,\pi^E}$ of the expert's
expected feature count \cite{arora2018survey}. Otherwise, it directly estimates
$\widehat{d}^E\approx d^{p,\pi^E}$ the expert's occupancy measure. Such
estimates can be used to derive $\widehat{J}^E(r)$ straightforwardly.
Regarding $\widehat{J}^*(r)$, \caty exploits the \emph{planning} phase of the
corresponding RFE (or BPI) algorithm adopted at exploration
phase.\footnote{RFE/BPI algorithms, at planning phase, return a policy, and not
its estimated performance. Since BPI-UCBVI, RF-Express, and RFLin each compute
an estimate of $J^*(r;p)$ as an intermediate step, with negligible abuse of notation, we
assume that they output such estimate.}
Finally, \caty applies the (potentially negative) input threshold $\Delta$ to
the difference $\widehat{J}^*(r)-\widehat{J}^E(r)$ to perform the classification.
See Appendix \ref{section: more on algorithm} for the full pseudo-code.
Clearly, \caty can be implemented in practice, since it considers a single
reward at a time instead of computing the full feasible set, and it is
computationally efficient in linear MDPs, since it uses a computationally
efficient algorithm as subroutine (see \cite{wagenmaker2022noharder}).

%

\textbf{Sample Efficiency.}~~The next result analyzes the sample complexity (Definition \ref{def: new pac
framework}) of \caty.

\begin{restatable}[Sample Complexity of \caty]{thr}{upperboundtabular}\label{thr: bounds tabular} 
Let
$\epsilon,\delta\in(0,1)$. Then \caty is $(\epsilon,\delta)$-PAC for IRL with a
sample complexity upper bounded by:
    \[\begin{array}{lll}
    \text{Tabular MDPs:}&
    \displaystyle \tau^E\le
    \widetilde{\mathcal{O}}\Big(\frac{H^{3}SA}{\epsilon^2}\log\frac{1}{\delta}\Big), &
    \displaystyle \tau\le \widetilde{\mathcal{O}}\Big(\frac{H^3SA}{\epsilon^2}\Big(N+\log\frac{1}{\delta}\Big)\Big),\\[.35cm]
    \text{Tabular MDPs with linear rewards:}&
    \displaystyle \tau^E\le \widetilde{\mathcal{O}}\Big(\frac{H^{3}d}{\epsilon^2}\log\frac{1}{\delta}\Big),& 
        \displaystyle \tau\le \widetilde{\mathcal{O}}\Big(\frac{H^3SA}{\epsilon^2}\Big(N+\log\frac{1}{\delta}\Big)\Big),\\[.35cm]
    \text{Linear MDPs:}&
    \displaystyle \tau^E\le \widetilde{\mathcal{O}}\Big(\frac{H^{3}d}{\epsilon^2}\log\frac{1}{\delta}\Big),&
          \displaystyle \tau\le \widetilde{\mathcal{O}}\Big(\frac{H^5d}{\epsilon^2}\Big(d+\log\frac{1}{\delta}\Big)\Big),
    \end{array}\]
    where $N=0$ if $|\R|=\Theta(1)$, and $N=S$ otherwise.  
\end{restatable}
Some observations are in order.
%
%
We conjecture that the $d^2$ dependence when $|\R|=\Theta(1)$ is
unavoidable in Linear MDPs because of the lower bound for BPI in
\cite{wagenmaker2022noharder}.
In tabular MDPs with deterministic expert, one might use the results in
\cite{Xu2023ProvablyEA} to reduce the rate of $\tau^E$ from
$\widetilde{\mathcal{O}}(SAH^3 \log(\delta^{-1})/\epsilon^2)$ to $\widetilde{\mathcal{O}}(SH^{3/2}\log(\delta^{-1})/\epsilon^2)$. 
Finally, note that the choice $\Delta=\epsilon$ allows us to positively classify
all the rewards in the feasible set $\R_{p,\pie}$ w.h.p. and, in this case, other rewards
positively classified have true (non)compatibility at most $2\epsilon$ w.h.p. In light of this result we conclude that  \emph{rewards compatibility} framework allows the \emph{practical} development of     \emph{sample efficient} algorithms (e.g., \caty) in Linear MDPs with    large/continuous state spaces.
%
%

\section{Statistical Barriers and Objective-Free Exploration}\label{section: statistical barriers}

In this section, we show that \caty is minimax optimal for the number of
exploration episodes in tabular MDPs, and that RFE and IRL share the same
theoretical sample complexity. This allows us to formulate \emph{Objective-Free
Exploration}, a unifying setting for exploration problems.

\subsection{The Theoretical Limits of IRL (and RFE) in the Tabular Setting}\label{section: insights rfe and irl}
In \caty, we use a minimax optimal RFE algorithm for exploration. However,
this does not entail that \caty is minimax optimal for the IRL
classification problem. There might exist another PAC algorithm
with a sample complexity smaller than \caty. The following result
states that, in the tabular setting, the bound in Theorem \ref{thr: bounds
tabular} is tight for the number of
episodes $\tau$.

\begin{restatable} [IRL Classification - Lower Bound]{thr}{instancedependentlowerbound}\label{thr: instance
dependent lower bound} Let $\mathfrak{A}$ be an $(\epsilon,\delta)$-PAC
algorithm for the IRL classification in tabular MDPs. Let
$\tau$ be the number of exploration episodes. Then, there exists
an IRL classification instance such that:
    \begin{align*}
      \text{if }|\R|\ge1:\;\tau\ge\Omega\bigg(
            \frac{H^3SA}{\epsilon^2}\log\frac{1}{\delta}
            \bigg),\qquad \text{if }\R=\mathfrak{R}:\;\tau\ge\Omega\bigg(
            \frac{H^3SA}{\epsilon^2}\Big(
            S+\log\frac{1}{\delta}    
            \Big)    
            \bigg).
    \end{align*}
\end{restatable}

In both cases, the lower bound is \emph{matched} by \caty, up to logarithmic factors. Note that \caty
explores without using $\D^E$, thus, minimax optimality for $\tau$ can be
achieved without the knowledge of $\D^E$ at exploration phase.
As a by-product, we observe that a similar lower bound construction can be made
also for RFE, leading to the following result.

\begin{restatable} [RFE - Refined Lower Bound]{thr}{lowerboundrfe}\label{thr: lower bound rfe} Let
    $\mathfrak{A}$ be an $(\epsilon,\delta)$-PAC algorithm for RFE in tabular
    MDPs. Let $\tau$ be the number of exploration episodes.
    Then, there exists an RFE instance such that:
        \begin{align*}
            \tau\ge\Omega\bigg(
                \frac{H^3SA}{\epsilon^2}\Big(
                S+\log\frac{1}{\delta}    
                \Big)    
                \bigg).
        \end{align*}
\end{restatable}

This bound improves the state-of-the-art RFE lower bound
$\Omega(\frac{H^3SA}{\epsilon^2}(\frac{S}{H}+\log\frac{1}{\delta}))$ (obtained
combining the bounds in \cite{jin2020RFE} and \cite{domingues2021episodic}) by
one $H$ factor, and it is matched by RF-Express
\cite{menard2021fast}.

\subsection{Objective-Free Exploration (OFE)}\label{section: ofe}

What is the most efficient exploration strategy that can be performed in an
unknown environment? It \emph{depends} on the subsequent task that shall be
solved. However, if the task is unknown at the exploration phase, we need a strategy
that suffices for all the tasks that one might be interested in solving. Let us
denote by $\mathscr{F}$ the set of RL and IRL classification tasks. Since \caty
is a sample efficient algorithm for the IRL classification problem, and it uses
RFE as a subroutine, we conclude that the RFE exploration strategy is sufficient
(and also minimax optimal in tabular MDPs) to obtain guarantees for class
$\mathscr{F}$.
Are there other problems for which RFE exploration suffices when the specific
problem instance is revealed \emph{a posteriori} of the exploration phase? We
believe so, and in Appendix \ref{section: more lower bound}, we identify two
additional problems, i.e., Matching Performance (MP) and Imitation Learning from
Observations alone (ILfO) \cite{liu2018imitation}, that represent potential
candidates to belong to $\mathscr{F}$.

More in general, we formulate the \emph{Objective-Free Exploration (OFE)}
problem as follows:

\begin{defi}[Objective-Free Exploration]
Given a tuple $\tuple{\M,\mathscr{F},(\epsilon,\delta)}$, where $\M$ is an
\emph{unknown} environment (e.g., MDP without reward), and $\mathscr{F}$ is a
certain class of tasks (e.g., all RL and IRL problems), the \emph{Objective-Free
Exploration} (OFE) problem aims to find an exploration of the environment
$\M$ (e.g., RFE exploration) that permits to solve \emph{any} task
$f\in\mathscr{F}$ in an $(\epsilon,\delta)$-correct manner.
\end{defi}

This problem is called ``objective-free'' because it does not require the
knowledge of the specific ``objective'' $f\in\mathscr{F}$ to be solved. In
Appendix \ref{section: more on ofe}, we describe a use case for OFE. We believe
this is an interesting problem to be studied in future.

\section{Conclusions}\label{section: conclusions}

In this paper, we have shown that the feasible set cannot be learned efficiently
in problems with large/continuous state spaces even under the strong structure
provided by Linear MDPs.
For this reason, we have introduced the powerful framework of \emph{compatible
rewards}, which formalizes the intuitive notion of compatibility of a
reward function with expert demonstrations, and it allows us to formulate the
IRL problem as a \emph{classification} task.
In this context, we have devised \caty, a provably efficient IRL algorithm for
Linear MDPs with large/continuous state spaces.
Furthermore, in tabular MDPs, we have demonstrated the minimax optimality of
\caty at exploration by presenting a novel lower bound to the IRL classification
problem. As a by-product, our construction improves the current state-of-the-art
lower bound for RFE.
Finally, we have introduced OFE, a unifying problem setting for exploration
problems, which generalizes both RFE and IRL.

\textbf{Limitations.}~~A limitation of our contributions concerns the adoption
of the \emph{Linear MDP} model, whose assumptions are overly strong to be
consistently applied to real-world applications. Nevertheless, while the rewards
compatibility framework is general and not tied to Linear MDPs, we believe that
Linear MDPs represent an important initial step toward the development of
provably efficient IRL algorithms with more general function
approximation structures. Although
a lower bound for Linear MDPs is missing, we believe that it represents an
interesting direction for future works.
Finally, we note that the \emph{empirical
validation} of the proposed algorithm is out of the scope of this work.

\textbf{Future Directions.}~~Promising directions for future works concern the
extension of the analysis of the \emph{rewards compatibility} framework beyond
Linear MDPs to general function approximation and to the offline setting.
In addition, it might be fascinating to extend the notion of reward
compatibility to other kinds of expert feedback (in the context of ReL), and
to other IRL settings (e.g., suboptimal experts). Finally, we believe that OFE
should be analysed in-depth given its practical importance.

\begin{ack}
    AI4REALNET has received funding from European Union's Horizon Europe
    Research and Innovation programme under the Grant Agreement No 101119527.
    Views and opinions expressed are however those of the author(s) only and do
    not necessarily reflect those of the European Union. Neither the European
    Union nor the granting authority can be held responsible for them.

    Funded by the European Union - Next Generation EU within the project NRPP
    M4C2, Investment 1.,3 DD. 341 - 15 march 2022 - FAIR - Future Artificial
    Intelligence Research - Spoke 4 - PE00000013 - D53C22002380006.
\end{ack}

\bibliographystyle{plain}
\bibliography{refs.bib}

\begin{thebibliography}{10}

\bibitem{abbeel2006helicopter}
Pieter Abbeel, Adam Coates, Morgan Quigley, and Andrew Ng.
\newblock An application of reinforcement learning to aerobatic helicopter flight.
\newblock In {\em Advances in Neural Information Processing Systems 19 (NeurIPS)}, 2006.

\bibitem{abbeel2004apprenticeship}
Pieter Abbeel and Andrew~Y. Ng.
\newblock Apprenticeship learning via inverse reinforcement learning.
\newblock In {\em International Conference on Machine Learning 21 (ICML)}, 2004.

\bibitem{abbeel2005exploration}
Pieter Abbeel and Andrew~Y. Ng.
\newblock Exploration and apprenticeship learning in reinforcement learning.
\newblock In {\em International Conference on Machine Learning 22 (ICML)}, 2005.

\bibitem{adams2022survey}
Stephen Adams, Tyler Cody, and Peter~A. Beling.
\newblock A survey of inverse reinforcement learning.
\newblock {\em Artificial Intelligence Review}, 55:4307--4346, 2022.

\bibitem{amin2016resolving}
Kareem Amin and Satinder Singh.
\newblock Towards resolving unidentifiability in inverse reinforcement learning, 2016.

\bibitem{arora2018survey}
Saurabh Arora and Prashant Doshi.
\newblock A survey of inverse reinforcement learning: Challenges, methods and progress.
\newblock {\em Artificial Intelligence}, 297:103500, 2018.

\bibitem{barnes2024massively}
Matt Barnes, Matthew Abueg, Oliver~F. Lange, Matt Deeds, Jason Trader, Denali Molitor, Markus Wulfmeier, and Shawn O'Banion.
\newblock Massively scalable inverse reinforcement learning in google maps, 2024.

\bibitem{bertsekas2009convex}
P.~Dimitri Bertsekas.
\newblock {\em Convex Optimization Theory}.
\newblock Athena Scientific, 2009.

\bibitem{bowling2023settlingrewardhypothesis}
Michael Bowling, John~D. Martin, David Abel, and Will Dabney.
\newblock Settling the reward hypothesis.
\newblock In {\em International Conference on Machine Learning 40 (ICML)}, 2023.

\bibitem{cao2021identifiability}
Haoyang Cao, Samuel Cohen, and Lukasz Szpruch.
\newblock Identifiability in inverse reinforcement learning.
\newblock In {\em Advances in Neural Information Processing Systems 34 (NeurIPS)}, pages 12362--12373, 2021.

\bibitem{dexter2021IRL}
Gregory Dexter, Kevin Bello, and Jean Honorio.
\newblock Inverse reinforcement learning in a continuous state space with formal guarantees.
\newblock In {\em Advances in Neural Information Processing Systems 34 (NeurIPS)}, pages 6972--6982, 2021.

\bibitem{domingues2021episodic}
Omar~Darwiche Domingues, Pierre M{\'e}nard, Emilie Kaufmann, and Michal Valko.
\newblock Episodic reinforcement learning in finite mdps: Minimax lower bounds revisited.
\newblock In {\em International Conference on Algorithmic Learning Theory 32 (ALT)}, volume 132, pages 578--598, 2021.

\bibitem{du2021bilinear}
Simon Du, Sham Kakade, Jason Lee, Shachar Lovett, Gaurav Mahajan, Wen Sun, and Ruosong Wang.
\newblock Bilinear classes: A structural framework for provable generalization in rl.
\newblock In {\em International Conference on Machine Learning 38 (ICML)}, volume 139, pages 2826--2836, 2021.

\bibitem{finn2016guided}
Chelsea Finn, Sergey Levine, and Pieter Abbeel.
\newblock Guided cost learning: Deep inverse optimal control via policy optimization.
\newblock In {\em International Conference on Machine Learning 33 (ICML)}, volume~48, pages 49--58, 2016.

\bibitem{Fu2017LearningRR}
Justin Fu, Katie Luo, and Sergey Levine.
\newblock Learning robust rewards with adversarial inverse reinforcement learning.
\newblock In {\em International Conference on Learning Representations 5 (ICLR)}, 2017.

\bibitem{hadfieldmenell2017inverserewarddesign}
Dylan Hadfield-Menell, Smitha Milli, Pieter Abbeel, Stuart~J Russell, and Anca Dragan.
\newblock Inverse reward design.
\newblock In {\em Advances in Neural Information Processing Systems 30 (NeurIPS)}, 2017.

\bibitem{hadfieldmenell2016cooperativeIRL}
Dylan Hadfield-Menell, Stuart~J Russell, Pieter Abbeel, and Anca Dragan.
\newblock Cooperative inverse reinforcement learning.
\newblock In {\em Advances in Neural Information Processing Systems 29 (NeurIPS)}, 2016.

\bibitem{ho2016gail}
Jonathan Ho and Stefano Ermon.
\newblock Generative adversarial imitation learning.
\newblock In {\em Advances in Neural Information Processing Systems 29 (NeurIPS)}, 2016.

\bibitem{jeon2020rewardrational}
Hong~Jun Jeon, Smitha Milli, and Anca Dragan.
\newblock Reward-rational (implicit) choice: A unifying formalism for reward learning.
\newblock In {\em Advances in Neural Information Processing Systems 33 (NeurIPS)}, pages 4415--4426, 2020.

\bibitem{jiang2017contextual}
Nan Jiang, Akshay Krishnamurthy, Alekh Agarwal, John Langford, and Robert~E. Schapire.
\newblock Contextual decision processes with low {B}ellman rank are {PAC}-learnable.
\newblock In {\em International Conference on Machine Learning 34 (ICML)}, volume~70, pages 1704--1713, 2017.

\bibitem{jin2020RFE}
Chi Jin, Akshay Krishnamurthy, Max Simchowitz, and Tiancheng Yu.
\newblock Reward-free exploration for reinforcement learning.
\newblock In {\em International Conference on Machine Learning 37 (ICML)}, volume 119, pages 4870--4879, 2020.

\bibitem{jin2021eluder}
Chi Jin, Qinghua Liu, and Sobhan Miryoosefi.
\newblock Bellman eluder dimension: New rich classes of rl problems, and sample-efficient algorithms.
\newblock In {\em Advances in Neural Information Processing Systems 34 (NeurIPS)}, pages 13406--13418, 2021.

\bibitem{jin2020provablyefficient}
Chi Jin, Zhuoran Yang, Zhaoran Wang, and Michael~I Jordan.
\newblock Provably efficient reinforcement learning with linear function approximation.
\newblock In {\em Conference on Learning Theory 33 (COLT)}, volume 125, pages 2137--2143, 2020.

\bibitem{kaufmann2021adaptiveRFE}
Emilie Kaufmann, Pierre M{\'e}nard, Omar Darwiche~Domingues, Anders Jonsson, Edouard Leurent, and Michal Valko.
\newblock Adaptive reward-free exploration.
\newblock In {\em International Conference on Algorithmic Learning Theory 32 (ALT)}, volume 132, pages 865--891, 2021.

\bibitem{kim2021rewardidentification}
Kuno Kim, Shivam Garg, Kirankumar Shiragur, and Stefano Ermon.
\newblock Reward identification in inverse reinforcement learning.
\newblock In {\em International Conference on Machine Learning 38 (ICML)}, pages 5496--5505, 2021.

\bibitem{kim2020domain}
Kuno Kim, Yihong Gu, Jiaming Song, Shengjia Zhao, and Stefano Ermon.
\newblock Domain adaptive imitation learning.
\newblock In {\em International Conference on Machine Learning 37 (ICML)}, volume 119, pages 5286--5295, 2020.

\bibitem{klein2012IRLclassification}
Edouard Klein, Matthieu Geist, Bilal Piot, and Olivier Pietquin.
\newblock Inverse reinforcement learning through structured classification.
\newblock In {\em Advances in Neural Information Processing Systems 25 (NeurIPS)}, 2012.

\bibitem{komanduru2019correctness}
Abi Komanduru and Jean Honorio.
\newblock On the correctness and sample complexity of inverse reinforcement learning.
\newblock In {\em Advances in Neural Information Processing Systems 32 (NeurIPS)}, 2019.

\bibitem{komanduru2021lowerbound}
Abi Komanduru and Jean Honorio.
\newblock A lower bound for the sample complexity of inverse reinforcement learning.
\newblock In {\em International Conference on Machine Learning 38 (ICML)}, volume 139, pages 5676--5685, 2021.

\bibitem{Kreps1988NotesOT}
David~M. Kreps.
\newblock {\em Notes On The Theory Of Choice}.
\newblock Westview Press, 1988.

\bibitem{lazzati2024offline}
Filippo Lazzati, Mirco Mutti, and Alberto~Maria Metelli.
\newblock Offline inverse rl: New solution concepts and provably efficient algorithms.
\newblock In {\em International Conference on Machine Learning 41 (ICML)}, 2024.

\bibitem{li2023minimaxoptimal}
Gen Li, Yuling Yan, Yuxin Chen, and Jianqing Fan.
\newblock Minimax-optimal reward-agnostic exploration in reinforcement learning, 2023.

\bibitem{lindner2022active}
David Lindner, Andreas Krause, and Giorgia Ramponi.
\newblock Active exploration for inverse reinforcement learning.
\newblock In {\em Advances in Neural Information Processing Systems 35 (NeurIPS)}, pages 5843--5853, 2022.

\bibitem{liu2018imitation}
YuXuan Liu, Abhishek Gupta, Pieter Abbeel, and Sergey Levine.
\newblock Imitation from observation: Learning to imitate behaviors from raw video via context translation.
\newblock In {\em IEEE International Conference on Robotics and Automation (ICRA)}, pages 1118--1125, 2018.

\bibitem{liu2022learning}
Zhihan Liu, Yufeng Zhang, Zuyue Fu, Zhuoran Yang, and Zhaoran Wang.
\newblock Learning from demonstration: Provably efficient adversarial policy imitation with linear function approximation.
\newblock In {\em International Conference on Machine Learning 39 (ICML)}, volume 162, pages 14094--14138, 2022.

\bibitem{lopes2009active}
Manuel Lopes, Francisco Melo, and Luis Montesano.
\newblock Active learning for reward estimation in inverse reinforcement learning.
\newblock In {\em Machine Learning and Knowledge Discovery in Databases (ECML PKDD)}, pages 31--46, 2009.

\bibitem{menard2021fast}
Pierre Menard, Omar~Darwiche Domingues, Anders Jonsson, Emilie Kaufmann, Edouard Leurent, and Michal Valko.
\newblock Fast active learning for pure exploration in reinforcement learning.
\newblock In {\em International Conference on Machine Learning 38 (ICML)}, volume 139, pages 7599--7608, 2021.

\bibitem{metelli2023towards}
Alberto~Maria Metelli, Filippo Lazzati, and Marcello Restelli.
\newblock Towards theoretical understanding of inverse reinforcement learning.
\newblock In {\em International Conference on Machine Learning 40 (ICML)}, volume 202, pages 24555--24591, 2023.

\bibitem{metelli2021provably}
Alberto~Maria Metelli, Giorgia Ramponi, Alessandro Concetti, and Marcello Restelli.
\newblock Provably efficient learning of transferable rewards.
\newblock In {\em International Conference on Machine Learning 38 (ICML)}, volume 139, pages 7665--7676, 2021.

\bibitem{michini2013scalable}
Bernard Michini, Mark Cutler, and Jonathan~P. How.
\newblock Scalable reward learning from demonstration.
\newblock In {\em IEEE International Conference on Robotics and Automation (ICRA)}, pages 303--308, 2013.

\bibitem{mnih2013playing}
Volodymyr Mnih, Koray Kavukcuoglu, David Silver, Alex Graves, Ioannis Antonoglou, Daan Wierstra, and Martin Riedmiller.
\newblock Playing atari with deep reinforcement learning, 2013.

\bibitem{ng2000algorithms}
Andrew~Y. Ng and Stuart~J. Russell.
\newblock Algorithms for inverse reinforcement learning.
\newblock In {\em International Conference on Machine Learning 17 (ICML)}, pages 663--670, 2000.

\bibitem{ornstein1969existence}
Donald Ornstein.
\newblock On the existence of stationary optimal strategies.
\newblock {\em Proceedings of the American Mathematical Society}, 20(2):563--569, 1969.

\bibitem{poiani2024inverse}
Riccardo Poiani, Gabriele Curti, Alberto~Maria Metelli, and Marcello Restelli.
\newblock Inverse reinforcement learning with sub-optimal experts, 2024.

\bibitem{puterman1994markov}
Martin~Lee Puterman.
\newblock {\em {M}arkov Decision Processes: Discrete Stochastic Dynamic Programming}.
\newblock John Wiley \& Sons, Inc., 1994.

\bibitem{rajaraman2021value}
Nived Rajaraman, Yanjun Han, Lin Yang, Jingbo Liu, Jiantao Jiao, and Kannan Ramchandran.
\newblock On the value of interaction and function approximation in imitation learning.
\newblock In {\em Advances in Neural Information Processing Systems 34 (NeurIPS)}, volume~34, pages 1325--1336, 2021.

\bibitem{Rajaraman2020TowardTF}
Nived Rajaraman, Lin Yang, Jiantao Jiao, and Kannan Ramchandran.
\newblock Toward the fundamental limits of imitation learning.
\newblock In {\em Advances in Neural Information Processing Systems 33 (NeurIPS)}, pages 2914--2924, 2020.

\bibitem{ratliff2006mmp}
Nathan~D. Ratliff, J.~Andrew Bagnell, and Martin~A. Zinkevich.
\newblock Maximum margin planning.
\newblock In {\em International Conference on Machine Learning 23 (ICML)}, pages 729--736, 2006.

\bibitem{russell1998learning}
Stuart Russell.
\newblock Learning agents for uncertain environments (extended abstract).
\newblock In {\em Conference on Computational Learning Theory 11 (COLT)}, pages 101--103, 1998.

\bibitem{russel2010ai}
Stuart Russell and Peter Norvig.
\newblock {\em Artificial Intelligence: A Modern Approach}.
\newblock Prentice Hall, 3 edition, 2010.

\bibitem{shah2019feasibility}
Rohin Shah, Noah Gundotra, Pieter Abbeel, and Anca Dragan.
\newblock On the feasibility of learning, rather than assuming, human biases for reward inference.
\newblock In {\em International Conference on Machine Learning 36 (ICML)}, volume~97, pages 5670--5679, 2019.

\bibitem{shakerinava2022utility}
Mehran Shakerinava and Siamak Ravanbakhsh.
\newblock Utility theory for sequential decision making.
\newblock In {\em International Conference on Machine Learning 39 (ICML)}, volume 162, pages 19616--19625, 2022.

\bibitem{Shani2021OnlineAL}
Lior Shani, Tom Zahavy, and Shie Mannor.
\newblock Online apprenticeship learning.
\newblock In {\em AAAI Conference on Artificial Intelligence 36 (AAAI)}, pages 8240--8248, 2022.

\bibitem{silver2016go}
David Silver, Aja Huang, Christopher~J. Maddison, Arthur Guez, Laurent Sifre, George van~den Driessche, Julian Schrittwieser, Ioannis Antonoglou, Veda Panneershelvam, Marc Lanctot, Sander Dieleman, Dominik Grewe, John Nham, Nal Kalchbrenner, Ilya Sutskever, Timothy Lillicrap, Madeleine Leach, Koray Kavukcuoglu, Thore Graepel, and Demis Hassabis.
\newblock Mastering the game of go with deep neural networks and tree search.
\newblock {\em Nature}, 529:484--503, 2016.

\bibitem{skalse2023misspecification}
Joar Skalse and Alessandro Abate.
\newblock Misspecification in inverse reinforcement learning.
\newblock In {\em AAAI Conference on Artificial Intelligence 37 (AAAI)}, pages 15136--15143, 2023.

\bibitem{skalse2023invariance}
Joar Max~Viktor Skalse, Matthew Farrugia-Roberts, Stuart Russell, Alessandro Abate, and Adam Gleave.
\newblock Invariance in policy optimisation and partial identifiability in reward learning.
\newblock In {\em International Conference on Machine Learning 40 (ICML)}, volume 202, pages 32033--32058, 2023.

\bibitem{sun2019provably}
Wen Sun, Anirudh Vemula, Byron Boots, and Drew Bagnell.
\newblock Provably efficient imitation learning from observation alone.
\newblock In {\em International Conference on Machine Learning 36 (ICML)}, volume~97, pages 6036--6045, 2019.

\bibitem{sutton2018reinforcement}
Richard~S. Sutton and Andrew~G. Barto.
\newblock {\em Reinforcement Learning: An Introduction}.
\newblock A Bradford Book, 2018.

\bibitem{swamy2022minimax}
Gokul Swamy, Nived Rajaraman, Matt Peng, Sanjiban Choudhury, J.~Bagnell, Steven~Z. Wu, Jiantao Jiao, and Kannan Ramchandran.
\newblock Minimax optimal online imitation learning via replay estimation.
\newblock In {\em Advances in Neural Information Processing Systems 35 (NeruIPS)}, volume~35, pages 7077--7088, 2022.

\bibitem{syed2007game}
Umar Syed and Robert~E Schapire.
\newblock A game-theoretic approach to apprenticeship learning.
\newblock In {\em Advances in Neural Information Processing Systems 20 (NeurIPS)}, 2007.

\bibitem{viano2024imitation}
Luca Viano, Stratis Skoulakis, and Volkan Cevher.
\newblock Imitation learning in discounted linear mdps without exploration assumptions, 2024.

\bibitem{wagenmaker2022noharder}
Andrew~J Wagenmaker, Yifang Chen, Max Simchowitz, Simon Du, and Kevin Jamieson.
\newblock Reward-free {RL} is no harder than reward-aware {RL} in linear {M}arkov decision processes.
\newblock In {\em International Conference on Machine Learning 39 (ICML)}, volume 162, pages 22430--22456, 2022.

\bibitem{wang2020onrewardfree}
Ruosong Wang, Simon~S Du, Lin Yang, and Russ~R Salakhutdinov.
\newblock On reward-free reinforcement learning with linear function approximation.
\newblock In {\em Advances in Neural Information Processing Systems 33 (NeurIPS)}, pages 17816--17826, 2020.

\bibitem{wang2020general}
Ruosong Wang, Ruslan Salakhutdinov, and Lin~F. Yang.
\newblock Reinforcement learning with general value function approximation: provably efficient approach via bounded eluder dimension.
\newblock In {\em Advances in Neural Information Processing Systems 34 (NeurIPS)}, pages 6123--6135, 2020.

\bibitem{wirth2017surveyPbRL}
Christian Wirth, Riad Akrour, Gerhard Neumann, and Johannes F{{\"u}}rnkranz.
\newblock A survey of preference-based reinforcement learning methods.
\newblock {\em Journal of Machine Learning Research}, 18:1--46, 2017.

\bibitem{Xu2023ProvablyEA}
Tian Xu, Ziniu Li, Yang Yu, and Zhimin Luo.
\newblock Provably efficient adversarial imitation learning with unknown transitions.
\newblock In {\em Conference on Uncertainty in Artificial Intelligence 39 (UAI)}, volume 216, pages 2367--2378, 2023.

\bibitem{yang2019sampleoptimal}
Lin Yang and Mengdi Wang.
\newblock Sample-optimal parametric q-learning using linearly additive features.
\newblock In {\em International Conference on Machine Learning 36 (ICML)}, volume~97, pages 6995--7004, 2019.

\bibitem{zhao2023inverse}
Lei Zhao, Mengdi Wang, and Yu~Bai.
\newblock Is inverse reinforcement learning harder than standard reinforcement learning?
\newblock In {\em International Conference on Machine Learning 41 (ICML)}, 2024.

\bibitem{zhou2021nearly}
Dongruo Zhou, Quanquan Gu, and Csaba Szepesvari.
\newblock Nearly minimax optimal reinforcement learning for linear mixture markov decision processes.
\newblock In {\em Conference on Learning Theory 34 (COLT)}, pages 4532--4576, 2021.

\bibitem{ziebart2008maximum}
Brian~D. Ziebart, Andrew~L. Maas, J.~Andrew Bagnell, and Anind~K. Dey.
\newblock Maximum entropy inverse reinforcement learning.
\newblock In {\em AAAI Conference on Artificial Intelligence 23 (AAAI)}, pages 1433--1438, 2008.

\end{thebibliography}

\appendix

\newpage

\section{Related Works}\label{section: additional related works}

In this appendix, we report and describe the literature that
most relates to this paper. Theoretical works concerning the online IRL
problem can be grouped in works that concern the feasible set, and works that do
not.

Let us begin with works related to the feasible set.
While the notion of feasible set has been introduced implicitly in
\cite{ng2000algorithms}, the first paper that analyses the sample complexity of
estimating the feasible set in online IRL is \cite{metelli2021provably}. Authors
in \cite{metelli2021provably} adopt the simple generative model in
tabular MDPs, and devise two sample efficient algorithms.
\cite{lindner2022active} focuses on the same problem as
\cite{metelli2021provably}, but adopts a forward model in tabular MDPs. By
adopting RFE exploration algorithms, they devise sample efficient algorithms.
However, as remarked in \cite{zhao2023inverse}, the learning framework
considered in \cite{lindner2022active} suffers from a major issue.
\cite{metelli2023towards} builds upon \cite{metelli2021provably} to construct
the first minimax lower bound for the problem of estimating the feasible set
using a generative model. The lower bound is in the order of
$\Omega\big(\frac{H^3SA}{\epsilon^2}(S+\log\frac{1}{\delta})\big)$, where $S$
and $A$ are the cardinality of the state and action spaces, $H$ is the horizon,
$\epsilon$ is the accuracy and $\delta$ the failure probability. In addition,
\cite{metelli2023towards} develops US-IRL, an efficient algorithm whose sample
complexity matches the lower bound. \cite{poiani2024inverse} analyze a setting
analogous to that of \cite{metelli2023towards}, in which there is availability
of a single optimal expert and multiple suboptimal experts with known suboptimality.
\cite{lazzati2024offline} analyse the problem of estimating the feasible set
when no active exploration of the environment is allowed, but the learner is
given a batch dataset collected by some behavior policy $\pi^b$. Interestingly,
\cite{lazzati2024offline} focuses on two novel learning targets that are suited
for the offline setting, i.e., a subset and a superset of the feasible set.
Authors in \cite{lazzati2024offline} demonstrate that such sets are the tightest
learnable subset and superset of the feasible set, and propose a pessimistic
algoroithm, PIRLO, to estimate them. \cite{zhao2023inverse} analyses the same
offline setting as \cite{lazzati2024offline}, but instead of focusing on the
notion of feasible set directly, it considers the notion of reward mapping, which
considers reward functions as parametrized by their value and advantage
functions, and whose image coincides with the feasible set.

With regards to online IRL works that do not consider the feasible set, we
mention \cite{lopes2009active}, which analyses an active learning framework for
IRL. However, \cite{lopes2009active} assumes that the transition model is known,
and its goal is to estimate the expert policy only. Works
\cite{komanduru2019correctness} and \cite{komanduru2021lowerbound} provide,
respectively, an upper bound and a lower bound to the sample complexity of IRL
for $\beta$-strict separable problems in the tabular setting. However, both the
setting considered and the bound obtained are fairly different from ours. 
Analogously, \cite{dexter2021IRL} provides a sample efficient IRL algorithm for
$\beta$-strict separable problems with continuous state space. However, their
setting is different from ours since they assume that the system can be modelled
using a basis of orthonormal functions.

\subsection{Additional Related Works}
In this section, we collect additional related works that deserve to be mentioned.

\paragraph{Identifiability and Reward Learning.}
As aforementioned, the IRL problem is ill-posed, thus, to retrieve a single
reward, additional constraints shall be imposed. 
\cite{amin2016resolving} analyses the setting in which demonstrations of an
optimal policy for the same reward function are provided across environments
with different transition models. In this way, authors can reduce the
experimental unidentifiability, and recover the state-only reward function.
\cite{cao2021identifiability} and \cite{kim2020domain} concern reward
identifiability but in entropy-regularized MDPs
\cite{ziebart2008maximum,Fu2017LearningRR}. Such setting is in some sense easier
than the common IRL setting, because entropy-regularization permits a unique optimal policy for any reward function. \cite{cao2021identifiability}
uses expert demonstrations from multiple transition models and multiple discount
factors to retrieve the reward function, while \cite{kim2020domain} analyses
properties of the dynamics of the MDP to increase the constraints.
With regards to the more general field of Reward Learning (ReL), we mention
\cite{jeon2020rewardrational}, which introduces a framework that formalizes the
constraints imposed by various kinds of human feedback (like demonstrations or
preferences \cite{wirth2017surveyPbRL}). Intuitively, multiple
feedbacks about the same reward represent additional constraints beyond mere
demonstrations. \cite{skalse2023invariance} characterizes the partial
identifiability of the reward function based on various reward learning data
sources.

\paragraph{Linear MDPs and Extensions.}
As explained for instance in \cite{jin2020provablyefficient}, since lower bounds
to the sample complexity of various RL tasks in tabular MDPs depend explicitly
on the cardinality state space $S$, then we need to add structure to the problem if we
want to develop efficient algorithms that scale to large state spaces. For this
reason, the works \cite{yang2019sampleoptimal,jin2020provablyefficient} analyze the
Linear MDP model, which enforces some linearity constraints to the common MDP
model. In this way, authors are able to provide efficient algorithms for
RL in problems with large/continuous state spaces. However, there are other
settings beyond Linear MDPs that are analysed in the RL literature.
\cite{jiang2017contextual} introduces the notion of Bellman rank as complexity
measure, and provides a sample efficient algorithm for problems with small
Bellman rank. \cite{wang2020general} analyzes general value function
approximation when the function class has a low eluder dimension.
\cite{jin2021eluder} generalizes both the eluder dimension and Bellman
rank complexity measures by defining the Bellman eluder dimension and providing a provably
efficient algorithm. \cite{du2021bilinear} introduces bilinear classes, a
structural framework that, among the others, generalizes Linear MDPs.

\paragraph{Reward-Free Exploration (RFE) in Tabular and Linear MDPs.} The RFE
problem was introduced in \cite{jin2020RFE}, where authors provided a sample
efficient algorithm and a lower bound for tabular MDPs. Later on, the
state-of-the-art sample-efficient algorithms for RFE in tabular MDPs have been
developed in \cite{kaufmann2021adaptiveRFE,menard2021fast,li2023minimaxoptimal}.
It should be remarked that RFE requires more samples than common RL in tabular
MDPs. \cite{wang2020onrewardfree} proposes a sample efficient algorithm for RFE
in linear MDPs. \cite{wagenmaker2022noharder} improves the algorithm of
\cite{wang2020onrewardfree} and, interestingly, demonstrates that RFE is no harder
than RL in Linear MDPs.

\paragraph{Online Apprenticeship Learning (AL).}
The first works that provide a theoretical analysis of the AL setting when
the transition model is unknown are \cite{abbeel2005exploration,syed2007game}.
Recently, \cite{Shani2021OnlineAL} formulates the online AL problem, which
closely resembles the online IRL problem. The main difference is that in online
AL the ultimate goal is to imitate the expert, while in IRL is to recover a
reward function. \cite{Xu2023ProvablyEA} improves the results in
\cite{Shani2021OnlineAL} by combining an RFE algorithm with an efficient algorithm for the estimation of the visitation distribution of the
deterministic expert's policy in tabular MDPs, presented in
\cite{Rajaraman2020TowardTF}. We mention
also \cite{rajaraman2021value,swamy2022minimax} for the sample complexity of
estimating the expert's policy in problems with linear function approximation.
In the context of Imitation Learning from Observation alone (ILfO)
\cite{liu2018imitation}, the work \cite{sun2019provably} proposes a probably
efficient algorithm for large-scale MDPs with unknown transition model.
\cite{liu2022learning} provides an efficient AL algorithm based on GAIL
\cite{ho2016gail} in Linear Kernel Episodic MDPs \cite{zhou2021nearly} with
unknown transition model.

\paragraph{Others.} We mention work \cite{klein2012IRLclassification}, which
considers a classification approach for IRL. However, this is fairly different
from our IRL problem formulation in Section \ref{section: rewards compatibility}.

\section{Additional Results and Proofs for Section \ref{section: extension framework past works}}\label{section: remove assumption expert policy}

In this section, we provide additional results beyond those presented in Section
\ref{section: extension framework past works}, and then we report the missing
proofs. Specifically,
in Appendix \ref{appendix: examples degenerate fs}, we
provide two numerical examples that explain Proposition \ref{prop:
degeneratefs},
in Appendix \ref{section: additional regularity} we show
that some additional regularity assumptions beyond the Linear MDP cannot remove
the dependence on the cardinality of the state space in the sample complexity. In
Appendix \ref{section: algorithm linear fs appendix}, we report and describe the
sample efficient algorithm mentioned in Section \ref{section: extension
framework past works}, while in Appendix \ref{section: missing proofs} we
collect all the missing proofs of this section.

\subsection{Some Examples for Proposition \ref{prop: degeneratefs}}\label{appendix: examples degenerate fs}

The following examples aim to explain Proposition \ref{prop: degeneratefs} in a
simple manner.

\begin{example}[Non-degenerate feasible set]
    Let $\M=\tuple{\S,\A,H,d_0}\cup\{\pie\}$ be an IRL instance such that
    $\S=\{s_1,s_2\},\A=\{a_1,a_2\},H=1,d_0(s_1)=d_0(s_2)=1/2,\pie(s_1)=\pie(s_2)=a_1$.
    Consider the feature mapping $\phi_1$ s.t. $\phi_1(s,a)=\indic{a=a_1}$ for
    all $s\in\S$.
    Then, we have $\Phi^{\pie}=\{1\}$ and $\overline{\Phi}=\{0\}$. Clearly,
    these sets can be separated by any hyperplane $w\in\RR_{>0}$, since $1\cdot
    w > 0\cdot w$, and so $\R_{p,\pie}\neq \{\overline{r}\}$, with
    $\overline{r}_h(s,a)=0$ $\forall (s,a,h)\in\SAH$. Actually,
    $\R_{p,\pie}=\{r\in\mathfrak{R}\,|\,\exists\theta\in(0,1]:\,r_1(s,a)=\dotp{\phi(s,a),\theta}\,\forall
    (s,a)\in\SA\}$.
\end{example}

\begin{example}[Degenerate feasible set]
    Consider the same IRL instance as in the previous example, but this time
    consider the feature mapping $\phi_2$ s.t. $\phi_2(s_1,a)=\indic{a=a_1}$, and
    $\phi_2(s_2,a)=\indic{a=a_2}$. Then, we have $\Phi^{\pie}=\{0,1\}$ and
    $\overline{\Phi}=\{0,1\}$. Clearly, the two sets coincide, thus they cannot
    be separated, and $\R_{p,\pie}= \{\overline{r}\}$, with
    $\overline{r}_h(s,a)=0$ $\forall (s,a,h)\in\SAH$.    
\end{example}

\subsection{Additional Regularity Assumptions of the State Space do not Make the Problem Learnable}\label{section: additional regularity}

In tabular MDPs with small state space $\S$, collecting samples from every state
$s\in\S$ is feasible, and it is exactly what previous works do:
\begin{itemize}
    \item Under the assumption that $\pie$ is deterministic,
    \cite{metelli2021provably,metelli2023towards} collect one sample from every
    $(s,h)\in\SH$ using a generative model, obtaining $\pie$ exactly.
    \item If $\pie$ is stochastic, under the
    assumption that all actions in the support of the expert's policy are played
    with probability at least $\pi_{\min}$ (see Assumption D.1 of
    \cite{metelli2023towards}), both \cite{metelli2023towards,zhao2023inverse} are
    able to learn the support of $\pie$ exactly w.h.p. using $\propto
    1/\pi_{\min}$ samples in the online setting.\footnote{Actually,
    \cite{zhao2023inverse} makes use of a
    concentrability assumption too.}
    \item In the offline setting, assuming that the occupancy measure of the
    expert's policy is at least $d_{\min}$ in all reachable $(s,a)\in\SA$, then
    \cite{lazzati2024offline} learns the support of $\pie$ exactly w.h.p. using
    $\propto 1/d_{\min}$ episodes.
\end{itemize}
However, when $\S$ is large, even under the Linear MDP assumption, this is not possible.
In Section \ref{section: extension framework past works}, we have formalized
this fact with the following proposition:
\propnoass*
Theorem \ref{prop: fs not learnable 0 assumptions} tells us that the Linear
MDP assumption is too weak for the feasible set to be learnable using the PAC
framework of Definition \ref{def: pac framework bad} with a number of samples
independent of the cardinality of the state space. Therefore, we can try to introduce
an additional assumption on the structure of the IRL problem $\M\cup\{\pie\}$
and see whether it helps in alleviating the issue. Let us consider the following first assumption.
\begin{ass}\label{ass: lipschitz features only}
    We assume a Lipschitz continuity property between features and states:
    \begin{align*}
        \forall (s,a,s')\in\SA\times\S:\quad\|\phi(s,a)-\phi(s',a)\|_2\le L\|s-s'\|,
    \end{align*}
    for some $L>0$ and some distance $\|\cdot-\cdot\|$ in $\S$.
\end{ass}
The intuition is that, based on the fact that in Linear MDPs the $Q$-function of
any policy $\pi$ is linear in the feature mapping
$Q^{\pi}_h(\cdot,\cdot)=\tuple{\phi(\cdot,\cdot),w^{\pi}_h}$ for some parameter
vector $w_h^\pi\in \RR^d$ (see \cite{jin2020provablyefficient}), then if we are
able to $\epsilon$-cover the state space $\S$, we can approximate the
$Q$-function $Q^{\pi}_h(s,\cdot)$ in any $s\in\S$ with
the $Q$-function $Q_h^{\pi}(s',\cdot)$ of the closest point $s'$ int the covering, so that
$|Q_h^{\pi}(s,a)-Q_h^{\pi}(s',a)|=|(\phi(s,a)-\phi(s',a))^\intercal
w_h^{\pi}|\le \|\phi(s,a)-\phi(s',a)\|_2\|w_h^{\pi}\|_2\le
L\epsilon\|w_h^{\pi}\|_2$.
However, this assumption is not sufficient.
\begin{restatable}{prop}{proponlyfeaturesass}
    Under the setting of Proposition \ref{prop: fs not learnable 0 assumptions},
    even under Assumption \ref{ass: lipschitz features only}, then an
    algorithm is $(\epsilon,\delta)$-PAC only if $\tau^E=\Omega(S)$.
\end{restatable}
Assumption \ref{ass: lipschitz features only} fails because it does not provide
any information about how the knowledge of the expert's policy at a state can be
``transferred'' to other states, and thus we still need to sample almost all the
states of $\S^{p,\pie}$ to get an acceptable feasible set.

We devise another assumption to attempt to fix this issue.
\begin{ass}\label{ass: lipschitz features and pie}
    We assume the following Lipschitz continuity property:
    \begin{align*}
        \forall (s,s')\in\S\times\S:\quad\|\phi(s,\pi^E_h(s))-\phi(s',\pi^E_h(s'))\|_2\le L\|s-s'\|,
    \end{align*}
    for some $L>0$ and some distance $\|\cdot-\cdot\|$ in $\S$.
\end{ass}
This assumption says that states that are close to each other
cannot have the features corresponding to the expert's action too far away from
each other. From a high-level point of view, it says that the features are
``somehow'' regular with $\pie$, so that when the expert lies in $s'$ which is
really close to $s$, then she plays an action which has the same ``effect''
(i.e., same transition model and same reward, due to the Linear MDP assumption)
as the expert's action in $s$.

Assumption \ref{ass: lipschitz features and pie} is not comparable with Assumption
\ref{ass: lipschitz features only} since, on the one hand, it does not hold for all actions in
$\A$, but only for those corresponding to $\pie$, but, on the other hand, provides information on how to transfer knowledge about $\pie$ to neighbor
states.

Let $\Delta'\coloneqq\min_{s\in\S,a,a'\in\A:
\phi(s,a)\neq\phi(s,a')}\|\phi(s,a)-\phi(s,a')\|_2$, i.e., the smallest non-zero
distance between the features of different actions. Clearly, when $\S$ is finite,
since in Linear MDPs also $A\coloneqq|\A|$ is finite, then $\Delta'$ is finite too. So we can
define a new quantity $\Delta$ to be any number $0<\Delta<\Delta'$.
\begin{restatable}{prop}{propfeaturespolicyass}\label{prop: features policy ass}
    Under the setting of Proposition \ref{prop: fs not learnable 0 assumptions},
    under Assumption \ref{ass: lipschitz features and pie}, then a number of samples
    $\tau^E=|\N(\frac{\Delta}{2L};\S,\|\cdot\|)|$ is sufficient to recover $\pie$
    exactly in any $(s,h)\in\S$, where $|\N(\frac{\Delta}{2L};\S,\|\cdot\|)|$ is
    the $\Delta/(2L)$-covering number of space $\S$  w.r.t. distance $\|\cdot\|$.
\end{restatable}
Intuitively, by constructing a covering with a sufficiently small radius in the state space
$\S$, then we are able to retrieve the exact expert's action in the neighborood
of each state of the covering. Doing so, we are able to construct
$\epsilon$-correct estimates of the feasible set. Of course, this is possible as
long as $\Delta'$ is not too small, and $L$ is not too large. When $\S$ is
infinitely large or continuous, it might be possible to construct feature
mappings in which $\Delta'\to0$, and so the approach would still require too
many samples.

However, even for cases with finite and not too small $\Delta'$, the result in
Proposition \ref{prop: features policy ass} is not satisfactory, because it just
allows to retrieve $\pie$ under a stronger assumption than Linear MDPs, but not
to perform an interesting learning process. We observe that the feasible set is
an ``unstable'' concept, in the sense that, based on Proposition \ref{prop:
degeneratefs}, changing the expert action in a single state might reduce the
feasible set from a continuum of rewards to a singleton, or vice versa.

\begin{remark}
    If we want to be able to recover the exact feasible set efficiently, we need to
    recover the exact expert's policy almost everywhere.
\end{remark}

\subsection{Algorithm}\label{section: algorithm linear fs appendix}

By exploiting an RFE algorithm as sub-routine like that of Algorithm 1 in
\cite{wang2020onrewardfree} or Algorithm 1 in \cite{wagenmaker2022noharder}, we
are able to construct estimates of the transition model $\widehat{p}$, that
can be used to compute an ``empirical'' estimate of the feasible set
$\widehat{\R}\approx\R_{\widehat{p},\pie}$ (since $\phi$ and $\pie$ are known).
The algorithm is presented in Algorithm \ref{alg: known pie}.
\SetKwComment{Comment}{/* }{ */}
\begin{algorithm}[!h]
\caption{IRL for Linear MDPs (known expert's policy)}\label{alg: known pie}
\DontPrintSemicolon
\KwData{failure probability $\delta>0$, error tolerance $\epsilon>0$, expert
policy $\pie$, all sets $\Z\subseteq \SH$ that coincide with $\S^{p,\pie}$
almost everywhere based on measure $d^{p,\pie}$}
$\D\gets \text{RFE\_Exploration}(\delta,\epsilon)$\Comment*[r]{Various choices}
\For{$h$ in $\{H,H-1,\dotsc,2,1\}$}{
$\Lambda_h\gets I +\sum\limits_{k=1}^\tau
\phi(s_h^k,a_h^k)\phi(s_h^k,a_h^k)^\intercal$\;
$\widehat{\mu}_h(\cdot)\gets \Lambda^{-1}_h 
\sum\limits_{k=1}^\tau \phi(s_h^k,a_h^k)\delta(\cdot,s_{h+1}^k)$
}
$\widehat{p}_h(\cdot|s,a)\gets \dotp{\phi(s,a),\widehat{\mu}_h(\cdot)}$ for all $(s,a,h)\in\SAH$\;
$\widehat{\R}\gets\big\{\widehat{r}\in\mathfrak{R}\,\Big|\,\exists \Z,\forall (s,h)\in\Z,\forall a\in\A:\;
\E\limits_{a'\sim\pi^E_h(\cdot|s)}Q^*_h(s,a';\widehat{p},\widehat{r})
\ge Q^*_h(s,a;\widehat{p},\widehat{r})\big\}$\;
Return $\widehat{\R}$
\end{algorithm}

Simply put, Algorithm \ref{alg: known pie} uses the dataset collected by an RFE
algorithm to compute a least-squares estimate of the transition model
$\widehat{p}$, and then it returns the feasible set defined according to it
(recall that $\phi$ and $\pie$ are known). Notice that this algorithm cannot
be implemented in practice due to various reasons, like the presence of the Dirac delta
$\delta$ measure in the definition of some quantities (see Appendix~\ref{section: proof thr bound tabular}), and the fact that the feasible set is, potentially, a
set containing infinite rewards. Nevertheless, Theorem \ref{thr: upper bound linear old framework}
states that this algorithm is sample efficient. The proof of the theorem is
provided in Appendix \ref{section: proof thr bound tabular}.

It should be remarked that Algorithm \ref{alg: known pie} takes in input also
the true support of the visit distribution of the expert policy $\S^{p,\pie}$
in case $\S$ is finite, and all the possible sets $\Z$ that agree with
$\S^{p,\pie}$ a.e. based on the measure $d^{p,\pie}$ in case $\S$ is infinite.
Intuitively, this set ($\S^{p,\pie}$) of $(s,h)$ pairs represents the domain in
which $\pie$ is defined. Indeed, since the expert in the true problem $p$ never
visits pairs $(s',h')\notin\S^{p,\pie}$, its expert policy might reasonably be
non well-defined there. When $\S$ is infinite, we require all sets $\Z$ because
otherwise we cannot know which are the sets $\S^{p,\pie}\setminus\Z$ with zero
measure, i.e., in which the reward can induce an optimal action different from
the expert's one, since the overall contribution to the expected return is zero.

The proof of Theorem \ref{thr: upper bound linear old framework} is obtained by 
using Algorithm 1 of \cite{wagenmaker2022noharder} at Line 1 of Algorithm
\ref{alg: known pie}.
In Appendix \ref{section: proof thr bound tabular}, we demonstrate an upper
bound also if we use Algorithm 1 in
\cite{wang2020onrewardfree}.

\subsection{Missing Proofs}\label{section: missing proofs}

Before diving into the proofs, we recall some important properties of the
feasible set and of the Linear MDPs that will be useful in the proofs.
First, we provide an explicit form for the feasible set presented at Definition
\ref{def: fs implicit}.
\begin{lemma}[Lemma E.1 in
\cite{lazzati2024offline}]\label{lemma: fs explicit finite S}
In the setting of Definition \ref{def: fs implicit}, if $\S$ is finite, then the
feasible set
    $\R_{p,\pie}$ satisfies:
    \begin{align*}
        \R_{p,\pie}=\Big\{r\in\mathfrak{R}\,\Big|\,\forall (s,h)\in\S^{p,\pie},\forall a\in\A:\;
        \E\limits_{a'\sim\pi^E_h(\cdot|s)}Q^*_h(s,a';p,r)
        \ge Q^*_h(s,a;p,r)
        \Big\}.
    \end{align*}
\end{lemma}
Notice that we have extended Lemma E.1 in \cite{lazzati2024offline} to
consider stochastic expert policies (the extension is trivial).
We can easily extend it to problems with large/continuous $\S$.
\begin{lemma}[Feasible Set Explicit]\label{lemma: fs explicit}
    In the setting of Definition \ref{def: fs implicit}, then the
feasible set
    $\R_{p,\pie}$ satisfies:
    \begin{align*}
        \R_{p,\pie}=&\Big\{r\in\mathfrak{R}\,\Big|\,
        \forall h\in\dsb{H}, \exists \overline{\S}\subseteq 
        \S_h^{p,\pie}: d^{p,\pie}_h(\overline{\S})=0\wedge
        \forall s\notin\overline{\S},\forall a\in\A:\\
        &\qquad\E\limits_{a'\sim\pi^E_h(\cdot|s)}Q^*_h(s,a';p,r)
        \ge Q^*_h(s,a;p,r)
        \Big\}.
    \end{align*}
\end{lemma}
Simply, Lemma \ref{lemma: fs explicit} improves on Lemma \ref{lemma: fs explicit
finite S} by allowing the reward to enforce the ``wrong'' action (i.e.,
different from the expert's action) in a subset with zero measure based on the
visitation distribution.
\begin{proof}
    The proof is completely analogous to that of Lemma E.1 in
    \cite{lazzati2024offline}. We just need to observe that if set
    $\overline{\S}$ has zero measure (and the set of rewards $\mathfrak{R}$ contains bounded rewards), then it does not affect the expected return.
\end{proof}
Another useful property that we need is that the $Q$-function is always linear
in the feature map for any policy in Linear MDPs.
\begin{prop}[Proposition 2.3 in \cite{jin2020provablyefficient}]\label{prop:
linear Q function in linear MDPs}
    For a Linear MDP, for any policy $\pi$, there exist weights
    $\{w_h^\pi\}_{h\in\dsb{H}}$ such that, for any $(s,a,h)\in\SAH$, we have
    $Q^\pi_h(s,a)=\dotp{\phi(s,a),w_h^\pi}$.
\end{prop}
We can combine the results of Lemma \ref{lemma: fs explicit} and Proposition
\ref{prop: linear Q function in linear MDPs} to obtain the following
characterization of the feasible set in Linear MDPs.
\begin{lemma}
    In the setting of Definition \ref{def: fs implicit}, the feasible set
    $\R_{p,\pie}$ satisfies:
    \begin{align*}
        \R_{p,\pie}=\Big\{
        r\in&\mathfrak{R}\,\Big|\,
        \exists \{w_h\}_{h\in\dsb{H}},\forall (s,a,h)\in\SAH:\,
        r_h(s,a)=\dotp{\phi(s,a),\theta_h}\\
        &\wedge\forall h\in\dsb{H}, \exists \overline{\S}\subseteq 
        \S_h^{p,\pie}: d^{p,\pie}_h(\overline{\S})=0\wedge
        \forall s\notin\overline{\S},\forall a^E\in\A^E_h(s):\\
        &\dotp{\phi(s,a^E),w_h}=\max\limits_{a\in\A}\dotp{\phi(s,a),w_h}
        \Big\},
    \end{align*}
    where $\theta_h\coloneqq w_h-
    \int_{\S}\max_{a'\in\A}\dotp{\phi(s',a'),w_{h+1}}d\mu_h(s')$ for all
    $h\in\dsb{H}$, and $\A^E_h(s)\coloneqq\{a\in\A|\pi^E_h(a|s)>0\}$.
\end{lemma}
\begin{proof}
    From \cite{puterman1994markov}, we know that in any MDP there exists an
    optimal policy. Therefore, thanks to Proposition \ref{prop: linear Q
    function in linear MDPs}, we know that the optimal $Q$-function $Q^*$ is
    linear in the feature map too. So, there exist parameters $\{w_h\}_h$
    such that, for any $(s,a,h)\in\SAH$, the optimal $Q$-function can be
    rewritten as $Q^*_h(s,a)=\dotp{\phi(s,a),w_h}$. From the Bellman equation,
    we know that:
    \begin{align*}
        Q^*_h(s,a;p,r)&=r_h(s,a)+\int\limits_{\S}V^*_{h+1}(s';p,r)dp_h(s'|s,a)\\
        &=\dotp{\phi(s,a),\theta_h}+\dotp{\phi(s,a),\int\limits_{\S}
        \max\limits_{a'\in\A}\dotp{\phi(s',a'),w_{h+1}}d\mu_h(s')}.
    \end{align*}
    By rearranging this equation, and removing the dot product with $\phi(s,a)$,
    we obtain that:
    \begin{align*}
        \theta_h= w_h-\int_{\S}\max_{a'\in\A}\dotp{\phi(s',a'),w_{h+1}}d\mu_h(s').
    \end{align*}
    Now, this holds in any Linear MDP. If we desire to enforce the constraints
    in Lemma \ref{lemma: fs explicit}, we simply have to impose the constraint
    on the optimal $Q$-function using parameters $\{w_h\}_h$ outside some
    $\overline{\S}$. This concludes the proof.
\end{proof}
It is useful to introduce the following definitions. First we define the set of
parameters that induce a $Q$-function compatible with $\pie$:
\begin{align*}
    \W_{p,\pie}\coloneqq\Big\{
    w:\dsb{H}\rightarrow  \RR^d\,\Big|\,&
    \forall h\in\dsb{H}, \exists \overline{\S}\subseteq 
        \S_h^{p,\pie}: d^{p,\pie}_h(\overline{\S})=0\wedge
        \forall s\notin\overline{\S},\forall a^E\in\A^E_h(s):\\
    &\dotp{\phi(s,a^E),w_h}=\max\limits_{a\in\A}\dotp{\phi(s,a),w_h}
    \Big\}.
\end{align*}
Next, we define the set of parameters of the reward function obtained by using
$Q$-functions parametrized by $w\in\W_{p,\pie}$:
\begin{align*}
    \Theta_{p,\pie}\coloneqq\Big\{
    \theta:\dsb{H}\rightarrow  \RR^d\,\Big|\,
    \exists \{w_h\}_h\in\W_{p,\pie}:\;
    \theta_h= w_h-\int_{\S}\max_{a'\in\A}\dotp{\phi(s',a'),w_{h+1}}d\mu_h(s')
    \Big\}.
\end{align*}
Irrespective of the transition model $\{\mu_h\}_h$ and the feature map $\phi$,
we see that it is always possible to construct a surjective map from
$\Theta_{p,\pie}$ to $\W_{p,\pie}$ (the map in the definition of
$\Theta_{p,\pie}$).
Thanks to these definitions, the feasible set can be rewritten as:
\begin{align*}
    \R_{p,\pie}=\{r\in\mathfrak{R}\,|\,\exists \{\theta_h\}_h\in\Theta_{p,\pie},
    \forall (s,a,h)\in\SAH:\,r_h(s,a)=\dotp{\phi(s,a),\theta_h}
    \}.
\end{align*}
We are now ready to provide the proofs of the various results of this section.

\subsubsection{Proof of Proposition \ref{prop: degeneratefs}}

\degeneratefs*
\begin{proof}
    From \cite{bertsekas2009convex}, we recall that two sets $\Y_1,\Y_2$ are
    separated by a hyperplane $H=\{x|a^\intercal x = b\}$ if each lies in a
    different closed halfspace associated with $H$, i.e., if either:
    \begin{align*}
        a^\intercal y_1\le b\le a^\intercal y_2,\quad\forall y_1\in\Y_1,\forall y_2\in\Y_2,
    \end{align*}
    or:
    \begin{align*}
        a^\intercal y_2\le b\le a^\intercal y_1,\quad\forall y_1\in\Y_1,\forall y_2\in\Y_2.
    \end{align*}
    By definition of $\W_{p,\pie}$, for each stage $h\in\dsb{H}$, we are
    looking for vectors $w_h\in \RR^d$ such that $\forall
    (s,h)\in\S^{p,\pie}$, it holds that:
    \begin{align*}
        w_h^\intercal\phi(s,a)\le w_h^\intercal\phi(s,a^E)\quad
        \forall a^E\in\A^E_h(s),\forall a\in\A\setminus\A^E_h(s).
    \end{align*}
    In words, for each $(s,h)\in\S^{p,\pie}$, we are looking for non-affine
    separating hyperplanes between features of expert and non-expert actions.
    However, since the hyperplane parameter $w_h$ is common to all states
    $s\in\S^{p,\pie}_h$, then it must separate expert from non-expert actions
    at all states. This is equivalent to finding the separating hyperplanes to
    the sets $\Phi_h^{\pie}$ and $\overline{\Phi}_h$ which contain all the
    points. Clearly, when the separating hyperplanes do not exist at all
    $h\in\dsb{H}$, then the condition in $\W_{p,\pie}$ is satisfied by the zero
    vector alone. As a consequence, set $\Theta_{p,\pie}$ contains only the
    zero vector, and so does $\R_{p,\pie}$.
\end{proof}
\begin{remark}
    By using the result of Lemma \ref{lemma: fs explicit}, we can easily convert
    Proposition \ref{prop: degeneratefs} into a more general result by
    considering the impossibility of separating any pair of sets constructed by
    varying at will some subsets with zero measure. We will not provide such
    result explicitly.
\end{remark}

\subsubsection{Proofs of Proposition \ref{prop: fs not learnable 0 assumptions} and Appendix \ref{section: additional regularity}}

In the PAC framework of Definition \ref{def: pac framework bad}, we have not
specified formally the inner distance $d$:
\begin{align}\label{eq: inner distance d}
    d(r,\widehat{r})\coloneqq \frac{1}{M_{r,\widehat{r}}}\sup_{\pi \in \Pi} \sum_{h\in\dsb{H}} \E_{(s,a)\sim
    d^{p,\pi}_h(\cdot,\cdot)}|r_h(s,a)-\widehat{r}_h(s,a)|,
\end{align}
where:
\begin{align*}
    M_{r,\widehat{r}}\coloneqq
    \max\{\sqrt{d},\max_{h \in \dsb{H}}\|\theta_h\|_2,\max_{h \in \dsb{H}}\|\widehat{\theta}_h\|_2\}/\sqrt{d},
\end{align*}
    where $\{\theta_h\}_h$ and $\{\widehat{\theta}_h\}_h$ are the (unbounded)
    parameters of rewards $r$ and $\widehat{r}$. As explained in
    \cite{lazzati2024offline}, such normalization term allows us to work with unbounded reward
    functions. In practice, we are relaxing the Linear MDP assumption presented
    in Section \ref{section: preliminaries} about the boundedness of the
    parameters $\theta$ of the rewards to avoid the issue described in
    \cite{metelli2023towards} and \cite{lazzati2024offline}. We still assume
    that the feature mapping is bounded. Observe that this relaxation does
    \emph{not} affect the results we present, which would hold even if we
    considered bounded parameters $\theta$. Indeed, as visible in the proofs, the instances do not need to be constructed with unbounded $\theta$.

\propnoass*
\begin{proof}
    We construct two problem instances that lie at a finite Hausdorff distance,
    and show that, with less than $S$ calls to the sampling oracle, we are
    not able to discriminate between the two instances.

    Let $\S$ be the finite state space with cardinality $S$,
    $\A=\{a_1,a_2\},H=1, d_0(s)=1/S \;\forall s\in\S$, $\phi(s,a)=\indic{a=a_1}$, and consider
    two deterministic expert's policies $\pi^E_1(s)=a_1\,\forall s\in\S$, and
    $\pi^E_2(s)=a_1\,\forall s\in\S\setminus\{\overline{s}\}$, and
    $\pi^E_2(\overline{s})=a_2$, for a certain $\overline{s}\in\S$.    
    The set of parameters compatible with $\pi^E_1$ is:
    \begin{align*}
        \Theta_{p,\pi^E_1}=\{\theta\in \RR\,|\,\theta\ge 0\},
    \end{align*}
    since $Q^{\pi^E_1}(s,a_1)\ge Q^{\pi^E_1}(s,a_2)\iff r(s,a_1)\ge r(s,a_2)\iff
    \phi(s,a_1)\theta\ge\phi(s,a_2)\theta\iff 1\cdot\theta\ge 0\cdot\theta$.
    Observe that, for $\pi^E_2$, due to the presence of $\overline{s}$, we have:
    \begin{align*}
        \Theta_{p,\pi^E_2}=\{\theta\in \RR\,|\,\theta=0\},
    \end{align*}
    since $\overline{s}$ imposes $\theta\le0$, and the other states impose
    $\theta\ge0$.
    
    Therefore, the Hausdorff distance between the two problems is:
    \begin{align*}
        \mathcal{H}(\R_{\pi^E_1},\R_{\pi^E_2})&=
        \sup\limits_{\theta\ge0}\frac{1}{\max\{1,\theta,0\}}\theta
        =\sup\limits_{\theta\ge0}\frac{1}{\max\{1,\theta\}}\theta=1.
    \end{align*}
    Obviously, we need a $\Omega(S)$ samples to
    spot, if it exists, state $\overline{s}$, and thus distinguish between
    $\R_{\pi^E_1}$ and $\R_{\pi^E_2}$.
\end{proof}

\proponlyfeaturesass*
\begin{proof}
    The same proof of Proposition \ref{prop: fs not learnable 0 assumptions}
    works here.
    
    In particular, we now show that Assumption \ref{ass: lipschitz features
    only} does not help. The Hausdorff distance between the instances in the
    proof of Proposition \ref{prop: fs not learnable 0 assumptions}
    can be written as:
    \begin{align*}
        \mathcal{H}(\R_{\pi^E_1},\R_{\pi^E_2})&=
        \sup\limits_{\theta_1\ge0}\inf\limits_{\theta_2 =0}\frac{1}{\max\{1,\theta_1,\theta_2\}}
        \sup\limits_{\pi\in\Pi}
        \E\limits_{s\sim d_0(\cdot),a\sim\pi(\cdot|s)}
        |r_1(s,a)-r_2(s,a)|\\
        &=\sup\limits_{\theta_1\ge0}\inf\limits_{\theta_2 =0}\frac{1}{\max\{1,\theta_1\}}
        \sup\limits_{\pi\in\Pi}
        \E\limits_{s\sim d_0(\cdot),a\sim\pi(\cdot|s)}
        |\phi(s,a)\theta_1-\phi(s,a)\theta_2 \\
        & \qquad\qquad\qquad\qquad\qquad\qquad\qquad\qquad \pm \phi(s',a)
        \theta_1\pm \phi(s',a)\theta_2|\\
        &\le\sup\limits_{\pi\in\Pi}\E\limits_{s\sim d_0(\cdot),a\sim\pi(\cdot|s)}|\phi(s,a)-\phi(s',a)|+0\\
        &\qquad+\sup\limits_{\pi\in\Pi}\sup\limits_{\theta_1\ge0}\frac{1}{\max\{1,\theta_1\}}
        \inf\limits_{\theta_2 =0}
        \E\limits_{s\sim d_0(\cdot),a\sim\pi(\cdot|s)}|\phi(s',a)
        \theta_1- \phi(s',a)\theta_2|\\
        &=\sup\limits_{\pi\in\Pi}\E\limits_{s\sim d_0(\cdot),a\sim\pi(\cdot|s)}|\phi(s,a)-\phi(s',a)|+
        \sup\limits_{\pi\in\Pi}\E\limits_{s\sim d_0(\cdot),a\sim\pi(\cdot|s)}\phi(s',a),
    \end{align*}
    where $s'$ is the state in the covering closest to state $s$; while the
    first term can be bounded, the assumption does not help us with the second term.
\end{proof}

\propfeaturespolicyass*
\begin{proof}
    For any state $s\in\S^{p,\pie}$, by definition of covering
    $\N(\frac{\Delta}{2L};\S,\|\cdot\|)$, there always exist another state
    $s'\in \N(\frac{\Delta}{2L};\S,\|\cdot\|)$ such that $\|s'-s\|\le
    \frac{\Delta}{2L}$.
    By Assumption \ref{ass: lipschitz features and pie} we know that:
    \begin{align*}
        \|\phi(s,\pi^E_h(s))-\phi(s',\pi^E_h(s'))\|_2\le L\|s'-s\|\le \frac{\Delta}{2},
    \end{align*}
    and since $\pi^E_h(s')$ and thus $\phi(s',\pi^E_h(s'))$ is known, then
    the fact that $\Delta$ is finite guarantees us that
    $\pi^E_h(s)$ is equal to the action $a$ that minimizes the distance to
    $\phi(s',\pi^E_h(s'))$.
    Notice that if, by contradiction, there were two actions $a_1,a_2$ with
    $\|\phi(s,a_1)-\phi(s',\pi^E_h(s'))\|_2\le\frac{\Delta}{2}$
    and $\|\phi(s,a_2)-\phi(s',\pi^E_h(s'))\|_2\le\frac{\Delta}{2}$,
    then by triangle inequality and finiteness of $\Delta$, we would have:
    \begin{align*}
        \Delta&<\|\phi(s,a_1)-\phi(s,a_2)\|_2\\
        &\le \|\phi(s,a_1)-\phi(s',\pi^E_h(s'))\|_2+\|\phi(s,a_2)-\phi(s',\pi^E_h(s'))\|_2\\
        &\le \frac{\Delta}{2}+\frac{\Delta}{2}=\Delta,
    \end{align*}
    which is clearly a contradiction.
\end{proof}

\subsubsection{Proof of Theorem \ref{thr: upper bound linear old framework}}\label{section: proof thr bound tabular}

The proof is based on deriving an upper bound to the Hausdorff distance
between the true feasible set and its estimate.
To do so, first, using the notation of \cite{jin2020provablyefficient}, let us define the
following quantities:
\begin{align*}
    \mathbb{P}_h(\cdot|s,a)&\coloneqq \dotp{\phi(s,a),\mu_h(\cdot)},\\
    \widehat{\mathbb{P}}_h(\cdot|s,a)&\coloneqq \phi(s,a)^\intercal\Lambda_h^{-1}
    \sum\limits_{k=1}^\tau\phi(s_h^k,a_h^k)\delta(\cdot,s_{h+1}^k),\\
    \overline{\mathbb{P}}_h(\cdot|s,a)&\coloneqq \phi(s,a)^\intercal\Lambda_h^{-1}
    \sum\limits_{k=1}^\tau\phi(s_h^k,a_h^k)\mathbb{P}_h(\cdot|s_h^k,a_h^k),\\
\end{align*}
where $\delta(\cdot,x)$ is the Dirac measure, and $(s_h^k,a_h^k)$ represents the
state-action pair visited at stage $h$ of exploration episode $k\in\dsb{\tau}$.
In words, $\mathbb{P}$ denotes the true transition model, $\widehat{\mathbb{P}}$
denotes the least squares estimate computed by Algorithm \ref{alg: known pie},
and $\overline{\mathbb{P}}$ represents a bridge between the two. As we will see,
the core of the proof consists in upper bounding the term
$\big|\big(\mathbb{P}_h-\widehat{\mathbb{P}}\big)V_{h+1}(s,a)\big|$ at all
$h\in\dsb{H}$ and reachable $(s,a)\in\SA$, for all the bounded linear functions
$V$ in class $\V$, defined as:
\begin{align}\label{eq: definition class V}
    \V\coloneqq\Big\{
    V:\SH\to[-H,+H]\,\Big|\,V(\cdot)=\max\limits_{a\in\A}
    \phi(\cdot,a)^\intercal w,\; \|w\|_2\le 2H\sqrt{d}
    \Big\}.
\end{align}
To achieve this goal, it will be useful to apply triangle inequality and to
bound the following two terms separately:
\begin{align*}
    \Big|\big(\mathbb{P}_h-\widehat{\mathbb{P}}\big)V_{h+1}(s,a)\Big|
    \le \Big|\big(\mathbb{P}_h-\overline{\mathbb{P}}\big)V_{h+1}(s,a)\Big|+
    \Big|\big(\overline{\mathbb{P}}_h-\widehat{\mathbb{P}}\big)V_{h+1}(s,a)\Big|.
\end{align*}
Lemma \ref{lemma: bound pbar p} and Lemma \ref{lemma: bound pbar phat}, which we
now present, serve exactly this purpose.
\begin{lemma}\label{lemma: bound pbar p}
    For any value function $V$ in the class $\V$, for any $(s,a,h)\in\SAH$, it holds that:
    \begin{align*}
        \Big|\Big(\overline{\mathbb{P}}_h-\mathbb{P}_h\Big)V_{h+1}(s,a)\Big|
        \le \min\Big\{H\sqrt{d}\|\phi(s,a)\|_{\Lambda_h^{-1}},2H\Big\}.
    \end{align*}
\end{lemma}
\begin{proof}
    We have:
\begin{align*}
    \Big(\overline{\mathbb{P}}_h-\mathbb{P}_h\Big)V_{h+1}(s,a)
    &=\phi(s,a)^\intercal\Lambda_h^{-1}
    \sum\limits_{k=1}^\tau\phi(s_h^k,a_h^k)\mathbb{P}_h V_{h+1}(s_h^k,a_h^k)
    -\mathbb{P}_h V_{h+1}(s,a)\\
    &\markref{(1)}{=}
    \phi(s,a)^\intercal\Lambda_h^{-1}
    \sum\limits_{k=1}^\tau\phi(s_h^k,a_h^k)\mathbb{P}_h V_{h+1}(s_h^k,a_h^k)
    - \phi(s,a)^\intercal\popblue{\widetilde{w}_h}\\
    &= \phi(s,a)^\intercal\Lambda_h^{-1}
    \sum\limits_{k=1}^\tau\phi(s_h^k,a_h^k)\mathbb{P}_h V_{h+1}(s_h^k,a_h^k)
    - \phi(s,a)^\intercal\popblue{\Lambda_h^{-1}\Lambda_h}\widetilde{w}_h\\
    &= \phi(s,a)^\intercal\Lambda_h^{-1}\popblue{\Big[}
    \sum\limits_{k=1}^\tau\phi(s_h^k,a_h^k)\mathbb{P}_h V_{h+1}(s_h^k,a_h^k)
    - \Lambda_h\widetilde{w}_h
    \popblue{\Big]}\\
    &\markref{(2)}{=}
    \phi(s,a)^\intercal\Lambda_h^{-1}\Big[
    \sum\limits_{k=1}^\tau\phi(s_h^k,a_h^k)\mathbb{P}_h V_{h+1}(s_h^k,a_h^k)
    \\
    & \qquad \popblue{- I \widetilde{w}_h -
    \sum\limits_{k=1}^\tau\phi(s_h^k,a_h^k)\phi(s_h^k,a_h^k)^\intercal\widetilde{w}_h}\Big]\\
    &\markref{(3)}{=}
     \phi(s,a)^\intercal\Lambda_h^{-1}\Big[
    \sum\limits_{k=1}^\tau\phi(s_h^k,a_h^k)\mathbb{P}_h V_{h+1}(s_h^k,a_h^k)  \\
    & \qquad 
    -\sum\limits_{k=1}^\tau\phi(s_h^k,a_h^k)\popblue{\mathbb{P}_h V_{h+1}(s_h^k,a_h^k)}
    - \widetilde{w}_h\Big]\\
    &= - \phi(s,a)^\intercal\Lambda_h^{-1}\widetilde{w}_h,
\end{align*}
where at (1) we have defined vector $\widetilde{w}_h\coloneqq
\int_{\S} V_{h+1}(s') d \mu_h(s')$, at (2) we have used the definition of
$\Lambda_h$, and at (3) we have recognized that
$\phi(s_h^k,a_h^k)^\intercal\widetilde{w}_h=\mathbb{P}_h V_{h+1}(s_h^k,a_h^k)$.

By taking the absolute value, we can write:
\begin{align*}
    \Big|\big(\overline{\mathbb{P}}_h-\mathbb{P}_h\big)V_{h+1}(s,a)\Big|
    &=\big| \phi(s,a)^\intercal\Lambda_h^{-1}\widetilde{w}_h\big|\\
    &\markref{(4)}{\le}\|\widetilde{w}_h\|_{\popblue{\Lambda_h^{-1}}}\|\phi(s,a)\|_{\popblue{\Lambda_h^{-1}}}\\
    &\markref{(5)}{\le} \|\widetilde{w}_h\|_{\popblue{2}}\|\phi(s,a)\|_{\Lambda_h^{-1}}\\
    &\markref{(6)}{\le} H \sqrt{d} \|\phi(s,a)\|_{\Lambda_h^{-1}},
\end{align*}
where at (4) we have applied Cauchy-Schwarz's inequality, at (5) we have
bounded the quadratic form with the 2-norm and the largest eigenvector of the
matrix, i.e.,
$\|\widetilde{w}_h\|_{\Lambda_h^{-1}}=\sqrt{\widetilde{w}_h^\intercal
\Lambda_h^{-1} \widetilde{w}_h}\le \sqrt{\sigma} \|\widetilde{w}_h\|_2$, where
$\sigma$ is the largest eigenvalue of matrix $\Lambda_h^{-1}$, and then we have
upper bounded $\sigma\le 1$, since $1$ is the smallest eigenvalue of invertible
matrix $\Lambda_h$ (see \cite{jin2020provablyefficient}); finally, at (6) we
have used the fact that $|V_{h+1}(\cdot)|\le H$, and so that
$\|\widetilde{w}_h\|_2=\|\int_{\S} V_{h+1}(s') d \mu_h(s')\|_2 \le H
\|\mu_h(\S)\|_2\le H\sqrt{d}$.

The result follows by noticing that the quantity to bound cannot be larger than $2H$.
\end{proof}

\begin{lemma}\label{lemma: bound pbar phat}
    Let $\delta\in(0,1)$. For any value
    function $V$ in the class $\V$, for any $(s,a,h)\in\SAH$, with probability at least $1-\delta/2$, it
    holds that:
    \begin{align*}
        \Big|\Big(\widehat{\mathbb{P}}_h-\overline{\mathbb{P}}_h\Big)V_{h+1}(s,a)\Big|
        \le \min\bigg\{cH\sqrt{d\log\big(1+\tau\big)+\log\frac{H}{\delta}}\|\phi(s,a)\|_{\Lambda_h^{-1}},2H\bigg\},
    \end{align*}
    for some constant $c$.
\end{lemma}
\begin{proof}
    We can write:
\begin{align*}
    \Big|\Big(\widehat{\mathbb{P}}_h  -\overline{\mathbb{P}}_h\Big)&V_{h+1}(s,a)\Big|
    =\bigg|\phi(s,a)^\intercal\Lambda_h^{-1}
    \sum\limits_{k=1}^\tau\phi(s_h^k,a_h^k)
    \Big[
    V_{h+1}(s_{h+1}^k)-\mathbb{P}_h V_{h+1}(s_h^k,a_h^k)
    \Big]\bigg|\\
    &\markref{(1)}{\le} \popblue{\Big\|}\sum\limits_{k=1}^\tau\phi(s_h^k,a_h^k)
    \Big[
    V_{h+1}(s_{h+1}^k)-\mathbb{P}_h V_{h+1}(s_h^k,a_h^k)
    \Big]
    \popblue{\Big\|}_{\popblue{\Lambda_h^{-1}}}
    \popblue{\Big\|}\phi(s,a)\popblue{\Big\|}_{\popblue{\Lambda_h^{-1}}}\\
    &\markref{(2)}{\le} 
    \popblue{\sqrt{4H^2 \Big(\frac{d}{2}\log(1+\tau)+\log\frac{2\N_\epsilon}{\delta}\Big)
    +8\tau^2\epsilon^2}}
    \Big\|\phi(s,a)\Big\|_{\Lambda_h^{-1}}\\
    &\markref{(3)}{\le} 
    \sqrt{4H^2 \Big(\frac{d}{2}\log(1+\tau)+\popblue{2d\log\Big(
    1+\frac{H\sqrt{d}}{\epsilon}    
    \Big)}+\log\frac{\popblue{1}}{\delta}\Big)
    +8\tau^2\epsilon^2}
    \Big\|\phi(s,a)\Big\|_{\Lambda_h^{-1}}\\
    &\markref{(4)}{=} 
    \sqrt{4H^2 \Big(\frac{d}{2}\log(1+\tau)+2d\log\Big(
    1+4\popblue{\tau}\Big)+\log\frac{1}{\delta}\Big)
    +8\popblue{H^2d}}
    \Big\|\phi(s,a)\Big\|_{\Lambda_h^{-1}}\\
    &\markref{(5)}{\le}
     \popblue{c}H\sqrt{d\log(1+\tau)+\log\frac{1}{\delta}}\Big\|\phi(s,a)\Big\|_{\Lambda_h^{-1}},
\end{align*}
where at (1) we have applied Cauchy-Schwarz's inequality, at (2) we have
applied Lemma \ref{lemma: lemma d4 jin}, at (3) we have upper bounded
$\N_\epsilon$ using Lemma \ref{lemma: covering number class V}, at (4),
similarly to \cite{wagenmaker2022noharder}, unlike
\cite{jin2020provablyefficient}, we see that no union bound is needed (because
there is no dependence on $\Lambda$), thus by choosing
$\epsilon=H\sqrt{d}/\tau$, we get the passage. Passage (5) follows for some
constant $c$.

The result follows by a union bound over $h\in\dsb{H}$, and by noticing that the
quantity to bound cannot be larger than $2H$.
\end{proof}

We are now ready to upper bound the Hausdorff distance using the two lemmas just
presented.
Recall that we work with unbounded rewards (parameters $\theta$), and that the
definition of inner distance $d$ is provided in Equation \eqref{eq: inner distance d}.
\begin{lemma}\label{lemma: upper bound to Hausdorff distance}
    With probability at least $1-\delta/2$, the Hausdorff distance
    between the true feasible set $\R_{p,\pie}$ and its estimate $\widehat{\R}$
    returned by Algorithm \ref{alg: known pie} can be upper bounded by:
    \begin{align*}
        \mathcal{H}(\R_{p,\pie},\widehat{\R})\le 4 J^*(u;p),
    \end{align*}
    where $u_h(s,a)\coloneqq \min\{\beta\|\phi(s,a)\|_{\Lambda_h^{-1}},H\}$ for
    all $(s,a,h)\in\SAH$, and
    $\beta\coloneqq cH\sqrt{d\log(1+\tau)+\log(H/\delta)}$ for some absolute  constant $c>0$.
\end{lemma}
\begin{proof}
    Let us begin to bound the first branch of the Hausdorff distance.
\begin{align*}
    \sup \limits_{r\in\R_{p,\pie}} & \inf\limits_{\widehat{r}\in\widehat{\R}}d(r,\widehat{r})
    =
        \sup\limits_{r\in\R_{p,\pie}}\inf\limits_{\widehat{r}\in\widehat{\R}}
        \frac{1}{M_{r,\widehat{r}}}\sup\limits_{\pi\in\Pi}\sum_{h\in\dsb{H}}
        \E_{(s,a)\sim d^{p,\pi}_h(\cdot,\cdot)}|r_h(s,a)-\widehat{r}_h(s,a)|\\
    &\markref{(1)}{=}
        \sup\limits_{r\in\R_{p,\pie}}\inf\limits_{\widehat{r}\in\widehat{\R}}
        \frac{1}{M_{r,\widehat{r}}}\sup\limits_{\pi\in\Pi}
        \sum\limits_{h\in\dsb{H}}
    \E\limits_{(s,a)\sim d_h^{p,\pi}(\cdot,\cdot)}
    \big|\popblue{Q^*_h(s,a;p,r)-\mathbb{P}_h V^*_{h+1}(s,a;p,r)}\\
    &\qquad\quad\popblue{-
    Q^*_h(s,a;\widehat{p},\widehat{r})+\widehat{\mathbb{P}}_h V^{*}_{h+1}(s,a;\widehat{p},\widehat{r})}\big|\\
    &\markref{(2)}{\le}
        \sup\limits_{r\in\R_{p,\pie}}
        \frac{1}{M_{r,\widetilde{r}}}\sup\limits_{\pi\in\Pi}
        \sum\limits_{h\in\dsb{H}}
    \E\limits_{(s,a)\sim d_h^{p,\pi}(\cdot,\cdot)}
    \big|Q^*_h(s,a;p,r)-\mathbb{P}_h V^*_{h+1}(s,a;p,r)\\
    &\qquad\quad-
    Q^*_h(s,a;\widehat{p},\popblue{\widetilde{r}})+\widehat{\mathbb{P}}_h
    V^{*}_{h+1}(s,a;\widehat{p},\popblue{\widetilde{r}})\big|,\\
    &\markref{(3)}{=}
        \sup\limits_{r\in\R_{p,\pie}}
        \frac{1}{M_{r,\widetilde{r}}}\sup\limits_{\pi\in\Pi}
        \sum\limits_{h\in\dsb{H}}
    \E\limits_{(s,a)\sim d_h^{p,\pi}(\cdot,\cdot)}
    \big|Q^*_h(s,a;p,r)-\mathbb{P}_h V^*_{h+1}(s,a;p,r)\\
    &\qquad\quad-
    Q^*_h(s,a;\popblue{p,r})+\widehat{\mathbb{P}}_h V^{*}_{h+1}(s,a;\popblue{p,r})\big|\\
    &=
        \sup\limits_{r\in\R_{p,\pie}}
        \frac{1}{M_{r,\widetilde{r}}}\sup\limits_{\pi\in\Pi}
        \sum\limits_{h\in\dsb{H}}
    \E\limits_{(s,a)\sim d_h^{p,\pi}(\cdot,\cdot)}
    \big|\big(\widehat{\mathbb{P}}_h-\mathbb{P}_h \big)V^*_{h+1}(s,a;p,r)\big|\\
    &\markref{(4)}{\le}
        \sup\limits_{r\in\R_{p,\pie}}
        \sum\limits_{h\in\dsb{H}}
    \E\limits_{(s,a)\sim d_h^{p,\pi}(\cdot,\cdot)}
    \big|\big(\widehat{\mathbb{P}}_h-\mathbb{P}_h \big)\popblue{\frac{V^*_{h+1}(s,a;p,r)}
    {\max\{1,\max_h\|\theta_h\|_2/\sqrt{d}\}}}\big|\\
    &\markref{(5)}{=}
        \sup\limits_{r\in\R_{p,\pie}}
        \sum\limits_{h\in\dsb{H}}
    \E\limits_{(s,a)\sim d_h^{p,\pi}(\cdot,\cdot)}
    \big|\big(\widehat{\mathbb{P}}_h-\mathbb{P}_h \big)V^*_{h+1}(s,a;p,\popblue{\frac{r}{K}})\big|\\
    &\markref{(6)}{\le}\popblue{\sup\limits_{V\in\V}}
    \sup\limits_{\pi\in\Pi}
    \sum\limits_{h\in\dsb{H}}
\E\limits_{(s,a)\sim d_h^{p,\pi}(\cdot,\cdot)}
\big|\big(\widehat{\mathbb{P}}_h-\mathbb{P}_h \big)\popblue{V_{h+1}(s,a)}\big|,
\end{align*}
where at (1) we have simply applied the Bellman optimality equation twice w.r.t.
the reward function, at (2) we have upper bounded the infimum over the second
set of rewards $\widehat{\R}$ with the specific choice of reward
$\widetilde{r}\in\widehat{\R}$ provided by Lemma \ref{lemma: reward choice}, at
(3) we use the property of $\widetilde{r}$ described in Lemma \ref{lemma: reward
choice}, at (4) we bring term $1/M_{r,\widetilde{r}}$ inside, and then we upper
bound it by: $1/M_{r,\widetilde{r}}\coloneqq
1/\max\{\sqrt{d},\max_h\|\theta_h\|_2,\max_h\|\widetilde{\theta}_h\|_2\}/\sqrt{d}
\le 1/\max\{1,\max_h\|\theta_h\|_2/\sqrt{d}\}$, i.e., by simply removing one of
the terms inside the maximum operator at denominator. At (5) we define
$K\coloneqq \max\{1,\max_h\|\theta_h\|_2/\sqrt{d}\}$, and, since the value
function is linear in the reward, we apply $K$ directly to the
reward. At (6) we realize that the possible optimal value functions that can be
constructed in $p$ using rewards in $\R_{p,\pie}$ normalized by $K$ are a
subset of the value functions in class $\V$, i.e., of all the possible optimal
value functions with parameters $\|w_h\|_2\le 2H\sqrt{d}$. This is not trivial
since we are working with \emph{unbounded} rewards $r$, and thus their
parameters $\{\theta_h\}_h$ can be any. The normalization by $K$ permits this in
the following manner. For any $h\in\dsb{H}$, we have
$r_h(\cdot,\cdot)/K=\dotp{\phi(\cdot,\cdot),\theta_h/K}=
\dotp{\phi(\cdot,\cdot),\theta_h/\max\{1,\max_{h'}\|\theta_{h'}\|_2/\sqrt{d}\}}$.
Therefore, if $\max_{h'}\|\theta_{h'}\|_2>\sqrt{d}$, then the normalization
makes sure that $\max_{h'}\|\theta_{h'}\|_2=\sqrt{d}$, while if
$\max_{h'}\|\theta_{h'}\|_2\le\sqrt{d}$,
then the normalization is by $1$ and it has no effect. In this way, we see that 
value functions $V^*_{h+1}(s,a;p,\popblue{\frac{r}{K}})$ can be created by a
simple $r'$ with parameters $\{\theta_h'\}_h$ with 2-norms bounded by
$\sqrt{d}$. This guarantees that, since by hypothesis of Linear MDPs
$\|\phi(\cdot,\cdot)\|_2\le1$, the value function never exceeds $H$, and that
the norm of the $Q$-function parameters $\{w_h^{\pi}\}_h$ for any policy $\pi$
can be bounded as:
$\|w_h^{\pi}\|_2\le\|\theta_h/K\|_2+\|\int_{\S} V_{h+1}^\pi(s') d \mu_h(s')\|_2\le
\sqrt{d}+H\|\mu_h(\S)\|_2\le \sqrt{d}+H\sqrt{d}\le 2H\sqrt{d}$ (similarly to
Lemma B.1 of \cite{jin2020provablyefficient}). It should be remarked that class
$\V$ is more general than the actual set of optimal value functions that can be
obtained using $r\in\R_{p,\pie}$ in $p$, since such rewards induce optimal
value functions for which the optimal action in $\S^{p,\pie}$ is always the
expert's action/s $\pie(s)$.

Notice that the same derivation can be carried out also for the
other branch of the Hausdorff distance, ending up with the same expression.
Therefore, the last line is an upper bound to the Hausdorff distance:
\begin{align*}
    \mathcal{H}_d(\R_{p,\pie},\widehat{\R})&\le \sup\limits_{V\in\V}
    \sup\limits_{\pi\in\Pi}
    \sum\limits_{h\in\dsb{H}}
\E\limits_{(s,a)\sim d_h^{p,\pi}(\cdot,\cdot)}
\big|\big(\widehat{\mathbb{P}}_h-\mathbb{P}_h \big)V_{h+1}(s,a)\big|\\
    &\markref{(7)}{=}
    \sup\limits_{\pi\in\Pi}
    \sum\limits_{h\in\dsb{H}}
\E\limits_{(s,a)\sim d_h^{p,\pi}(\cdot,\cdot)}
\popblue{\sup\limits_{V\in\V}}\big|\big(\widehat{\mathbb{P}}_h-\mathbb{P}_h \big)V_{h+1}(s,a)\big|\\
    &=
        \sup\limits_{\pi\in\Pi}
        \sum\limits_{h\in\dsb{H}}
    \E\limits_{(s,a)\sim d_h^{p,\pi}(\cdot,\cdot)}
    \sup\limits_{V\in\V}\big|\big(\widehat{\mathbb{P}}_h-\mathbb{P}_h \big)V_{h+1}(s,a)
    \popblue{\pm \overline{\mathbb{P}}_hV_{h+1}(s,a)}\big|\\
    &\markref{(8)}{\le}
    \sup\limits_{\pi\in\Pi}
    \sum\limits_{h\in\dsb{H}}
\E\limits_{(s,a)\sim d_h^{p,\pi}(\cdot,\cdot)}\sup\limits_{V\in\V}\popblue{\big|}
\big(\overline{\mathbb{P}}_h-
\mathbb{P}_h\big)V_{h+1}(s,a)\popblue{\big|}
    +\popblue{\big|}\big(\widehat{\mathbb{P}}_h-\overline{\mathbb{P}}_h\big)V_{h+1}(s,a)\popblue{\big|}\\
    &\markref{(9)}{\le}
    \sup\limits_{\pi\in\Pi}
    \sum\limits_{h\in\dsb{H}} \E\limits_{(s,a)\sim d_h^{p,\pi}(\cdot,\cdot)}
    \popblue{\min\bigg\{c_1H\sqrt{d\log(1+\tau)+\log\frac{H}{\delta}}\|\phi(s,a)\|_{\Lambda_h^{-1}},4H\bigg\}}\\
    &\le
    \popblue{4}
    \sup\limits_{\pi\in\Pi}
    \sum\limits_{h\in\dsb{H}} \E\limits_{(s,a)\sim d_h^{p,\pi}(\cdot,\cdot)}
    \min\big\{
        \underbrace{\popblue{c_2}H\sqrt{d\log(1+\tau)+\log\frac{H}{\delta}}}_{\popblue{\eqqcolon\beta}}
    \|\phi(s,a)\|_{\Lambda_h^{-1}},\popblue{H}\big\}\\
    &=
    4\sup\limits_{\pi\in\Pi}
    \sum\limits_{h\in\dsb{H}} \E\limits_{(s,a)\sim d_h^{p,\pi}(\cdot,\cdot)}
    \underbrace{\min\big\{
        \popblue{\beta}\|\phi(s,a)\|_{\Lambda_h^{-1}},H\big\}}_{\popblue{\eqqcolon u_h(s,a)}}\\
    &=
    4\sup\limits_{\pi\in\Pi}
    \sum\limits_{h\in\dsb{H}} \E\limits_{(s,a)\sim d_h^{p,\pi}(\cdot,\cdot)}
    \popblue{u_h(s,a)}\\
    &=
    4 \popblue{J^*(u;p)},
\end{align*}
where at (7) we have noticed that class $\V$ contains the cartesian product of
$H$ sets, one for each stage, and therefore the supremum can be brought inside
the summation, at (8) we have applied triangle inequality, at (9) we have
applied Lemma \ref{lemma: bound pbar p} and Lemma \ref{lemma: bound pbar phat}
and used some absolute constants $c_1,c_2>0$, and also the fact that for any numbers $x,y,w,z$,
we have $\min\{x,y\}+\min\{w,z\}\le\min\{x+w,y+z\}$.

\end{proof}

To conclude the proof of the main theorem, we simply have to observe that any
RFE algorithm provides a bound to $J^*(u';p)$ for some $u'$ similar to $u$.
Depending on the RFE algorithm instantiated as sub-routine, the sample
complexity of Algorithm \ref{alg: known pie} varies.
\upperboundlinearoldframework*
\begin{proof}
    To get the result, we instantiate Algorithm 1 of
    \cite{wagenmaker2022noharder} as RFE sub-routine. Simply, observe that
    \cite{wagenmaker2022noharder} sets $\beta'$ so that
    $\beta'\ge \widetilde{\beta}\coloneqq c'H\sqrt{d\log(1+dH\tau)+\log
    (H/\delta)}\ge\beta$. By Lemma \ref{lemma: upper bound to Hausdorff distance}, we
    know that:
    \begin{align*}
        \mathcal{H}(\R_{p,\pie},\widehat{\R})&\le
        4\sup\limits_{\pi\in\Pi}
        \sum\limits_{h\in\dsb{H}} \E\limits_{(s,a)\sim d_h^{p,\pi}(\cdot,\cdot)}
        \min\big\{
            \beta\|\phi(s,a)\|_{\Lambda_h^{-1}},H\big\}\\
        &\le \popblue{2 c_1 \beta'}\sum\limits_{h\in\dsb{H}}
        \popblue{\sup\limits_{\pi\in\Pi}}
        \E\limits_{(s,a)\sim d_h^{p,\pi}(\cdot,\cdot)}
        \popblue{\|\phi(s,a)\|_{\Lambda_h^{-1}}},
    \end{align*}
    for some absolute constant $c_1>0$. It should be remarked that the quantity in the last
    line is, modulo $c_1$, the quantity that \cite{wagenmaker2022noharder} bound
    in the proof of their Theorem 1 using their algorithm. Specifically, by
    taking:
    \begin{align*}
        \tau\le \widetilde{\mathcal{O}}\bigg(
        \frac{H^5d}{\epsilon^2}\Big(d+\log\frac{1}{\delta}\Big)+
        \frac{H^6 d^{9/2}}{\epsilon}\log^4\frac{1}{\delta}    
        \bigg),
    \end{align*}
    and a union bound over the two events that hold w.p. $1-\delta/2$, and
    re-setting $\epsilon\gets c_1\epsilon$, we get the result.    
\end{proof}

Notice that if we run Algorithm 1 of \cite{wang2020onrewardfree} for exploration
instead of Algorithm 1 of \cite{wagenmaker2022noharder}, we obtain:
\begin{thr}
If we use Algorithm 1 of \cite{wang2020onrewardfree} at Line 1 of Algorithm
\ref{alg: known pie}, then for any $\epsilon,\delta\in(0,1)$, such algorithm is
$(\epsilon,\delta)$-PAC for IRL with a number of episodes $\tau$ upper bounded
by:
\begin{align*}
    \tau\le\widetilde{\mathcal{O}}\bigg(
        \frac{H^6d^3}{\epsilon^2}\log\frac{1}{\delta}\bigg).
\end{align*}
\end{thr}
\begin{proof}
    By Lemma \ref{lemma: upper bound to Hausdorff distance}, we know that:
    \begin{align*}
        \mathcal{H}(\R_{p,\pie},\widehat{\R})&\le 4 J^*(u;p)\\
        &=
        4\sup\limits_{\pi\in\Pi}
        \sum\limits_{h\in\dsb{H}} \E\limits_{(s,a)\sim d_h^{p,\pi}(\cdot,\cdot)}
        \min\big\{
            \beta\|\phi(s,a)\|_{\Lambda_h^{-1}},H\big\},
    \end{align*}
    for $\beta\coloneqq cH\sqrt{d\log(1+\tau)+\log(H/\delta)}$.
    Now, let us define, similarly to Appendix A of \cite{wang2020onrewardfree},
    the quantities $u_h'(s,a)\coloneqq
    \min\{\beta'\|\phi(s,a)\|_{\Lambda_h^{-1}},H\}$ for all $(s,a,h)\in\SAH$,
    and $\beta'\coloneqq c'dH\sqrt{\log(dH/\delta/\epsilon)}$ for some absolute constant
    $c'>0$. In addition, set the number of exploration episodes $\tau$ to
    $\tau=c''d^3H^6\log(dH\delta^{-1}\epsilon^{-1})/\epsilon^2$, and notice
    that, for appropriate choices of $c',c''$, it holds that: $\beta'\ge
    c'dH\sqrt{\log(dH\tau/\delta)}\ge\beta\coloneqq
    cH\sqrt{d\log(1+\tau)+\log(1/\delta)}$. This entails that $u'_h(s,a)\ge
    u_h(s,a)$ at all $s,a,h$, and so:
    \begin{align*}
        \mathcal{H}(\R_{p,\pie},\widehat{\R})&\le c J^*(u';p)\\
        &= cHJ^*(u'/H;p)\\
        &\markref{(1)}{\le}
        c_1 H \sqrt{\frac{d^3 H^4\log\frac{d\tau H}{\delta}}{\tau}}\\
        &\markref{(2)}{\le} c_2\epsilon,
    \end{align*}
    where at (1) we have applied Lemma 3.2 of \cite{wang2020onrewardfree}
    (reported in Lemma \ref{lemma: bound j star} for simplicity) with some new
    constant $c_1>0$, and at (2) we
    have simply replaced $\tau$ with its value defined in Algorithm 1 of
    \cite{wang2020onrewardfree}.

    The result follows by union bound between the two events that hold w.p.
    $1-\delta/2$ to get $1-\delta$, and by noticing that $c_2$ is a constant,
    thus setting $\epsilon\gets c_2\epsilon$ provides the result.
\end{proof}

\begin{lemma}\label{lemma: reward choice}
    Let $\R_{p,\pie}$ be the feasible set of policy
    $\pie$ w.r.t. transition models $p$, and let $\widehat{\R}$ be its estimate
    constructed as in Algorithm \ref{alg: known pie} using the true
    $\pie,\S^{p,\pie}$ (or sets $\Z$) and some $\widehat{p}$. For any reward $r\in
    \R_{p,\pie}$, the reward
    $\widehat{r}$ such that, for all $(s,a,h)\in\SAH$:
    \begin{align*}
        \widehat{r}_h(s,a)=r_h(s,a)+\int\limits_{s'\in\S}p_h(s'|s,a)V_{h+1}^{*}(s';p,r)
        -\int\limits_{s'\in\S}\widehat{p}_h(s'|s,a)V_{h+1}^{*}(s';\widehat{p},\widehat{r}),
    \end{align*}
    belongs to $\widehat{\R}$.
    Moreover, observe that:
    $Q^{*}_h(s,a;p,r)=Q^{*}_h(s,a;\widehat{p},\widehat{r})$
    at all $(s,a,h)\in\SAH$.
    In addition, for any reward $\widehat{r}\in\widehat{\R}$, it is possible to
    construct a reward $r$ in analogous manner so that $r\in\R_{p,\pie}$, and
    such that $Q^{*}_h(s,a;p,r)=Q^{*}_h(s,a;\widehat{p},\widehat{r})$
    at all $(s,a,h)\in\SAH$.
\end{lemma}
\begin{proof}
    First, we consider the case when $\S$ is finite.
    By rearranging the terms in the definition of $\widehat{r}$, we see that, for
    all $(s,a,h)\in\SAH$:
    \begin{align*}
        \widehat{r}_h(s,a)+\sum\limits_{s'\in\S}\widehat{p}_h(s'|s,a)
        V_{h+1}^{*}(s';\widehat{p},\widehat{r})
        =r_h(s,a)+\sum\limits_{s'\in\S}p_h(s'|s,a)V_{h+1}^{*}(s';p,r),
    \end{align*}
    which, by the Bellman optimality equation, entails that
    $Q^{*}_h(s,a;p,r)=Q^{*}_h(s,a;\widehat{p},\widehat{r})$.

    We recall that $\widehat{\R}$ is defined as:
    \begin{align*}
        \widehat{\R}=\big\{\widehat{r}\in\mathfrak{R}\,\Big|\,\forall (s,h)\in\S^{p,\pie},\forall a\in\A:\;
\E\limits_{a'\sim\pi^E_h(\cdot|s)}Q^*_h(s,a';\widehat{p},\widehat{r})
\ge Q^*_h(s,a;\widehat{p},\widehat{r})\big\},
    \end{align*}
    while thanks to Lemma \ref{lemma: fs explicit}, the feasible set
    $\R_{p,\pie}$ can be written as:
    \begin{align*}
        \R_{p,\pie}=\big\{r\in\mathfrak{R}\,\Big|\,\forall (s,h)\in\S^{p,\pie},\forall a\in\A:\;
\E\limits_{a'\sim\pi^E_h(\cdot|s)}Q^*_h(s,a';p,{r})
\ge Q^*_h(s,a;p,{r})\big\}.
    \end{align*}
    It is clear that, if $Q^{*}_h(s,a;p,r)=Q^{*}_h(s,a;\widehat{p},\widehat{r})$
    for all $(s,a,h)\in\SAH$, then since $r\in\R_{p,\pie}$ we necessarily have
    $\widehat{r}\in\widehat{\R}$.
    
    The proof of the opposite case is completely analogous.

    In the case with infinite $\S$, notice that both the feasible set
    $\R_{p,\pie}$ in Lemma \ref{lemma: fs explicit} and the definition of
    $\widehat{\R}$ in Algorithm \ref{alg: known pie} make use of the same sets
    $\Z$. Thus, we simply make the choice of reward with same $\Z$ and proceed
    like in the finite case.
\end{proof}

\begin{lemma}[Covering Number of Class $\V$]\label{lemma: covering number class V}
    Let $\V$ be defined as in Equation \eqref{eq: definition class V}, and define
    distance dist in $\V$ as $\text{dist}(V,V')\coloneqq
    \sup_{s\in\S}|V(s)-V'(s)|$. Then, the $\epsilon$-covering number
    $|\N(\epsilon;\V,\text{dist})|$ of set $\V$ with distance dist can be bounded as:
    \begin{align*}
        \log |\N(\epsilon;\V,\text{dist})|\le d\log\Big(1+\frac{4H\sqrt{d}}{\epsilon}\Big).
    \end{align*}
\end{lemma}
\begin{proof}
    The proof follows that of Lemma D.6 of \cite{jin2020provablyefficient}, but
    is simpler because of the different form of $\V$.

    For any $V_1,V_2\in\V$ parametrized by $w_1,w_2$, we write:
    \begin{align*}
        \text{dist}(V_1,V_2)&=\sup\limits_{s\in\S}\Big|
        \max\limits_{a\in\A}\dotp{\phi(s,a),w_1}
        -\max\limits_{a\in\A}\dotp{\phi(s,a),w_2}
        \Big|\\
        &\markref{(1)}{\le}
        \popblue{\max\limits_{(s,a)\in\SA}}\Big|
        \phi(s,a)^\intercal(w_1-w_2)
        \Big|\\
        &\markref{(2)}{\le}
        \popblue{\sup\limits_{\phi:\|\phi\|_2\le1}}\Big|
        \popblue{\phi}^\intercal(w_1-w_2)
        \Big|\\
        &\markref{(3)}{=}
        \popblue{\|}w_1-w_2\popblue{\|_2},
    \end{align*}
    where at (1) we have used the common bound that the absolute difference of
    maxima is upper bounded by the maximum of the absolute difference of the two
    functions, at (2) we have used the fact that the feature map is always
    bounded by 1 in 2-norm, and at (3) we have recognized the dual norm of the
    2-norm, i.e., itself.

    If we construct an $\epsilon$-cover of $\W\coloneqq\{w\in \RR^d\,|\,\|w\|_2\le
    2H\sqrt{d}\}$ w.r.t. the 2-norm, we get a covering number bounded by
    $|\N(\epsilon;\W,\|\cdot\|_2)|\le (1+4H\sqrt{d}/\epsilon)^d$. Clearly, this
    value upper bounds the covering number of class $\V$ and the result follows.
\end{proof}

\begin{lemma}[Lemma D.4 of \cite{jin2020provablyefficient}]\label{lemma: lemma
d4 jin}
Let $\{s_k\}_{k=1}^\infty$ be a stochastic process on state space $\S$ with
corresponding filtration $\{\F_k\}_{k=0}^\infty$. Let $\{\phi_k\}_{k=0}^\infty$
be an $ \RR^d$-valued stochastic process where $\phi_k\in\F_{k-1}$, and
$\|\phi_k\|_2\le 1$. Let $\Lambda_\tau = I+\sum_{k=1}^\tau
\phi_k\phi_k^\intercal$. Then, for any $\delta>0$, with probability at least
$1-\delta$, for all $\tau\ge0$, and any $V\in\V$ so that $\sup_{s\in\S}|V(s)|\le
H$, we have:
\begin{align*}
    \bigg\|\sum\limits_{k=1}^\tau \phi_k \Big(
    V(s_k)-\E\big[V(s_k)|\F_{k-1}\big]    
    \Big)\bigg\|_{\Lambda_\tau^{-1}}
    \le 
    4H^2 \Big[\frac{d}{2}\log(1+\tau)+\log\frac{\N_\epsilon}{\delta}\Big]
    +8\tau^2\epsilon^2,
\end{align*}
where $\N_\epsilon$ is the $\epsilon$-covering number of $\V$ with respect to
the distance $\text{dist}(V,V')\coloneqq \sup_{s\in\S}|V(s)-V'(s)|$.
\end{lemma}

\begin{lemma}[Lemma 3.2 of \cite{wang2020onrewardfree}]\label{lemma: bound j star}
    With probability $1-\delta/2$, for the function $u'$ defined as
    $u_h'(s,a)\coloneqq\min\big\{\beta'\|\phi(s,a)\|_{\Lambda_h^{-1}},H
        \big\}$, with $\beta'\coloneqq c' dH\sqrt{\log(dH\delta^{-1}\epsilon^{-1})}$, we have:
    \begin{align*}
        J^*(u'/H)\le c
        \sqrt{\frac{d^3 H^4\log\frac{d\tau H}{\delta}}{\tau}},
    \end{align*}
    for some absolute constant $c>0$.
\end{lemma}

\section{Additional Insights on Compatibility}\label{section: more on compatibility}

In this appendix, we collect and describe additional insights to the notion of
\emph{rewards compatibility} introduced in Section \ref{section: rewards
compatibility}. The appendix is organized in the following manner: Appendix
\ref{section: visual explanation} provides a visual explanation to the notion of
rewards compatibility, in Appendix \ref{section: multiplicative alternative} we
analyse a multiplicative alternative to the definition of rewards compatibility,
and Appendix \ref{section: guarantees forward rl} discusses the conditions under
which a learned reward can be used for ``forward'' RL, by comparing rewards with
small (non)compatibility with rewards learned in previous works.

\subsection{A Visual Explanation for Rewards Compatibility}\label{section: visual explanation}

\begin{figure*}
    \centering
    \begin{subfigure}[b]{0.475\textwidth}
        \centering
        \begin{tikzpicture}
            \filldraw [black] (0,0) circle (2pt);
            \draw (0,0) -- (2,0);
            \draw[ultra thick] (0,0) -- (-2,0);
            \draw (0,0) -- (0,2);
            \draw (0,0) -- (0,-2);
            \draw (0,0) -- (1.41,1.41);
            \draw (0,0) -- (1.41,-1.41);
            \draw (0,0) -- (-1.41,1.41);
            \draw (0,0) -- (-1.41,-1.41);
            \draw (0,0) -- (1.85,0.76);
            \draw (0,0) -- (-1.85,0.76);
            \draw[ultra thick] (0,0) -- (1.85,-0.76);
            \draw (0,0) -- (-1.85,-0.76);
            \draw (0,0) -- (0.76,1.85);
            \draw (0,0) -- (0.76,-1.85);
            \draw (0,0) -- (-0.76,1.85);
            \draw (0,0) -- (-0.76,-1.85);
            \draw (0,2) node[anchor=south] {$d^{p,\pi_1}$};
            \draw (0.76,1.85) node[anchor=south west] {$d^{p,\pi_2}$};
            \draw (1.41,1.41) node[anchor=south west] {$\dotsc$};
            \draw (1.85,-0.76) node[anchor=west] {$\bf d^{p,\pie}$};
            \draw (-2,0) node[anchor=east] {$\bf d^{p,\overline{\pi}}$};
        \end{tikzpicture}
    \end{subfigure}
    \hfill
    \begin{subfigure}[b]{0.475\textwidth}  
        \centering 
        \begin{tikzpicture}
            \filldraw[draw=white, fill=red!20] (0,0) -- (1.41,-1.41) -- (2,0) -- cycle;
            \filldraw [black] (0,0) circle (2pt);
            \draw (0,0) -- (2,0);
            \draw (0,0) -- (-2,0);
            \draw (0,0) -- (0,2);
            \draw (0,0) -- (0,-2);
            \draw (0,0) -- (1.41,1.41);
            \draw (0,0) -- (1.41,-1.41);
            \draw (0,0) -- (-1.41,1.41);
            \draw (0,0) -- (-1.41,-1.41);
            \draw (0,0) -- (1.85,0.76);
            \draw (0,0) -- (-1.85,0.76);
            \draw[ultra thick] (0,0) -- (1.85,-0.76);
            \draw (0,0) -- (-1.85,-0.76);
            \draw (0,0) -- (0.76,1.85);
            \draw (0,0) -- (0.76,-1.85);
            \draw (0,0) -- (-0.76,1.85);
            \draw (0,0) -- (-0.76,-1.85);
            \draw (1.85,-0.76) node[anchor=west] {$\bf d^{p,\pie}$};
        \end{tikzpicture}
    \end{subfigure}
    \caption[]
    { In this figure, the point at the center represents the initial state
    $s_0= d_0$ of the environment $\M$, and each ray starting from it
    represents the occupancy measure $d^{p,\pi}$ of some policy $\pi$. The
    figure aims to provide the intuition that policies with rays close to each
    other induce similar visit distributions (e.g., both point towards the same
    direction in some grid-world), and policies with rays far away from each
    other point toward very different directions (i.e., they have different
    occupancy measures). The red area in the right denotes the set of directions
    (occupancy measures $d^{p,\pi}$ for some $\pi$) that are close in
    $\|\cdot\|_1$ norm to the direction of the expert $d^{p,\pie}$.
    } 
    \label{fig: ill-posedness}
\end{figure*}
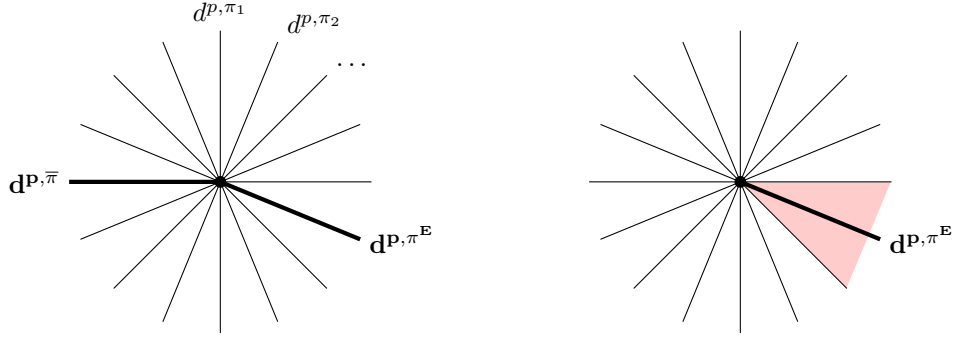

In this appendix, we aim to provide a visual intuition to the notion of
rewards compatibility. For this reason, the reader should keep in mind
Figure \ref{fig: ill-posedness}.

As explained in Section \ref{section: rewards compatibility}, even in the limit of infinite samples, i.e., even
if we know $\M=\tuple{\S,\A,H, d_0,p}\cup\{\pie\}$ exactly, and even if we
assume that the expert is exactly optimal, i.e.,
$J^*(r^E;p)-J^{\pie}(r^E;p)=0$ (where $r^E$ is the true reward optimized by the
expert),
then we still do not have idea of how other
policies perform. Expert demonstrations only provide information about the
performance of a \textit{single} policy, $\pie$, w.r.t. to the reference $J^*(r^E;p)$
under the unknown $r^E$, i.e., demonstrations say that $\pie$ in $r^E$ performs
as good as $J^*(r^E;p)$. But what about other policies? Demonstrations provide no information.

To see this, consider Figure \ref{fig: ill-posedness}, in which each line
exemplifies the visitation distribution induced by some policy $\pi\in\Pi$, and the
point in the middle represents the starting state $s_0= d_0$. Intuitively,
observing $d^{p,\pie}$ along with knowing that $J^{\pie}(r^E;p)$ is good
(i.e., because of expert demonstrations), does not tell us anything about the
distribution $d^{p,\overline{\pi}}$ induced by some other policy
$\overline{\pi}$ potentially arbitrarily different from $d^{p,\pie}$.
Indeed, it might be the case that $J^{\overline{\pi}}(r^E;p)$ is acceptable, or
that it is as good as $J^{\pie}(r^E;p)$, or that it is very bad. We cannot
know from demonstrations only.

For this reason, if we consider the set of rewards with 0-(non)compatibility,
i.e., the feasible reward set, we notice that it contains
the rewards $r$ that make $\overline{\pi}$ optimal
$J^{\overline{\pi}}(r;p)=J^*(r;p)$, but also the rewards $r'$ that make
$\overline{\pi}$ nearly optimal
$J^{\overline{\pi}}(r';p)\approx J^*(r';p)$, and also the rewards 
$r''$ that make
$\overline{\pi}$ a very bad-performing policy
$J^{\overline{\pi}}(r'';p)\ll J^*(r'';p)$.
Indeed, as long as both $r,r',r''$ make the direction pointed by $d^{p,\pie}$
in Figure \ref{fig: ill-posedness} a good direction, then they are in accordance
with the constraint imposed by the demonstrations. The additional Degrees of
Freedom (DoF) provided by policies beyond $\pie$ (e.g.,
$\overline{\pi},\dotsc$) permit the ill-posedness of IRL.

We said that expert demonstrations provide information just about the
performance of a single policy, $\pie$. However, to be precise, in the context
of IRL, this is not correct. Indeed, differently from the mere learning from
demonstrations setting, in which we just assume that $\pie$ is a very
good-performing policy, in IRL we assume that the underlying problem is an MDP,
i.e., that the expert agent is optimizing a reward function $r^E$.\footnote{When
this assumption does not hold, we incur in model misspecification
\cite{skalse2023misspecification,shah2019feasibility}.} This additional structure
(i.e., that the underlying environment is indeed an MDP), makes sure that the
performances of various directions $d^{p,\pi}$ in Figure \ref{fig:
ill-posedness} are measured through a dot product with a fixed reward function
$r$, i.e.:
\begin{align*}
    J^\pi(r;p)=\sum\limits_{h\in\dsb{H}}\dotp{d^{p,\pi}_h,r_h}.
\end{align*}
For this reason, we have the guarantee that the directions in the red area
surrounding $d^{p,\pie}$ are almost as good as $d^{p,\pie}$. Indeed, for all
policies $\pi$ such that $\sum_{h\in
\dsb{H}}\|d^{p,\pi}_h-d^{p,\pie}_h\|_1\le\epsilon$, i.e., for all policies
$\epsilon$-close to $\pie$ in 1-norm, we can write:
\begin{align*}
    |J^{\pie}(r^E;p)-J^{\pi}(r^E;p)|=
    \Big|\sum\limits_{h\in\dsb{H}}\dotp{d^{p,\pie}_h-d^{p,\pi}_h,r^E_h}\Big|
    \le
    \sum\limits_{h\in\dsb{H}}\|d^{p,\pi}_h-d^{p,\pie}_h\|_1
    \le\epsilon.
\end{align*}
In other words, policies $\pi$ and $\pie$ have similar performances.

However, it should be remarked that, since we aim to recover the rewards
explaining the expert's preferences, then we are guaranteed that policies close
in 1-norm perform similarly under any reward function (by definition of 1-norm),
and so we do not risk to incur in the error of representing $d^{p,\pie}$ and a
direction $d^{p,\pi}$ inside the red area of Figure \ref{fig: ill-posedness}
with very different performances.
 
\subsection{A Multiplicative Compatibility}\label{section: multiplicative alternative}

In Section \ref{section: rewards compatibility}, we have defined an
\emph{additive} notion of (non)compatibility, based on the difference of
performance between $\pie$ and $\pi^*$ (the optimal policy). Here, we analyze a
\emph{multiplicative} notion of (non)compatibility, based on the ratio of the
performances.\footnote{E.g., see Theorem 7.2.7 in \cite{puterman1994markov},
which is inspired by \cite{ornstein1969existence}. }

We make the following observation. Any reward $r\in\mathfrak{R}$ induces, in the
considered environment $p$, an ordering in the space of policies $\Pi$, based on
the performance $J^{\pi}(r;p)$ of each policy $\pi\in\Pi$. It is easy to notice
that for any scaling and translation parameters $\alpha \in \RR_{>
0},\beta\in \RR$, the reward constructed as $r'(\cdot,\cdot)=\alpha
r(\cdot,\cdot)+\beta$ induces the same ordering as $r$ in the space of
policies.\footnote{Indeed, simply observe that, for any $\pi\in\Pi$:
$J^{\pi}(r';p)=J^{\pi}(\alpha r+\beta;p)=\alpha J^{\pi}(r;p)+\beta$.}

For this reason, it seems desirable to use a notion of (non)compatibility such
that rewards $r$ and $r'(\cdot,\cdot)=\alpha r(\cdot,\cdot)+\beta$ for some
$\alpha,\beta$, suffer from the same (non)compatibility w.r.t. some expert
policy $\pie$. However, observe that, for the notion of compatibility
$\overline{\C}$ in Definition \ref{def: reward compatibility}, we have that, for any
$r\in\mathfrak{R}$:
\begin{align*}
    &\overline{\C}_{p,\pie}(r+\beta)=\overline{\C}_{p,\pie}(r)\qquad\forall \beta\in \RR,\\
    &\overline{\C}_{p,\pie}(\alpha r)=\alpha\overline{\C}_{p,\pie}(r)\neq\overline{\C}_{p,\pie}(r)
    \qquad\forall\alpha\in \RR_{>0}.
\end{align*}
Simply put, for the additive notion of (non)compatibility $\overline{\C}$, the
scale ($\alpha$) of a reward matters, and rescaling the reward modifies the (non)compatibility.

To solve this issue, one might introduce a \emph{multiplicative} notion of
compatibility $\F$ (defined only for non-negative rewards and setting $\F_{p,\pie}(r)=0$ when
the denominator is 0):
\begin{align*}
    \F_{p,\pie}(r)\coloneqq\frac{J^{\pie}(r;p)}{J^{*}(r;p)}.
\end{align*}
Clearly, the larger $\F_{p,\pie}(r)$, the closer is the performance of $\pie$
to the optimal performance.
Observe that, for this definition,we have:
\begin{align*}
    & \F_{p,\pie}(\alpha r)=\F_{p,\pie}(r)\qquad\forall\alpha\in \RR_{>0}\\
    & \F_{p,\pie}(r+\beta)\neq \F_{p,\pie}(r)\qquad\forall \beta\in \RR,
\end{align*}
i.e., this definition does not care about the scaling $\alpha$ of the reward,
but it is sensitive to the actual position $\beta$ of that reward.

Therefore, both $\overline{\C}$ and $\F$ suffer from some ``rescaling'' issues.
Is it possible to devise a notion of compatibility, i.e., a measure of
suboptimality, for a policy, that is independent of both the scale $\alpha$ and
position $\beta$?
Formally, we are looking for a function (notion of distance) $f: \RR\times \RR\to
 \RR_{\ge 0}$ such that, for any $J_1,J_2\in \RR$:
\begin{align}\label{eq: what i want for compatibility}
    f(\alpha J_1+\beta,\alpha J_2+\beta)=f(J_1,J_2),
\end{align}
for all $\alpha\in \RR_{> 0},\beta\in \RR$. Unfortunately, this is
not possible, since it is easy to show that all the functions $f$ of this kind
are of the following type:
\begin{align*}
    \forall J_1,J_2\in \RR\times \RR:\;f(J_1,J_2)=
    \begin{cases}
        K_+\qquad\text{if }J_1>J_2\\
        K_0\qquad\text{if }J_1=J_2\\
        K_-\qquad\text{if }J_1<J_2
    \end{cases},
\end{align*}
for some reals $K_+,K_0,K_-$. In words, any function $f$ that satisfies Equation
\eqref{eq: what i want for compatibility} is able to express just an ordering
between inputs $J_1$ and $J_2$, but not an actual measure of sub-optimality/compatibility.

We conclude by stating that we prefer to use $\overline{\C}$ instead of $\F$ for
the following reasons:
\begin{itemize}
    \item First, most RL literature prefers the additive notion of suboptimality
    towards the multiplicative one.
    \item The additive notion of suboptimality is simpler to analyze w.r.t. the
    multiplicative one.
\end{itemize}
 
\subsection{When can a learned reward be used for ``forward'' RL?}\label{section: guarantees forward rl}

In this appendix, we exploit the intuition developed in Appendix \ref{section:
visual explanation} to discuss under which conditions we can exploit demonstrations
alone to recover a single reward that \emph{can be used for ``forward'' RL},
i.e., to recover a single reward $r$ for which we have the guarantee that any
$\epsilon$-optimal policy $\pi$ to $r$ in the true environment $p$ has similar
performance in the same environment $p$ under the true reward $r^E$, that is,
policy $\pi$ is an $f(\epsilon)$-optimal policy to $r^E$ in $p$, for some
function $f$.

Applications of IRL range from Apprenticeship Learning (AL), to reward design,
to interpretability of expert's preferences. Concerning AL, it
is common to ``use'' the reward $r$ learned through IRL to optimize our learning
agent. But what properties $r$ should satisfy in order to obtain performance
guarantees on our learning agent w.r.t. the true (unknown) $r^E$? We now list and analyze
various plausible requirements.
\begin{itemize}
\item First, we might ask that, being $\pie$ optimal w.r.t. $r^E$, then
$\pie\in\argmax_\pi J^\pi(r)$, i.e., that the expert policy $\pie$ is optimal
under the learned reward $r$. However, this requirement is not satisfactory for
the following reason. Reward $r$ might induce more than one optimal policy (e.g.,
it might induce both $\overline{\pi},\pie$ as optimal), and optimal policies
other than $\pie$ (e.g., $\overline{\pi}$) are not guaranteed to perform well
under $r^E$ (actually, $\overline{\pi}$ can be any policy in $\Pi$). Clearly,
this is not satisfactory. Observe that there are rewards in the feasible set
$\R_{p,\pie}$ for which multiple policies are optimal (thus, not all the
rewards in the feasible set are satisfactory).
\item We might additionally ask that $\pie$ is the unique optimal policy of
reward $r$ (similarly to what happens in entropy-regularized MDPs
\cite{ziebart2008maximum,Fu2017LearningRR}). However, this is not
satisfactory for the following reason. In practice, it is really difficult
(almost impossible) to compute the optimal policy of a given reward. Thus,
what is usually done in RL, is to settle for an $\epsilon$-optimal policy. Since
any policy can be $\epsilon$-optimal under reward $r$, then no guarantee we can
have for such policy w.r.t. $r^E$.
\item What if we ask that $\pie$ is at least $\epsilon$-optimal under $r$
(i.e., the requirement provided by $\epsilon$-(non)compatible rewards)?
Well, this is not satisfactory because optimal policies can be any, and
because there might be other $\epsilon$-optimal policies that can perform
arbitrarily bad under $r^E$.
\end{itemize}
All the three requirements described above on $r$ do not provide guarantees that
optimizing the considered reward $r$ provides a policy with satisfactory
performance w.r.t. the true $r^E$. However, as mentioned in Section
\ref{section: rewards compatibility} and in Appendix \ref{section: visual
explanation}, expert demonstrations \emph{do not provide any information about
the performance of policies other than $\pie$ under $r^E$}. 
\begin{remark}
    If we want to be sure that an $\epsilon$-optimal policy $\pi$ for the
    learned reward $r$ in $p$ is if $f(\epsilon)$-optimal for $r^E$ in $p$ (for
    some function $f$), then, clearly, we need that \emph{all the (at least)
    $\epsilon$-optimal policies under the learned $r$ have visitation distribution
    close to that of $\pie$ in 1-norm} (see Appendix \ref{section: visual explanation}).
\end{remark}
We stress that many IRL algorithms for AL, like max-margin
\cite{abbeel2004apprenticeship}, learn a reward function just as a mere
mathematical tool to compute a policy $\pi$ which is close in 1-norm
$\|d^{\pi}-d^{\pie}\|_1$ to $\pie$.

\paragraph{A remark about works on the feasible set.}
If we look at recent works about the feasible set
\cite{metelli2023towards,lazzati2024offline,zhao2023inverse}, it might seem that
these works are able to provide guarantees between $r,r^E$ under distance $d^{all}$
(see Section 3.1 of \cite{zhao2023inverse}), defined as:
\begin{align*}
d^{all}(r,r^E)\coloneqq \sup\limits_{\pi\in\Pi}|J^\pi(r)-J^\pi(r^E)|.
\end{align*}
If $d^{all}(r,r^E)$ is small, then \emph{the performance of any
policy in $r$, not just optimal policy or $\epsilon$-optimal policy, is similar
also under $r^E$}. In other words, if we use/optimize reward $r$, then we have the
guarantee that the performance of the retrieved policy under $r^E$ is more or
less the same as its performance in $r$. Therefore, clearly, \emph{rewards $r$ with
small distance to $r^E$ w.r.t. $d^{all}$ \emph{can} be used for ``forward'' RL}.
However, we have the following result:
\begin{prop}\label{prop: no guarantees on dall}
    Let $\M=\tuple{\S,\A,H, d_0,p}$ be a known MDP without reward, and let
    $\pie$ be a known expert's policy. Let $r^E$ the true unknown reward
    optimized by the expert to construct $\pie$. Then, there does not exist a
    learning algorithm that receives in input the pair $\tuple{\M,\pie}$ and outputs a
    single reward $r$ such that $d^{all}(r,r^E)\le\epsilon$ w.p. $1-\delta$.
\end{prop}
\begin{proof}
    The proof is trivial. Indeed, since the feasible set $\R_{p,\pie}$ contains
    an infinite amount of reward functions along with $r^E$, and the learning
    algorithm cannot discriminate $r^E$ inside $\R_{p,\pie}$, then the best it
    can do is to output an arbitrary reward function $r\in\R_{p,\pie}$.
    However, since $\R_{p,\pie}$ contains, for any reward $r\in\R_{p,\pie}$,
    at least another reward $r'\in\R_{p,\pie}$ such that $d^{all}(r,r')=c$ is
    finite and equal to some positive constant $c>0$,\footnote{This is immediate
    from the considerations in Appendix \ref{section: visual explanation}.} then we can simply
    construct the problem instance with $r^E\coloneqq r'$ to make the learning
    algorithm not able to output rewards that can be used for forward learning.
\end{proof}
Nevertheless, \cite{metelli2023towards,lazzati2024offline,zhao2023inverse} seem
to provide sample efficient algorithms w.r.t. $d^{all}$.\footnote{ Actually,
\cite{metelli2023towards,lazzati2024offline} use different notions of distance,
like $d_\infty(r,r')\coloneqq\|r-r'\|_\infty$. However, we can write
$\|r-r'\|_\infty\ge\|r-r'\|_1/(SAH)$, and by dual norms we have that
$d^{all}(r,r')=\sup_{\pi\in\Pi}|\dotp{d^{p,\pi},r-r'}|\le
\sup_{\overline{d}:\|\overline{d}\|_\infty\le1}|\dotp{\overline{d},r-r'}|=\|r-r'\|_1$.
Therefore, the guarantees of \cite{metelli2023towards,lazzati2024offline} can be
converted too $d^{all}$ guarantees too.} By looking at Proposition \ref{prop: no
guarantees on dall}, we realize that this is clearly a
\emph{contradiction}. What is the right interpretation?

The trick is that the algorithms proposed in works
\cite{zhao2023inverse,metelli2023towards,lazzati2024offline} are \emph{not} able
to output a single reward $r$ which is close to $r^E$ w.r.t. $d^{all}$, but,
\emph{for any possible reward $r^E=r^E(V,A)$ parametrized\footnote{While
\cite{zhao2023inverse} makes this parametrization explicit,
\cite{metelli2023towards,lazzati2024offline} keep the parametrization implicit,
but everything is analogous.} by some value and advantage functions $V,A$, they
are able to output a reward $r$ such that $d^{all}(r,r^E(V,A))$ is small.} In
other words, it is like if these works \emph{assume to know} the $V,A$
parametrization of the true reward $r^E$. Simply put, these works are able to
output a reward $r$ that can be used for ``forward'' RL just under such
assumption. Otherwise those algorithms do not provide such guarantee.

\paragraph{Conclusions.}
To sum up, we conclude that, in general, an arbitrary reward function with small (non)compatibility can
\emph{not} be used for ``forward'' learning (see Proposition \ref{prop: no
guarantees on dall}), because we cannot know given demonstrations alone whether
the performances assigned by such reward to policies other than the expert
policy are meaningful. In addition, for the same reason, we realize that also an arbitrary reward
with zero (non)compatibility, i.e., an arbitrary reward in the feasible set, can
\emph{not} be used for ``forward'' learning.

\subsection{Comparing the (non)compatibility of various rewards}\label{apx:
compatibility comparison}

In Section \ref{section: rewards compatibility}, we said that rewards $r$ with
smaller values of $\overline{\C}_{p,\pi^E}(r)$ are more compatible with $\pi^E$
in $\M=\tuple{\S,\A,H,d_0,p}$. However, one might provide the following
``counter-example'':
\begin{example}[Question by Reviewer KyLX]\label{example: reward scale}
    Let $r^1,r^2$ be two rewards such that $r^2_h(s,a)=2 r^1_h(s,a)\ge0$ for all
    $(s,a,h)\in\SAH$. Then, clearly,
    $\overline{\C}_{p,\pi^E}(r^2)=2\overline{\C}_{p,\pi^E}(r^1)$. Therefore,
    based on Section \ref{section: rewards compatibility}, we say that
    reward $r^1$ is more compatible than $r^2$ w.r.t. $\pi^E$ in $\M$. However,
    since $r^2$ is just $r^1$ re-scaled by a constant, the two MDPs
    $\M\cup\{r^1\}$ and $\M\cup\{r^2\}$ should be ``equivalent'', thus, $r^1$
    and $r^2$ should be, intuitively, equally compatible with $\pi^E$.
\end{example}

However, Example \ref{example: reward scale} misleads the correct
\emph{interpretation} of the notion of reward function in MDPs, and in
particular about the \emph{scale} of the rewards. Let us explain better our
point.

The MDP is a model, i.e., a simplified representation of reality, which is
commonly applied to 2 different kinds of real-world scenarios: $(i)$ problems in
which the agent (learner in RL or expert in IRL) actually receives some kind of
scalar feedback from the environment, which can be modelled as a reward
function; $(ii)$ problems in which the agent does not receive a feedback from
the environment, but its objective, i.e., its structure of preferences among
state-action trajectories (which trajectories are better than others), satisfies
some axioms that permit to represent it through a scalar reward
\cite{shakerinava2022utility,bowling2023settlingrewardhypothesis} (this is
referred to as the Reward Hypothesis in literature \cite{sutton2018reinforcement}).

There is an enormous difference between scenario $(i)$ and scenario $(ii)$. In
$(i)$ the notion of $\epsilon$-optimal policy is well-defined for any fixed
$\epsilon>0$, because the reward function is given and, thus, fixed. Instead, in
$(ii)$, the notion of reward function is a mere mathematical artifact used to
represent preferences among trajectories, whose existence is guaranteed by a set
of assumptions/axioms
\cite{shakerinava2022utility,bowling2023settlingrewardhypothesis,sutton2018reinforcement}.
As Example \ref{example: reward scale} shows, positive affine transformations of
the reward do not affect the structure of preferences represented (see
\cite{shakerinava2022utility} or Section 16.2 of \cite{russel2010ai} or
\cite{Kreps1988NotesOT}). Therefore, in $(ii)$, the notion of $\epsilon$-optimal
policy is not well-defined, because rescaling a reward function $r$ to $kr$
changes the suboptimality of some policy $\pi$ from $\epsilon$ to $k\epsilon$.
In other words, for fixed $\epsilon>0$, any policy can be made
$\epsilon$-optimal by simply rescaling a reward $r$ to $kr$ for some small
enough $k>0$.

In IRL, this issue is even more influential because, although we are in setting
$(i)$, we have no idea on the scale of the true reward function. For this
reason, our solution is to attach to any reward $r$ a notion of compatibility
$\overline{\mathcal{C}}(r)$ which implicitly contains information about the
scale of the reward $r$. Compatibilities of different rewards (e.g., $r^1$ and
$r^2$ in Example \ref{example: reward scale}) cannot be compared unless the
rewards have the same scale (e.g., $r^1$ and $r^2$ have different scales, thus
their compatibilities shall not be compared).

It should be observed that in Appendix \ref{section: multiplicative alternative}
we discuss a notion of compatibility independent of the scale of the reward.
However, we show that it suffers from major drawbacks that make the notion of
compatibility introduced in the main paper (Definition \ref{def: reward
compatibility}) more suitable for the IRL problem.

In conclusion, to settle Example \ref{example: reward scale}, rewards $r^1$ and
$r^2$ should not have the same compatibility, because they have different
scales, and the notion of compatibility (i.e., suboptimality) is strictly
connected to the scale of the reward. To carry out a fair comparison of
compatibilities, one should rescale the compatibility of each reward based on
the scale of the reward.

\section{Missing Proofs and Additional Results for Section \ref{section: finally linear mdps}}\label{section: more on algorithm}

This appendix is organized as follows. First, we report the full
pseudo-code of \caty. Then, we provide the proof of Theorem \ref{thr: bounds
tabular} in Appendix \ref{appendix: proof main algorithm}.

\subsection{Algorithm}

In this section, we provide the extended version of \caty containing the
explicit conditions under which we shall instantiate one BPI/RFE algorithm
instead of another.

\begin{algorithm}[!h]
    \caption{\caty - exploration}
    \DontPrintSemicolon
    \KwData{Failure probability $\delta>0$, target accuracy $\epsilon>0$,
    expert demonstrations $\D^E$,
    set of rewards to classify $\R$, problem structure $\imath\in\{$tabular, linear
    rewards, Linear MDP$\}$
    }
    
    \uIf{$\imath \in \{ \text{tabular}, \,  \text{linear rewards}\}$}
    {\uIf{$|\R|$ is a small constant}{
        $\D\gets\{\}$\;
        \For{$r'\in \R$}{
    $\D\gets\D\cup \text{BPI\_Exploration}(\delta,\epsilon/2,r')$\Comment*[r]{Algorithm
    BPI-UCBVI \cite{menard2021fast}}\label{line:environment exploration bpi}
        }
    }
    \Else{
        $\D\gets \text{RFE\_Exploration}(\delta,\epsilon/2)$\Comment*[r]{Algorithm
    RF-Express \cite{menard2021fast}}
    }
    }
    \Else{
        $\D\gets \text{RFE\_Exploration}(\delta,\epsilon/2)$\Comment*[r]{Algorithm
    RFLin \cite{wagenmaker2022noharder}}\label{line:environment exploration rfe
    linear mdps}
    }
    Return $\D$
    \end{algorithm}

\begin{algorithm}[!h]
    \caption{\caty - classification}
    \DontPrintSemicolon
    \KwData{Failure probability $\delta>0$, target accuracy $\epsilon>0$,
    expert demonstrations $\D^E$, classification threshold $\Delta\in\RR$,
    reward to classify $r\in\R$, 
    problem structure $\imath\in\{$tabular, linear
    rewards, Linear MDP$\}$, dataset $\D$
    }
\nonl \texttt{// Estimate the expert's performance $\widehat{J}^E(r)$:}\;
\uIf{$\imath = $ tabular}{
    $\widehat{d}^E\gets$ empirical estimate of $d^{p,\pie}$ from
    $\D^E$\;
    $\widehat{J}^E(r)\gets \sum_h \dotp{\widehat{d}^E_h,r_h}$\;
}
\Else{
    $\widehat{\psi}^E\gets$ empirical estimate of $\psi^{p,\pie}$ from
    $\D^E$\label{line: estimate expert linear}\;
    $\widehat{J}^E(r)\gets \sum_h \dotp{\widehat{\psi}^E_h,r_h}$\label{line:
    estimate JE linear}\;
}    
\nonl \texttt{// Estimate the optimal performance $\widehat{J}^*(r)$:}\;
\uIf{$\imath \in \{ \text{tabular},\,\text{linear rewards}\}$}
{\uIf{$|\R|$ is a small constant}{
    $\widehat{J}^*(r)\gets \text{BPI\_Planning}(\D,r)$\Comment*[r]{Algorithm
    BPI-UCBVI \cite{menard2021fast}}\label{line: planning bpi}\;
}
\Else{
    $\widehat{J}^*(r)\gets \text{RFE\_Planning}(\D,r)$\Comment*[r]{Algorithm
    RF-Express \cite{menard2021fast}}\label{line: planning rfe tabular}\;
}
}
\Else{
    $\widehat{J}^*(r)\gets \text{RFE\_Planning}(\D,r)$\Comment*[r]{Algorithm
    RFLin \cite{wagenmaker2022noharder}}\;
}
\nonl \texttt{// Classify the reward:}\;
$\widehat{\C}(r)\gets \widehat{J}^*(r)-\widehat{J}^E(r)$\;
    $\text{class} \gets \text{True}$ \textbf{if} $\widehat{\C}(r)\le \Delta$ \textbf{else} $\text{False}$\;
    \textbf{return} $\text{class}$
    \end{algorithm}

\subsection{Proof of Theorem \ref{thr: bounds tabular}}\label{appendix: proof main algorithm}

Notice that, according to Definition \ref{def: new pac framework}, an algorithm
is $(\epsilon,\delta)$-PAC for IRL if it computes an estimate $\epsilon$-close
to the true (non)compatibility w.h.p.. Such definition does not depend on the
specific strategy adopted by the algorithm to actually classify the input reward
using the computed estimate of (non)compatibility.

Before diving into the proof of Theorem \ref{thr: bounds tabular}, we make the
following considerations.

In the common tabular MDPs setting without additional structure, we know that
the expected utility $J^{\pi}(r;p)$ of policy $\pi$ under reward $r$ in
environment with dynamics $p$ can computed as:
\begin{align*}
    J^{\pi}(r;p)=\sum\limits_{h\in\dsb{H}}\dotp{r_h,d^{p,\pi}_h},
\end{align*}
where $d^{p,\pi}_h$ is the occupancy measure of policy $\pi$ in $p$. It should
be remarked that both $r_h$ and $d^{p,\pi}_h$ have $SA$ components for all
$h\in\dsb{H}$.

In tabular MDPs with linear reward functions and in Linear MDPs, the reward
function is linear in some feature map $\phi$, i.e.:
\begin{align*}
    r_h(\cdot,\cdot)=\dotp{\phi(\cdot,\cdot),\theta_h}\qquad\forall h\in\dsb{H},
\end{align*}
where $\|\phi(s,a)\|_2\le1$ for all $(s,a)\in\SA$ and $\max_h
\|\theta_h\|_2\le\sqrt{d}$. Using this decomposition, we can rewrite the
expected utility $J^\pi(r;p)$ as:
\begin{align*}
    J^\pi(r;p)&=\sum\limits_{h\in\dsb{H}}\dotp{r_h,d^{p,\pi}_h}\\
    &=\sum\limits_{h\in\dsb{H}}\dotp{\popblue{\theta_h^\intercal \phi},d^{p,\pi}_h}\\
    &=\sum\limits_{h\in\dsb{H}}\theta_h^\intercal\popblue{\E\limits_{(s,a)\sim d^{p,\pi}_h}\phi(s,a)}\\
    &=\sum\limits_{h\in\dsb{H}}\theta_h^\intercal\popblue{\psi_h^{p,\pi}},
\end{align*}
where we have defined the feature expectations
$\{\psi_h^{p,\pi}\}_{h\in\dsb{H}}$ as $\psi_h^{p,\pi}\coloneqq
\E_{(s,a)\sim d^{p,\pi}_h}\phi(s,a)$. Observe that vector
$\psi_h^{p,\pi}$ has $d$ components instead of the $SA$ components of each
$d_h^{p,\pi}$ vector.

Since in our setting the IRL algorithm receives in input the reward function (or
its parameter $\theta\in\RR^d$), to estimate the expected utility
$J^{\pi}(r;p)$ we must estimate the visit distributions $\{d^{p,\pi}_h\}_h$ or the
feature expectations $\{\psi^{p,\pi}_h\}_h$. However, because of the different
dimensionalities of such quantities ($SA$ versus $d$), the estimates might
require different amounts of samples.
\upperboundtabular*
\begin{proof}
    To prove the theorem, we aim to find a bound to the number of samples
    $\tau^E$ such that the estimate $\widehat{J}^E(r)\approx J^{\pie}(r;p)$ is
    $\epsilon/2$-correct with probability at least $1-\delta/2$. Next,
    similarly, we aim to bound $\tau$ so that $\widehat{J}^*(r)\approx J^*(r;p)$ is
    $\epsilon/2$-correct with probability at least $1-\delta/2$. Then, the
    conclusion follows after performing a union bound and observing that, for
    any $r\in\R$:
    \begin{align*}
        \Big|\overline{\C}_{p,\pie}(r)-\widehat{\C}(r)\Big|&=
        \Big|\Big(J^*(r;p)-J^{\pie}(r;p)\Big)-\Big(\widehat{J}^*(r)-\widehat{J}^E(r)\Big)\Big|\\
        &\le \Big|J^*(r;p)-\popblue{\widehat{J}^*(r)}\Big|
        +\Big|\popblue{J^{\pie}(r;p)}-\widehat{J}^E(r)\Big|\\
        &\le \frac{\epsilon}{2}+\frac{\epsilon}{2}=\epsilon.
    \end{align*}    

\textbf{Estimating $\widehat{J}^E(r)\approx J^{\pie}(r;p)$}

To estimate $J^{\pie}(r;p)$, \caty simply
computes the empirical estimate of $\{d^{p,\pie}_h\}$ in case of tabular MDPs, and
the empirical estimate of $\{\psi^{p,\pie}_h\}$ in case of tabular MDPs with
linear rewards and Linear MDPs.
Notice that by empirical estimates we mean:
\begin{align*}
    \widehat{d}^E_h(s,a)\coloneqq \frac{\sum\limits_{i\in\dsb{\tau^E}}\indic{s_h^i=s\wedge a_h^i=a}}
    {\sum\limits_{i\in\dsb{\tau^E}}\indic{s_h^i=s}}\qquad \forall (s,a,h)\in\SAH,
\end{align*}
and:
\begin{align*}
    \widehat{\psi}^E_h\coloneqq \frac{\sum\limits_{i\in\dsb{\tau^E}}\phi(s_h^i,a_h^i)}{\tau^E}
    \qquad\forall h\in\dsb{H}.
\end{align*}
Concerning the estimate of the visit distribution $\widehat{d}^E$, we can
use the result of Lemma 6 in \cite{Shani2021OnlineAL} (we are working with
bounded rewards), to obtain that:
\begin{align*}
    \sum\limits_{h\in\dsb{H}}\|d^{p,\pie}_h- \widehat{d}^E_h\|_1\le
    \sqrt{\frac{SAH^3\log\frac{8SAH}{\delta}}{2\tau^E}}\le\frac{\epsilon}{2}.
\end{align*}
Solving w.r.t. $\tau^E$ we get the bound on $\tau^E$.

In a completely analogous manner, we can bound the feature expectations as:
\begin{align*}
    \sum\limits_{h\in\dsb{H}}\|\psi^{p,\pie}_h- \widehat{\psi}^E_h\|_1\le
    \sqrt{\frac{dH^3\log\frac{8dH}{\delta}}{2\tau^E}}\le\frac{\epsilon}{2}.
\end{align*}
Again, solving w.r.t. $\tau^E$ we get the bound on $\tau^E$.

\textbf{Estimating $\widehat{J}^*(r)\approx J^*(r;p)$}

Let us begin with the case in which $\R$ is large.
As explained for instance in Definition 4 of \cite{Xu2023ProvablyEA}, both
algorithms RF-Express \cite{menard2021fast} and RFLin
\cite{wagenmaker2022noharder} satisfy the \emph{uniform policy evaluation property},
i.e., they guarantee that, for any $\epsilon,\delta\in(0,1)$, after having
explored for
$\tau\le\widetilde{\mathcal{O}}\Big(\frac{H^3SA}{\epsilon^2}\big(S+\log\frac{1}{\delta}\big)\Big)$
in case of RF-Express \cite{menard2021fast}, and
$\tau\le\widetilde{\mathcal{O}}\Big(\frac{H^5d}{\epsilon^2}\big(d+\log\frac{1}{\delta}\big)\Big)$
for the algorithm in \cite{wagenmaker2022noharder} (we omit linear terms in
$1/\epsilon$), they compute an estimate $\widehat{p}\approx p$ of the true
transition model such that:
\begin{align*}
    \mathbb{P}\Big(
        \sup\limits_{r\in\mathfrak{R},\pi\in\Pi}
        \big|J^\pi(r;p)-J^\pi(r;\widehat{p})\big|\le\epsilon
        \Big)\ge 1-\delta.
\end{align*}
Clearly, if such property holds, then by computing the performance of the
policy $\widehat{\pi}$ outputted by the RFE algorithm we are able to obtain an
$\epsilon/2$-correct estimate of $J^*(r;p)$.\footnote{Actually, for Linear MDPs,
instead of evaluating the policy returned by Algorithm 2 of
\cite{wagenmaker2022noharder}, we can simply consider the optimistic estimate of
the $V$-function computed by such algorithm, which has the property of being
$\epsilon$-close to the true optimal $V$-function.}

Concerning the case in which $|\R|$ is a finite small constant, for tabular and
tabular with linear rewards MDPs, we can simply use algorithm BPI-UCBVI of
\cite{menard2021fast} as sub-routine, and run it as many times as there are
rewards in $\R$. When $|\R|$ is a small constant, we can proceed with a union
bound over $\R$:
\[
 \mathbb{P}\Big(\sup\limits_{r\in\mathfrak{R},\pi\in\Pi}
        \big|J^\pi(r;p)-J^\pi(r;\widehat{p})\big|\le\epsilon
        \Big) \ge 1 - \sum_{r \in \mathcal{R}} \mathbb{P}\Big(\sup\limits_{\pi\in\Pi}
        \big|J^\pi(r;p)-J^\pi(r;\widehat{p})\big|>\epsilon
        \Big) \ge 1 - |\mathcal{R}| \delta.
\]
This allows us to formally distinguish between small and large $|\R|$ based on
the following inequality:
\[
	S+ \log \frac{1}{\delta} < \log \frac{|\R|}{\delta} \implies S <\log |\R|.
\]
\end{proof}

\section{Missing Proofs and Additional Results for Section \ref{section: insights rfe and irl}}\label{section: more lower bound}

This appendix is organized as follows. First, in Appendix \ref{appendix: new
problems}, we introduce two problems that share similarities with RFE and IRL,
and we characterize the main differences among them. In addition, we enunciate a
lower bound to the sample complexity that is common to some of these 4 problems.
Next, in Appendix \ref{appendix: lower bound single reward}, we provide the
missing proofs.

\subsection{Four Problems}\label{appendix: new problems}

The 4 problems that we consider here are Reward-Free Exploration (RFE), Inverse
Reinforcement Learning (IRL), Matching Performance (MP), and Imitation Learning from
Demonstrations alone (ILfO). MP represents a novel generalization of RFE, while
ILfO, introduced in \cite{liu2018imitation}, represents an exemplification of
MP.
Before enunciating the minimax lower bound, it is
important to formally define each of these problems, as well as what we mean by
learning in each problem.

\subsubsection{Definition of the Problems}

In all the 4 problems, the learner is placed into an \emph{unknown} MDP without
reward $\M=\tuple{\S,\A,H, d_0,p}$, i.e., an environment whose dynamics $\tuple{
d_0,p}$ is unknown to the learner. For simplicity, w.l.o.g., we assume that
there is a single initial state $s_0\coloneqq  d_0$. In each problem, the
learner can explore the environment at will to collect samples about the
dynamics $p$, whose knowledge improves the performance of the agent at solving
the task. However, at exploration phase, the learner does not know which is the
specific task it has to solve. It just knows that the specific task belongs to a
given set of tasks $\mathfrak{T}$ (e.g., set of reward functions).  
The agent can use the knowledge of $\mathfrak{T}$ to engage in a more efficient
task-driven exploration. For any $\epsilon,\delta\in(0,1)$, the goal of the
agent is to being able to ouputting, for any task in $\mathfrak{T}$ a quantity
$\mathfrak{o}$ (e.g., a policy) that solves that specific task in an
$\epsilon$-correct manner with probability at least $1-\delta$. The ultimate
goal of exploration is to collect the least number of samples that permits
$(\epsilon,\delta)$-correctness for all the tasks in $\mathfrak{T}$.

Now, let us see what the quantities $\mathfrak{T}$ and $\mathfrak{o}$ represent in
each of the 4 problems. In Table \ref{tab: single reward}, we provide a sum up
of the various definitions.

\paragraph{Reward-Free Exploration (RFE).} In RFE, the learner receives a set of
reward functions $\mathfrak{T}=\R\subseteq\mathfrak{R}$ in input, and the goal
is to exploit the information about $p$ collected at exploration phase to
output, for any reward $r\in\R$, an $\epsilon$-optimal policy
$\mathfrak{o}=\widehat{\pi}_r$ w.p. $1-\delta$. When $\mathfrak{T}=\{r\}$ is a
singleton, the RFE problem is commonly termed the BPI problem. In symbols, any
RFE algorithm must guarantee that:
\begin{align*}
    \mathbb{P}\Big(
    \sup\limits_{r\in\R}J^*(r;p)-J^{\widehat{\pi}_r}(r;p)\le\epsilon
    \Big)\ge 1-\delta,   
\end{align*}
where $\widehat{\pi}_r$ is the estimate of the algorithm for reward $r$.

\paragraph{Inverse Reinforcement Learning (IRL).} In IRL, the learner receives
in input an occupancy measure\footnote{Actually, as explained in Section
\ref{section: insights rfe and irl}, the knowledge of $d^{p,\pie}$ at
exploration phase is useless. The visit measure might be provided after the
exploration along with the true reward to classify.}
$\{d^{p,\pie}_h\}_{h\in\dsb{H}}$ and a set of reward functions
$\R\subseteq\mathfrak{R}$: $\mathfrak{T}=\tuple{d^{p,\pie},\R}$, but it does
not know which specific reward it will have to classify. Under the assumption
that the occupancy measure $d^{p,\pie}$ is known,\footnote{The assumption that
$d^{p,\pie}$ is known is useful to reduce the estimation problem of the
(non)compatibility of a reward $\overline{\C}_{p,\pie}(r)\coloneqq
J^*(r;p)-J^{\pie}(r;p)$ to the problem of estimating the optimal utility
$J^*(r;p)$ only. Indeed, if $d^{p,\pie}$ is known, then, for any reward $r$,
the utility $J^{\pie}(r;p)$ is known.} the problem reduces to exploiting the
information about $p$ collected at exploration phase to output, for any reward
$r\in\R$, an $\epsilon$-correct estimate $\mathfrak{o}=\widehat{J}(r)$ of the
optimal utility $J^*(r)$ w.p. $1-\delta$. In symbols, under these conditions,
any IRL algorithm must guarantee that:
\begin{align*}
    \mathbb{P}\Big(
    \sup\limits_{r\in\R}\big|J^*(r;p)-\widehat{J}(r)\big|\le\epsilon
    \Big)\ge 1-\delta,   
\end{align*}
where $\widehat{J}(r)$ is the estimate of the algorithm for reward $r$.

\paragraph{Matching Performance (MP).} In MP, the learner receives in input a set
of reward functions $\R\subseteq\mathfrak{R}$ and a measure of performance for
each of them $\overline{J}:\R\to \RR$:
$\mathfrak{T}=\tuple{\overline{J},\R}$. For any $r\in\R$, the utility
$\overline{J}(r)$ represents a performance measure for which we aim to find the
policy that achieves closest performance. Thus, in MP, the goal is to exploit
the information about $p$ collected at exploration phase to output, for any
reward $r\in\R$, a policy $\mathfrak{o}=\widehat{\pi}_r$ such that, if we denote
the policy with performance closest to $\overline{J}(r)$ by
$\overline{\pi}_r\in\argmin_\pi |J^\pi(r)-\overline{J}(r)|$, then the utility of
policy $\widehat{\pi}_r$ is $\epsilon$-close to the utility of policy $\overline{\pi}_r$
w.p. $1-\delta$. In symbols, any MP algorithm must guarantee that:
\begin{align*}
    \mathbb{P}\Big(
    \sup\limits_{r\in\R}\big|J^{\overline{\pi}_r}(r;p)-J^{\widehat{\pi}_r}(r;p)\big|\le\epsilon
    \Big)\ge 1-\delta,   
\end{align*}
where $\overline{\pi}_r\in\argmin_\pi |J^\pi(r)-\overline{J}(r)|$, and
$\widehat{\pi}_r$ is the estimate of the algorithm for reward $r$.

\paragraph{Imitation Learning from Demonstrations alone (ILfO).} In ILfO, the
learner receives in input a set of \emph{state-only} reward functions
$\R\subset\mathfrak{R}$ and a \emph{state-only} occupancy measure
$\{\overline{d}_h\}_{h\in\dsb{H}}$: $\mathfrak{T}=\tuple{\overline{d},\R}$.
Under the assumption that $\overline{d}$ does not leak any information about the
true transition model $p$, the goal is to exploit the information about $p$
collected at exploration phase to output, for any reward $r\in\R$, a policy
$\mathfrak{o}=\widehat{\pi}_r$ such that, if we denote the policy with
performance closest to
$\overline{J}(r)\coloneqq\sum_{h\in\dsb{H}}\dotp{r_h,\overline{d}_h}$ by
$\overline{\pi}_r\in\argmin_\pi |J^\pi(r)-\overline{J}(r)|$, then the utility of
policy $\widehat{\pi}_r$ is $\epsilon$-close to the utility of policy
$\overline{\pi}_r$ w.p. $1-\delta$. Simply put, ILfO, as defined in this manner,
exemplifies the MP setting by providing a functional form to
$\overline{J}:\R\to \RR$ as an inner product between a certain state-only
occupancy measure and the input reward. It should be remarked that the
assumption made for ILfO is mild, because it is satisfied by the setting in
which the expert and the learner have the same state space but different action
spaces (or different dynamics). Indeed, in such case, the visit distribution $\overline{d}$ of the
expert would not leak any information about $p$. In symbols, any ILfO algorithm
must guarantee that:
\begin{align*}
    \mathbb{P}\Big(
    \sup\limits_{r\in\R}\big|J^{\overline{\pi}_r}(r;p)-J^{\widehat{\pi}_r}(r;p)\big|\le\epsilon
    \Big)\ge 1-\delta,   
\end{align*}
where $\overline{\pi}_r\in\argmin_\pi |J^\pi(r)-\overline{J}(r)|$ and
$\overline{J}(r)\coloneqq\sum_{h\in\dsb{H}}\dotp{r_h,\overline{d}_h}$, and
$\widehat{\pi}_r$ is the estimate of the algorithm for reward $r$.

\begin{table}[t]
    \centering
    \scalebox{0.8}{
    \begin{tabular}{c c c c c}
        & BPI & IRL & MP & ILfO\\
        \hline
        Set of Tasks $\mathfrak{T}$ & $\R$ &
        $\tuple{d^{p,\pie},\R}$ &
        $\tuple{\overline{J},\R}$ &
        $\tuple{\overline{d},\R}$\\
        Assumptions & / &
        $d^{p,\pie}$ known &
        $\overline{J}$ can be non-realisable &
        $r$ state-only, $\overline{d}$ no info\\
        Output $\mathfrak{o}$ & $\widehat{\pi}$ &
        $\widehat{J}$ &
        $\widehat{\pi}$ &
        $\widehat{\pi}$\\
        Goal & $J^{\widehat{\pi}}(r;p)\approx J^*(r;p)$ &
        $\widehat{J}\approx J^*(r;p)$ &
        $J^{\widehat{\pi}}(r;p)\approx \overline{J}$ &
        $J^{\widehat{\pi}}(r;p)\approx \sum_h\dotp{\overline{d}_h,r_h}$\\
    \end{tabular}
    }
    \caption{\small Summary of the problems.}
    \label{tab: single reward}
\end{table} 


\subsubsection{Lower Bound}

We now present a minimax lower bound rate that is common to RFE, IRL, and MP. We
report here the lower bounds presented in Section \ref{section: insights rfe and
irl}.
\instancedependentlowerbound*
\begin{proof}
    The proof is similar to that of \cite{metelli2023towards}. We split the
    proof in two parts, by considering two classes of difficult problem
    instances in Lemma \ref{lemma: lower bound single reward} and Lemma
    \ref{lemma: lower bound packing irl}. Next, we combine the two bounds through
    $\max\{a,b\}\ge (a+b)/2$ for all $a,b\ge0$. For the proof, we will assume
    that the expert visit distribution is known. The obtained bound represents a
    lower bound to the more general setting in which it is unknown.
\end{proof}

\lowerboundrfe
\begin{proof}
    The proof of this result is analogous to that of Theorem \ref{thr: instance
    dependent lower bound}, and it employs Lemma \ref{lemma: lower bound single reward} and Lemma
    \ref{lemma: lower bound packing irl}.
\end{proof}

Some observations are in order.
First, since MP is a more general setting than RFE, then this lower bound is a
lower bound for MP too. However, this is not guaranteed for ILfO.
We observe that, while for RFE and IRL the bound is tight, for MP we cannot say
so because we do not have the upper bound.
Notice that, in case the expert state-only distribution $\overline{d}$ was
unknown at exploration phase, and revealed afterwards, then the lower bound of
Theorem \ref{thr: instance dependent lower bound} holds for ILfO too, because we
might a posteriori reveal the state-only distribution $\overline{d}$ of the
optimal policy, and thus, in such manner, ILfO would be reduced to RFE.

\subsection{Missing proofs}\label{appendix: lower bound single reward}

\begin{lemma}\label{lemma: lower bound single reward}
    Let IRL and RFE be the learning problems defined as in Appendix \ref{appendix: new
    problems}. Then, for each problem, any $(\epsilon,\delta)$-PAC algorithm must collect at least
    the following number of exploration episodes:
    \begin{align*}
        \tau\ge \Omega\bigg(\frac{H^3SA}{\epsilon^2}\log\frac{1}{\delta}\bigg).
    \end{align*}
\end{lemma}
\begin{proof}
    Observe that the proof for RFE is present in \cite{domingues2021episodic}.
    Thus, we have to prove just the result for IRL. For doing so, we will use
    both the results of \cite{domingues2021episodic} and
    \cite{metelli2023towards}. Notice that for the sake of this proof we
    consider $\R = \{r\}$, that will reduce our problem to simple RL as, in
    order to compute the function $\overline{\mathcal{C}}_{p,\pi^E}(r)$, we just
    need to compute $J^*(r;p)$, being $J^{\pi^E}(r;p)$ known from the
    availability of $d^{p,\pi^E}$ and $r$.
    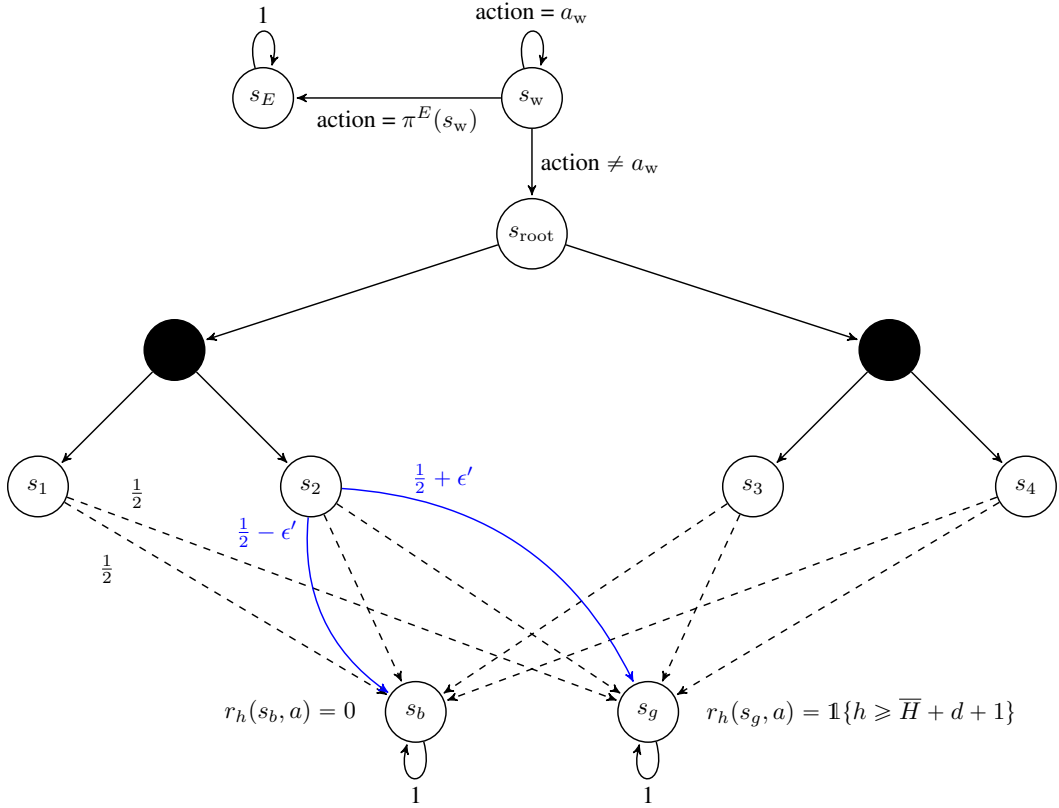
\begin{figure}[ht]
        \centering
        \resizebox{\textwidth}{!}{

\begin{tikzpicture}[->,>=stealth',shorten >=1pt,auto,node distance=2.8cm,
semithick]
\tikzstyle{every state}=[fill=white,text=black]

\node[state]                (W)                      {$s_{\mathrm{w}}$};
\node[state]                (A) [below=1cm of W]     {$s_\mathrm{root}$};

\node[state]                (E) [left=3cm of W]     {$s_E$};

\path
	(W) edge node{action = $\pi^E(s_{\mathrm{w}})$} (E);

\path
	(E) edge [loop above] node{1} (E);


\node[state,fill=black]                (N1)[below left=1cm and 4.5cm of A]    {}; 
\node[state,fill=black]                (N2)[below right=1cm  and 4.5cm of A]    {}; 

\path 
	(A) edge (N1)
	(A) edge (N2);

\node[state]     (N3)[below  left of=N1]    {$s_1$}; 
\node[state]     (N4)[below  right of=N1]    {$s_2$}; 

\path 
	(N1) edge (N3)
	(N1) edge (N4);

\node[state]     (N5)[below  left  of=N2]    {$s_3$}; 
\node[state]     (N6)[below  right of=N2]    {$s_4$}; 

\path 
	(N2) edge (N5)
	(N2) edge (N6);

\path  
	(W)  edge node{action  $\neq a_{\mathrm{w}}$} (A)
	(W)  edge [loop above] node{action = $a_{\mathrm{w}}$} (W)  
	;
	
\node[state, fill=white]    (GOOD) [below right = 6.25cm and 1cm of A] {$s_g$};
\node[state, fill=white]    (BAD) [below left  = 6.25cm and 1cm of A] {$s_b$};

\node[state, fill=white,draw=none]    (RG) [right=0.25cm of GOOD] {$r_h(s_g, a) =\indic{h \geq \overline{H}+d+1}$};
\node[state, fill=white,draw=none]    (RB) [left=0.25cm of BAD] {$r_h(s_b, a)=0$};

\path
	(N3)  edge[dashed]   node[above left=1.2cm and 2.75cm]{$\frac{1}{2}$} (GOOD)
	(N3)  edge[dashed]   node[above left=0.1cm and 1.5cm]{$\frac{1}{2}$} (BAD)
	
	(N4)  edge[dashed]   node[below]{} (GOOD)
	(N4)  edge[dashed]   node[below]{} (BAD)
	
	(N5)  edge[dashed]   node[below]{} (GOOD)
	(N5)  edge[dashed]   node[below]{} (BAD)
	
	(N6)  edge[dashed]   node[below]{} (GOOD)
	(N6)  edge[dashed]   node[below]{} (BAD)
	;

\path 
	(N4)  edge[blue, bend left]   node[above left=0.6cm and 0.5cm]{$\textcolor{blue}{\frac{1}{2}+\epsilon'} $} (GOOD)
	(N4)  edge[blue, bend right]   node[above left=0.9cm and 0.25cm]{$\textcolor{blue}{\frac{1}{2}-\epsilon'} $} (BAD)
	(GOOD)  edge [loop below] node{1} (GOOD)  
	(BAD)  edge [loop below] node{1}  (BAD)    
	;
\end{tikzpicture}}
        \caption{Hard instances.}
        \label{fig: hard instances domingues}
    \end{figure}

    \textbf{Instances Description}
    The hard instances considered are exactly the same as
    \cite{domingues2021episodic}, and are reported in Figure
    \ref{fig: hard instances domingues} for simplicity. The only difference is
    the presence of state $s_E$, to which the expert's policy $\pie$ brings,
    which is absorbing. Such state is needed to make the knowledge of the
    expert's visit distribution $d^{p,\pie}$ useless at inferring information
    about the transition model in other parts of the state-action space. Based on
    \cite{domingues2021episodic}, we describe such hard instances.
    Similarly to \cite{domingues2021episodic}, we assume that $S\ge 7, A\ge 2$,
    and there exists an integer $d$ such that $S=4+(A^d-1)/(A-1)$, and we assume
    that $H\ge 3d$. Note that \cite{domingues2021episodic} show how to relax the
    assumption on the existence of $d$.

    There are the initial state $s_{\mathrm{w}}$, from which the agent starts, and states
    $s_g,s_b$, respectively, the ``good'' and ``bad'' states which are
    absorbing. Moreover, there is state $s_E$, which is reached by the expert,
    and is absorbing. The remaining $S-4$ states are arranged in a full $A$-ary tree of
    depth $d-1$ with root $s_{\text{root}}$. We denote by $\overline{H}\le H-d$
    a certain integer parameter, and by $\mathcal{L}\coloneqq \{s_1,s_2,\dotsc,s_L\}$ the
    set of leaves of the tree.
    We define
    $\I\coloneqq\{1+d,\dotsc,\overline{H}+d\}\times\mathcal{L}\times\A$.
    For any $\imath\in\I$, we define and MDP $\M_\imath$ as follows.
    In any state of the tree, i.e., in states
    $\S\setminus\{s_{\mathrm{w}},s_g,s_b,s_E\}$, the transitions are deterministic,
    and the $a$-th action of a state brings to the $a$-th child of that node.
    
    The transitions from $s_{\mathrm{w}}$ are given by
\begin{align*}
& p_h(s_{\mathrm{w}}|s_{\mathrm{w}}, a) \coloneqq \indic{a =  a_{\mathrm{w}}, h \le \overline{H}}
\quad\text{and}\quad
p_h(s_{\text{root}}|s_{\mathrm{w}}, a) \coloneqq 1 -  p_h(s_{\mathrm{w}}|s_{\mathrm{w}}, a).
\end{align*}
In other words, action $a_{\mathrm{w}}$ allows the agent to remain in the
initial state $s_{\mathrm{w}}$ up to stage $\overline{H}$. After stage
$\overline{H}$, the agent is forced to leave $s_{\mathrm{w}}$ and to traverse
the tree down to the leaves. Action $a_E=\pi^E_1(s_{\mathrm{w}})$ is the only
action that brings to state $s_E$, which is absorbing. The transitions from any
leaf $s_i \in \mathcal{L}$ are given, as in \cite{domingues2021episodic}, by:
\begin{align}
\label{eq:transition-probs}
& p_h(s_g|s_i, a) \coloneqq \frac{1}{2} +  \Delta_{(h^*, \ell^*, a^*)}(h, s_i, a)
\quad\text{and}\quad
p_h(s_b|s_i, a) \coloneqq \frac{1}{2} -  \Delta_{(h^*, \ell^*, a^*)}(h, s_i, a),
\end{align}
where $\Delta_{(h^*, \ell^*, a^*)}(h, s_i, a) \coloneqq \indic{(h, s_i, a)=(h^*,
s_{\ell^*}, a^*)}\cdot \epsilon'$, for some $\epsilon' \in [0,1/2]$. For this
reason, there exists a (single) leaf $\ell^*$ where the agent can choose an
action $a^*$ at stage $h^*$ to increase its probability of arriving to the good
state $s_g$, which provides higher reward. We define states $s_g$ and $s_b$ to
be absorbing, i.e., they satisfy $p_h(s_b|s_b, a) \coloneqq p_h(s_g|s_g, a)
\coloneqq 1$ for any action $a$. The
reward function is state-only and is defined as
\begin{align*}
 \forall a\in\mathcal{A}, \quad r_h(s,a) \coloneqq \indic{s = s_g, h \geq \overline{H}+d+1},
\end{align*}
so that even though the agent decides to stay at $s_{\mathrm{w}}$ until stage
$\overline{H}$, it does not lose any reward. Observe that state $s_E$ does not
provide any reward, so that to estimate the (non)compatibility, any algorithm
must provide a good estimate of the optimal performance.

Finally, we define a reference MDP $\M_0$ which is an MDP of the above type but
for which $\Delta_{0}(h, s_i, a)\coloneqq0$ for all $(h,s_i,a)$. For certain
$\epsilon'$ and $\overline{H}$ to choose, we define the class $\mathbb{M}$ to be
the set $\mathbb{M}\coloneqq\{\M_0\}\cup\{\M_{\iota}\}_{\iota\in\I}$.

\textbf{Distance between problems}
We will prove the lower bound for instance $\M_0$. Observe that, in $\M_0$, the
optimal utility is:
\begin{align*}
    J^{*}_0=\frac{1}{2}(H-\overline{H}-d),
\end{align*}
because there is no triple with additional bias towards $s_g$. Instead, for any
other $\M_\imath\in\mathbb{M}$, the optimal utility is:
\begin{align*}
    J^{*}_\imath=(H-\overline{H}-d)\Big(\frac{1}{2}+\epsilon'\Big).
\end{align*}
Therefore, if we choose $\epsilon'\coloneqq 2\epsilon/(H-\overline{H}-d)$, we
have that, for any $\imath\in\I$:
\begin{align*}
    \big|J^{*}_0-J^{*}_\imath\big|=2\epsilon.
\end{align*}
Thus, in particular, for any estimate $\widehat{J}\in \RR$ we necessarily
have $|J^*_0-\widehat{J}|\le \epsilon \implies
|J^*_\imath-\widehat{J}|>\epsilon$, and vice versa, i.e., we cannot provide an
estimate $\widehat{J}$ that is $\epsilon$-close to both $J^*_0$ and $J^*_\imath$.

\textbf{Identifying the underlying problem}
Following \cite{metelli2023towards}, let us consider a generic
$(\epsilon,\delta)$-correct algorithm $\mathfrak{A}$ that outputs the estimated
optimal utility $\widehat{J}$. Then, for all $\imath\in\I$, we have:
    \begin{align*}
        \delta&\ge \sup\limits_{\text{all problem instances }\M}
        \mathbb{P}_{\M,\mathfrak{A}}\bigg(\Big|J^*_\M-\widehat{J}\Big|\ge\epsilon\bigg)\\
        &\ge \sup\limits_{\popblue{\M\in\mathbb{M}}}
        \mathbb{P}_{\M,\mathfrak{A}}\bigg(\Big|J^*_\M-\widehat{J}\Big|\ge\epsilon\bigg)\\
        &\ge \popblue{\max\limits_{\ell\in\{0,\imath\}}}
        \mathbb{P}_{\popblue{\M_{\ell}},\mathfrak{A}}
        \bigg(\Big|J^*_{\ell}-\widehat{J}\Big|\ge\epsilon\bigg).    
\end{align*}
For every $\imath\in\I$, we define the \emph{identification function}
$\Psi_\imath$ as the index of the problem ``recognized'' by algorithm
$\mathfrak{A}$. In symbols:
\begin{align*}
    \Psi_\imath\coloneqq \argmin\limits_{\ell\in\{0,\imath\}}\Big|J^*_{\ell}-\widehat{J}\Big|.
\end{align*}
In words, given estimate $\widehat{J}$ returned by algorithm $\mathfrak{A}$, the
identification function $\Psi_\imath$ returns the problem between $\M_0$ and
$\M_\imath$ whose optimal utility is closest to the estimate $\widehat{J}$.
For what we have seen in the previous paragraph, problems $\M_0$ and $\M_\imath$
lie at a distance of at least $2\epsilon$ for all $\imath\in\I$. Therefore, for
$\jmath\in\{0,\imath\}$, we have the following inclusion of events:
\begin{align*}
    \{\Psi_\imath\neq\jmath\}\subseteq \{|J^*_\jmath-\widehat{J}|>\epsilon\}.
\end{align*}
Thanks to this fact, we can continue lower bounding the probability as:
\begin{align*}
    \max\limits_{\ell\in\{0,\imath\}}
    \mathbb{P}_{\M_{\ell},\mathfrak{A}}
    \bigg(\Big|J^*_{\ell}-\widehat{J}\Big|\ge\epsilon\bigg)
    &\ge \max\limits_{\ell\in\{0,\imath\}}\mathbb{P}_{\M_{\ell},\mathfrak{A}}
    \popblue{\{\Psi_\imath\neq\ell\}}\\    
    &\markref{(1)}{\ge} \frac{1}{2}\bigg[\mathbb{P}_{\M_{0},
        \mathfrak{A}}\big(\Psi_\imath\neq 0\big)+
        \mathbb{P}_{\M_{\imath},\mathfrak{A}}\big(\Psi_\imath\neq\imath\big)\bigg]\\
    &= \frac{1}{2}\bigg[\mathbb{P}_{\M_{0},
    \mathfrak{A}}\big(\Psi_\imath\neq 0\big)+
    \mathbb{P}_{\M_{\imath},\mathfrak{A}}\big(\popblue{\Psi_\imath=0}\big)\bigg]\\
    &\markref{(2)}{\ge} \frac{1}{4}\exp^{-\text{KL}(\mathbb{P}_{\M_{0},
        \mathfrak{A}},\mathbb{P}_{\M_{\imath},\mathfrak{A}})},
    \end{align*}
    where at (1) we have lower bounded the maximum with the average, i.e., $\max\{a,b\}\ge(a+b)/2$
    for all $a,b\ge0$, and at (2) we have applied the Bretagnolle-Huber's
    inequality \cite{metelli2023towards}.

    \textbf{KL-divergence computation}
    The proof can be concluded by upper bounding the KL divergence $\text{KL}(\mathbb{P}_{\M_{0},
    \mathfrak{A}},\mathbb{P}_{\M_{\imath},\mathfrak{A}})$ as in the proof of
    Theorem 7 in \cite{domingues2021episodic}, and then summing over all the
    $\Theta(SAH)$ instances to retrieve the result.


\end{proof}

\begin{lemma}\label{lemma: lower bound packing irl} Let IRL and RFE be the
    learning problems defined as in Appendix \ref{appendix: new problems}. For
    each problem, if the set of reward functions $\R$ in input is
    $\R=\mathfrak{R}$, then any $(\epsilon,\delta)$-PAC algorithm must collect
    at least the following number of exploration episodes:
    \begin{align*}
        \tau\ge \Omega\bigg(\frac{H^3S^2A}{\epsilon^2}\bigg).
    \end{align*}
\end{lemma}
\begin{proof}
    \textbf{Instances description}
    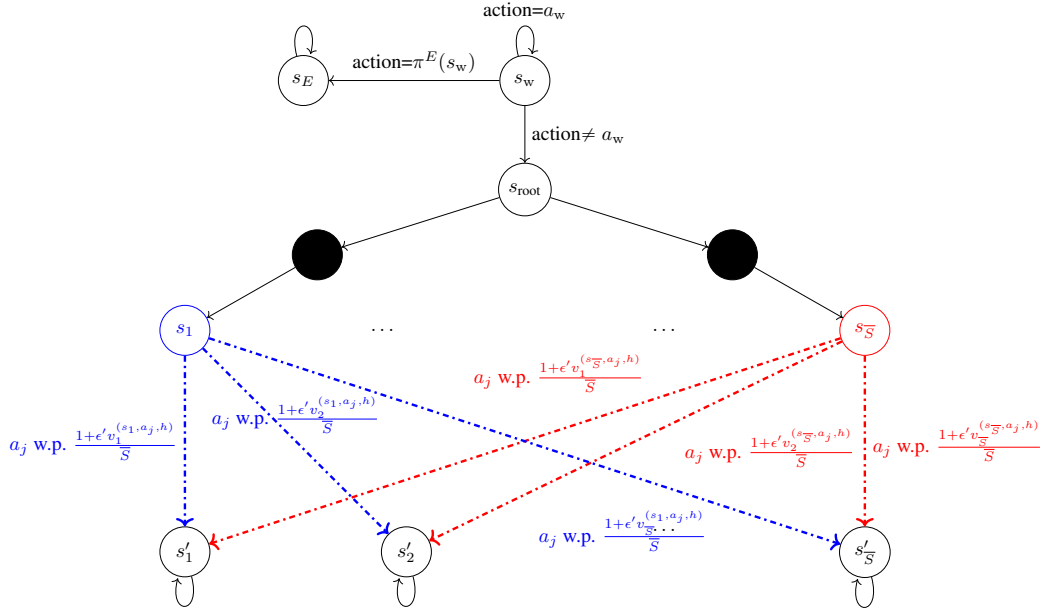
\begin{figure}[t]
        \resizebox{\textwidth}{!}{
        \begin{tikzpicture}[node distance=3.5cm]
        \node[state] (s0) {$s_{\mathrm{w}}$};
        \node[state, below=1cm of s0] (s00) {$s_{\text{root}}$};
        \node[state, left=3cm of s0] (sE) {$s_E$};

        \node[state,fill=black]    (N1)[below left=0.5cm and 3cm of s00]    {}; 
        \node[state,fill=black]    (N2)[below right=0.5cm  and 3cm of s00]    {};

        \node[state, below left of = s00, draw=none] (s1xx) {$\dots$};
        \node[state, below right of = s00, draw=none] (sSxx) {$\dots$};
        \node[state, left of = s1xx, blue] (s1) {$s_1$};
        \node[state, right of = sSxx, red] (sS) {$s_{\overline{S}}$};
        \node[state, below=3cm of s1] (s1') {$s_{1}'$};
        \node[state, right=3cm of s1'] (s2') {$s_{2}'$};
        \node[state, below=3cm of sS] (sS') {$s_{\overline{S}}'$};
        \draw (s1') edge[->, loop below] node{} (s1');
        \draw (s2') edge[->, loop below] node{} (s2');
        \draw (sS') edge[->, loop below] node{} (sS');
        
        \node[state, below of = s00, draw=none] (s1yy) {};
        \node[state, below right of = s1yy, draw=none] (s1yy) {$\dots$};
        \draw (s0) edge[->, loop above, solid] node{action=$a_{\mathrm{w}}$ } (s0);
        \draw (sE) edge[->, loop above, solid] (sE);
        \draw (s0) edge[->, solid, above] node{action=$\pie(s_{\mathrm{w}})$ } (sE);
        \draw (s0) edge[->, solid, right] node{action$\neq a_{\mathrm{w}}$} (s00);
        
        \draw (s00) edge[->, solid, above left] node{} (N1);
        \draw (s00) edge[->, solid, above right] node{} (N2);
        \draw (N1) edge[->, solid, above left] node{} (s1);
        \draw (N2) edge[->, solid, above right] node{} (sS);
        
        \draw (sS) edge[->, solid, above left, very thick, dashdotted, red] node[pos=0.3]{$a_j$ w.p. $\frac{1+\epsilon' v_{1}^{(s_{\overline{S}},a_j,h)}}{\overline{S}}$} (s1');
        \draw (sS) edge[->, solid, below right,  very thick, dashdotted, red] node[pos=0.4]{$a_j$ w.p. $\frac{1+\epsilon' v_{2}^{(s_{\overline{S}},a_j,h)}}{\overline{S}}$} (s2');
        \draw (sS) edge[->, solid, right,  very thick, dashdotted, red] node{$a_j$ w.p. $\frac{1+\epsilon' v_{\overline{S}}^{(s_{\overline{S}},a_j,h)}}{\overline{S}}$} (sS');
        
        \draw (s1) edge[->, solid, left, very thick, dashdotted, blue] node{$a_j$ w.p. $\frac{1+\epsilon' v_1^{(s_1,a_j,h)}}{\overline{S}}$} (s1');
        \draw (s1) edge[->, solid, above, very thick, dashdotted, blue] node{$a_j$ w.p. $\frac{1+\epsilon' v_2^{(s_1,a_j,h)}}{\overline{S}}$} (s2');
        \draw (s1) edge[->, solid, below left, very thick, dashdotted, blue] node[pos=0.8]{$a_j$ w.p. $\frac{1+\epsilon' v_{\overline{S}}^{(s_1,a_j,h)}}{\overline{S}}$} (sS');
        
        %
        \end{tikzpicture}	}
        \caption{Hard instances.}
        \label{fig: hard instances s squared}
        \end{figure}
        The hard instances that we use for the proof of this lemma are obtained
        by combining the hard instances in Lemma \ref{lemma: lower bound single
        reward} (i.e., the hard instances of \cite{domingues2021episodic}), with
        those in \cite{metelli2023towards}. Specifically, this construction is
        based on the intuition described in \cite{jin2020RFE} that, if we want
        to increase the sample complexity, we have to learn transitions also
        \emph{to} $\Theta(S)$ states, and not just \emph{from} $\Theta(S)$
        states. Observe the presence of state $s_E$ (only for IRL), which plays
        the same role as in the proof of Lemma \ref{lemma: lower bound single
        reward}. Any action in such state receives always reward $-1$, thus it
        is meaningless for the estimate of the (non)compatibility, which reduces
        to the estimation of the optimal performance. In this manner, the expert
        distribution $d^{p,\pie}$ does not provide additional information about
        the transition model of other portion of the state-action space.
        Therefore, in the following, we will present the lower bound
        construction as if such state did not exist.

        The hard instances are reported in Figure \ref{fig: hard instances s
        squared}. Notice that they are exactly the same instances as those
        presented in the proof of Lemma \ref{lemma: lower bound single reward},
        with the difference that, from the $\overline{S}$ leaves (differently
        from earlier, we now denote the number of leaves through $\overline{S}$
        instead of $L$), we do not reach just two states $s_{g},s_b$, but we
        reach $\Theta(S)$ absorbing states, i.e., $s_1',s_2',\dotsc,s_{\overline{S}}'$.
        The transitions from the leaves to such states is the same as in
        \cite{metelli2023towards}, and we report a description below.

        Let us introduce the set $\overline{\I} \coloneqq
        \{s_1,\dots,s_{\overline{S}}\} \times \A \times
        \{1+d,\dotsc,\overline{H}+d\}$.
        Let $\overline{\imath}\coloneqq (s_1,a_1,1+d)\in\overline{\I}$ be a specific triple of set $\I$, and denote $\I\coloneqq \overline{\I}\setminus\{\overline{\imath}\}$.
        Let us also introduce set $\mathcal{V}
        \coloneqq \{ v\in\{-1,1\}^{\overline{S}} : \sum_{j=1}^{{\overline{S}}} v_j =
        0\}$. Thanks to Lemma E.6 of \cite{metelli2023towards} (that we report
        in Lemma \ref{lemma: packing} for simplicity), we know that there exists
        a subset $\overline{V}\subseteq\V$ (of transition models) with
        cardinality at least $2^{\overline{S}/5}$ such that, for every pair
        $v,w\in\overline{\V}$ with $v\neq w$, we have that $\|v-w\|_1\ge
        \overline{S}/16$. In other words, we know that there exists a
        $\overline{S}/16$-packing of $\V$ with cardinality at least
        $2^{\overline{S}/5}$.

        Following \cite{metelli2023towards}, we denote by $\bm{v} = (v^\imath)_{\imath \in \mathcal{I}} \in \overline{\V}^{\mathcal{I}}$ 
        the generic vector of $\overline{\V}^{\mathcal{I}}$.
        Now, for any $\bm{v}\in\overline{\V}^{\mathcal{I}}$, for any triple $\overline{\jmath}\in\I$, and for some parameter
        $\epsilon'\in [0,1/2]$ to choose, we construct problem instance $\M_{\bm{v},\overline{\jmath}}$ as follows.

        First of all, we define the transition model at triple $\overline{\imath}$ as:
        \begin{align*}
        p_{h_{\overline{\imath}}}(s'_i|s_{\overline{\imath}},a_{\overline{\imath}})=\frac{1}{\overline{S}}\quad\forall i\in\dsb{\overline{S}},
        \end{align*}
        where observe that we use notation $\imath=(s_\imath,a_\imath,h_\imath)\in\overline{\I}$ to denote triples in $\overline{\I}$.
        Instead, for the generic triple $\imath\in\I$ (including triple $\jmath$), the probability
        distribution of the next state is given by:
        \begin{align*}
            p_{h_\imath}(s'_i|s_\imath,a_\imath)=\frac{1}{\overline{S}}+\frac{\epsilon'}{\overline{S}}\bm{v}^\imath_i
            \quad\forall i\in\dsb{\overline{S}},
        \end{align*}
        where $\bm{v}^\imath_i$ represents the $i$-th component of the $\imath$-th vector in $\bm{v}$.
        In words, the $i$-th component of vector $\bm{v}^\imath\in\overline{\V}$ creates a bias of
        $\epsilon'/\overline{S}$ towards the next state $s_i'$ for all
        $i\in\dsb{\overline{S}}$. Since $\bm{v}^\imath\in\overline{\V}$, then
        $p_{h_\imath}(\cdot|s_\imath,a_\imath)\in\Delta^{\dsb{\overline{S}}}$ for all $\imath\in\I$.
        
        We consider non-stationary reward functions. Specifically, all the rewards $r\in\mathfrak{R}$ that we consider assign reward 1 to both triples $\overline{\imath}$ and $\overline{\jmath}$, i.e., $r_{h_{\overline{\imath}}}(s_{\overline{\imath}},a_{\overline{\imath}})=1$ and $r_{h_{\overline{\jmath}}}(s_{\overline{\jmath}},a_{\overline{\jmath}})=1$. Next, for any other triple $(s,a,h)\in\SAH$ with state different from $s_1',s_2',\dotsc,s_{\overline{S}}'$, we assign reward 0. For states $s_1',s_2',\dotsc,s_{\overline{S}}'$, we consider state-only rewards whose value is always 0 in stages
        $[1,\overline{H}+d]$, and whose value is stationary and arbitrary afterwards. Intuitively, as in \cite{domingues2021episodic},
        forcing the reward to be 0 up $h=\overline{H}+d$ guarantees that we
        cannot obtain a higher expected return $J$ by reaching the leaves states
        earlier (i.e., by exiting from $s_{\mathrm{w}}$ before $\overline{H}$).
        
       Given the definition above, we construct the class of instances
        $\mathbb{M} \coloneqq\{ \mathcal{M}_{\bm{v},\imath}:
        \imath\in\I,\bm{v}\in\overline{\V}^\I\}$.
        Moreover, we will use the notation  $\mathcal{M}_{\bm{v} \stackrel{\imath}{\leftarrow} w,\jmath}$ to denote the instance in which we replace the $\imath$ component of $\bm{v}$, i.e., $\bm{v}^{\imath}$, with $w \in \mathcal{V}$ and $\mathcal{M}_{\bm{v} \stackrel{\imath}{\leftarrow} 0,\jmath}$ the instance in which we replace the $\imath$ component of $\bm{v}$, i.e., $v^{\imath}$, with the zero vector. Since we will always use this notation when substituting triple $\jmath$, i.e., we always use this notation in situations as $\mathcal{M}_{\bm{v} \stackrel{\jmath}{\leftarrow} w,\jmath}$, then we omit the second parameter, and write just $\mathcal{M}_{\bm{v} \stackrel{\imath}{\leftarrow} w}\coloneqq \mathcal{M}_{\bm{v} \stackrel{\imath}{\leftarrow} w,\jmath}$.

    \textbf{Distance between problems}
    Consider an arbitrary problem instance $\M_{\bm{v},\imath}\in\mathbb{M}$, for certain $\imath\in\I$ and $\bm{v}\in\overline{\V}^\I$.
    Let $r\in\mathfrak{R}$ be an arbitrary reward function that satisfies the
    constraints described earlier. Let
    $\pi_{\overline{\imath}}\in\Pi$ be the deterministic policy that brings to triple $\overline{\imath}$. Then, its expected return is:
        \begin{align*}
        J^{\pi_{\overline{\imath}}}(r;\M_{\bm{v},\imath})=1+\frac{H-\overline{H}-d}{\overline{S}}\sum\limits_{i=1}^{\overline{S}}r_i,
    \end{align*}
    where $r_i\coloneqq r_{\overline{H}+d+1}(s_i')$ for all $i\in\dsb{\overline{S}}$.
    Let policy $\pi_{\imath}\in\Pi$ be the deterministic policy that brings to triple $\imath$. Then, its expected return is:
        \begin{align*}
        J^{\pi_{\imath}}(r;\M_{\bm{v},\imath})=1+\frac{H-\overline{H}-d}{\overline{S}}\sum\limits_{i=1}^{\overline{S}}r_i+\popblue{\epsilon'\frac{(H-\overline{H}-d)}{\overline{S}}
        \sum\limits_{i=1}^{\overline{S}}\bm{v}^\imath_i r_i}.
    \end{align*}
    Finally, let policy $\pi_{\jmath}\in\Pi$ be the deterministic policy that brings to any other triple $\jmath\in\I\setminus\{\imath\}$. Then, its expected return is:
    \begin{align*}
    J^{\pi_{\jmath}}(r;\M_{\bm{v},\imath})=\popblue{0}+\frac{H-\overline{H}-d}{\overline{S}}\sum\limits_{i=1}^{\overline{S}}r_i+\epsilon'\frac{(H-\overline{H}-d)}{\overline{S}}
    \sum\limits_{i=1}^{\overline{S}}\bm{v}^\jmath_i r_i.
    \end{align*}
    It should be remarked that
    $\tuple{v,r}=\sum_{i\in\dsb{\overline{S}}}v_i r_i\in[-\overline{S},\overline{S}]$ for any $r\in\mathfrak{R}$ and $v\in\overline{\V}$, therefore, as long as:
    \begin{align}\label{eq: constraint 1 epsilon prime}
        \epsilon'(H-\overline{H}-d)<1-\epsilon'(H-\overline{H}-d)-\epsilon\iff
        \epsilon'<\frac{1-\epsilon}{2(H-\overline{H}-d)},
    \end{align}
    then any policy $\pi_{\jmath}$ is cannot be $\epsilon$-optimal in problem $\M_{\bm{v},\imath}$, in which, thus, the optimal policy shall be searched for between $\pi_{\overline{\imath}}$ and $\pi_{\imath}$.

    Now, consider an arbitrary pair $v,w\in\overline{\V}$ such that $v\neq w$,
    and an arbitrary triple $\imath\in\I$ and vector $\bm\in\overline{\V}^\I$. We now compare problem instances $\mathcal{M}_{\bm{v} \stackrel{\imath}{\leftarrow} v}$ and $\mathcal{M}_{\bm{v} \stackrel{\imath}{\leftarrow} w}$. 
    Among all possible reward functions that satisfy the definition provided in the construction of the hard instances, we find reward $r'$ such
    that, in every component $i\in\dsb{\overline{S}}$, satisfies:
    \begin{align*}
        r_i'=\begin{cases}
            +1\quad\text{if }v_i=+1\wedge w_i=-1\\
            -1\quad\text{if }v_i=-1\wedge w_i=+1\\
            0\quad\text{if }v_i=w_i
        \end{cases}.
    \end{align*}
    For what we have seen before about class $\overline{\V}$, we know that
    $\|v-w\|_1=\sum_{i\in\dsb{\overline{S}}}|v_i-w_i|\ge\overline{S}/16$, thus, since
    $v,w\in\V$, i.e., their components belong to $\{-1,+1\}$, we know that there
    are at least $\overline{S}/32$ components of $v,w$ that differ from each
    other. By using reward $r'$, we have that:
    \begin{align*}
        &\sum\limits_{i=1}^{\overline{S}}\popblue{v_i} r_i'\ge \frac{\overline{S}}{32}\ge 0,\\
        &\sum\limits_{i=1}^{\overline{S}}\popblue{w_i} r_i'\le -\frac{\overline{S}}{32}\le 0.
    \end{align*}
    As a consequence, the expected returns of policies $\pi_{\overline{\imath}}$ and $\pi_{\imath}$ in problems $\mathcal{M}_{\bm{v} \stackrel{\imath}{\leftarrow} v}$ and $\mathcal{M}_{\bm{v} \stackrel{\imath}{\leftarrow} w}$ are:
    \begin{align*}
        &J^{\popblue{\pi_{\overline{\imath}}}}(r';\mathcal{M}_{\bm{v} \stackrel{\imath}{\leftarrow} \popblue{v}})=J^{\popblue{\pi_{\overline{\imath}}}}(r';\mathcal{M}_{\bm{v} \stackrel{\imath}{\leftarrow} \popblue{w}})=1+\frac{H-\overline{H}-d}{\overline{S}}\sum\limits_{i=1}^{\overline{S}}r_i',\\
        &J^{\popblue{\pi_{\imath}}}(r';\mathcal{M}_{\bm{v} \stackrel{\imath}{\leftarrow} \popblue{v}})\popblue{\ge} 1+\frac{H-\overline{H}-d}{\overline{S}}\sum\limits_{i=1}^{\overline{S}}r_i'\popblue{+}\epsilon'\frac{(H-\overline{H}-d)}{32},\\
        &J^{\popblue{\pi_{\imath}}}(r';\mathcal{M}_{\bm{v} \stackrel{\imath}{\leftarrow} \popblue{w}})\popblue{\le} 1+\frac{H-\overline{H}-d}{\overline{S}}\sum\limits_{i=1}^{\overline{S}}r_i'\popblue{-}\epsilon'\frac{(H-\overline{H}-d)}{32},
    \end{align*}
    from which we infer that:
    \begin{align*}
        J^{\popblue{\pi_{\imath}}}(r';\mathcal{M}_{\bm{v} \stackrel{\imath}{\leftarrow} \popblue{v}})\ge J^{\popblue{\pi_{\overline{\imath}}}}(r';\mathcal{M}_{\bm{v} \stackrel{\imath}{\leftarrow} \popblue{v}})=J^{\popblue{\pi_{\overline{\imath}}}}(r';\mathcal{M}_{\bm{v} \stackrel{\imath}{\leftarrow} \popblue{w}})\ge J^{\popblue{\pi_{\imath}}}(r';\mathcal{M}_{\bm{v} \stackrel{\imath}{\leftarrow} \popblue{w}}).
    \end{align*}
    Now, let us choose $\epsilon'> 64\epsilon/(H-\overline{H}-d)$. To satisfy
    also the constraint in Equation \eqref{eq: constraint 1 epsilon prime}, we
    can roughly assume $\epsilon<1/256$ and set $\epsilon'=
    65\epsilon/(H-\overline{H}-d)$. Thanks to this choice, observe that:
    \begin{align*}
        &J^{\popblue{\pi_{\imath}}}(r';\mathcal{M}_{\bm{v} \stackrel{\imath}{\leftarrow} \popblue{v}})> J^{\popblue{\pi_{\overline{\imath}}}}(r';\mathcal{M}_{\bm{v} \stackrel{\imath}{\leftarrow} \popblue{v}})+2\epsilon,\\
        &J^{\popblue{\pi_{\overline{\imath}}}}(r';\mathcal{M}_{\bm{v} \stackrel{\imath}{\leftarrow} \popblue{w}})> J^{\popblue{\pi_{\imath}}}(r';\mathcal{M}_{\bm{v} \stackrel{\imath}{\leftarrow} \popblue{w}})+2\epsilon.
    \end{align*}
    In words, policy $\pi_\imath$ is optimal in problem $\mathcal{M}_{\bm{v} \stackrel{\imath}{\leftarrow} v}$, and policy $\pi_{\overline{\imath}}$ is worse than $2\epsilon$-suboptimal in such problem. In addition, observe that policy $\pi_{\overline{\imath}}$ is optimal in problem $\mathcal{M}_{\bm{v} \stackrel{\imath}{\leftarrow} w}$, and policy $\pi_{\imath}$ is worse than $2\epsilon$-suboptimal in such problem.
    We stress that any stochastic policy in-between $\pi_{\imath}$ and $\pi_{\overline{\imath}}$ cannot be $\epsilon$-optimal for both problems.

    To sum up, for the choice of $\epsilon'$ made earlier, for arbitrary pairs of problems $\mathcal{M}_{\bm{v} \stackrel{\imath}{\leftarrow} v}$ and $\mathcal{M}_{\bm{v} \stackrel{\imath}{\leftarrow} w}$, we have seen that there exist rewards in $\mathfrak{R}$ for which a policy $\epsilon$-optimal for problem $\M_{\imath,v}$ is not
    $\epsilon$-optimal for problem $\M_{\imath,w}$, and vice versa.

    \textbf{Identifying the underlying problem: RFE.}~~
    We consider first RFE, and then IRL.

Let us consider an $(\epsilon,\delta)$-correct algorithm $\mathfrak{A}$ for RFE,
that outputs, for any reward function $r\in\mathfrak{R}$, a policy
$\widehat{\pi}_r$. For simplicity, we consider as output of Algorithm
$\mathfrak{A}$ a function $\widehat{\pi}:\mathfrak{R}\to \Pi$, that takes in
input a reward and outputs a policy.

For any $\imath \in \mathcal{I}$ and $\bm{v}\in\overline{\V}^\I$, we can lower bound the error probability as:
\begin{align*}
    \delta&\ge \sup\limits_{\text{all problem instances }\M}
    \mathbb{P}_{\M,\mathfrak{A}}\bigg(
    \sup\limits_{r\in\mathfrak{R}}    
    J^*_\M(r)-J^{\widehat{\pi}_r}_\M(r)\ge\epsilon\bigg)\\
    &\markref{(1)}{\ge} \sup\limits_{\popblue{\M\in\mathbb{M}}}
    \mathbb{P}_{\M,\mathfrak{A}}\bigg(
    \sup\limits_{r\in\mathfrak{R}}    
    J^*_\M(r)-J^{\widehat{\pi}_r}_\M(r)\ge\epsilon\bigg)\\
    &\markref{(2)}{\ge} \max\limits_{\popblue{w\in\overline{\V}}}
    \mathbb{P}_{\popblue{\mathcal{M}_{\bm{v} \stackrel{\imath}{\leftarrow} w}},\mathfrak{A}}\bigg(\sup\limits_{r\in\mathfrak{R}}
    J^*_{\popblue{\mathcal{M}_{\bm{v} \stackrel{\imath}{\leftarrow} w}}}(r)-J^{\widehat{\pi}_r}_{\popblue{\mathcal{M}_{\bm{v} \stackrel{\imath}{\leftarrow} w}}}(r)\ge\epsilon\bigg),
\end{align*}
where at (1) we have lower bounded by replacing all possible RFE problem instances with
problem instances in $\mathbb{M}$, and at (2) we have lower bounded by replacing
all instances in $\mathbb{M}$ with just instances $\{\mathcal{M}_{\bm{v} \stackrel{\imath}{\leftarrow} w}:
w\in\overline{\V}\}$ for the fixed triple $\imath$ and vector $\bm{v}$.

For every $\imath\in\I$ and $\bm{v}\in\overline{\V}^\I$, we define the \emph{identification function}
$\Psi_{\imath,\bm{v}}$ as the index of the problem $w\in\overline{\V}$ ``recognized'' by algorithm
$\mathfrak{A}$. In symbols:
\begin{align*}
    \Psi_{\imath,\bm{v}}\coloneqq \argmin\limits_{w\in\overline{\V}}\sup\limits_{r\in\mathfrak{R}}
    J^*_{\mathcal{M}_{\bm{v} \stackrel{\imath}{\leftarrow} w}}(r)-J^{\widehat{\pi}_r}_{\mathcal{M}_{\bm{v} \stackrel{\imath}{\leftarrow} w}}(r).
\end{align*}
In words, given estimate $\widehat{\pi}:\mathfrak{R}\to\Pi$ returned by algorithm $\mathfrak{A}$, the
identification function $\Psi_{\imath,\bm{v}}$ returns the problem in $\{\mathcal{M}_{\bm{v} \stackrel{\imath}{\leftarrow} w}:
w\in\overline{\V}\}$ whose solution $\pi:\mathfrak{R}\to\Pi$ is closest to the estimate $\widehat{\pi}$.
For what we have seen in the previous paragraph, for any $v,w\in\overline{\V}$
with $v\neq w$, for any fixed $\imath\in\I$ and $\bm{v}\in\overline{\V}^\I$, there exists a reward function
$r'\in \mathfrak{R}$ such that no policy can have expected
utility
$\epsilon$-close to the optimal expected utility of both problems
$\mathcal{M}_{\bm{v} \stackrel{\imath}{\leftarrow} v}$ and $\mathcal{M}_{\bm{v} \stackrel{\imath}{\leftarrow} w}$. Therefore, for
$w\in\overline{\V}$, we have the following inclusion of events:
\begin{align*}
    \{\Psi_{\imath,\bm{v}}\neq w\}\subseteq \Big\{\sup\limits_{r\in\mathfrak{R}}
    J^*_{\mathcal{M}_{\bm{v} \stackrel{\imath}{\leftarrow} w}}(r)-J^{\widehat{\pi}_r}_{\mathcal{M}_{\bm{v} \stackrel{\imath}{\leftarrow} w}}(r)>\epsilon\Big\}.
\end{align*}
We can continue to lower bound the probability as:
\begin{align*}
    \max\limits_{w\in\overline{\V}} 
    \mathbb{P}_{\mathcal{M}_{\bm{v} \stackrel{\imath}{\leftarrow} w},\mathfrak{A}}& \bigg(\sup\limits_{r\in\mathfrak{R}}
    J^*_{\mathcal{M}_{\bm{v} \stackrel{\imath}{\leftarrow} w}}(r)-J^{\widehat{\pi}_r}_{\mathcal{M}_{\bm{v} \stackrel{\imath}{\leftarrow} w}}(r)\ge\epsilon\bigg)
     \markref{(3)}{\ge} \frac{1}{|\overline{\V}|}\sum\limits_{w\in\overline{\V}}
    \mathbb{P}_{\mathcal{M}_{\bm{v} \stackrel{\imath}{\leftarrow} w},
    \mathfrak{A}}\big(\Psi_{\imath,\bm{v}}\neq w\big)\\
    &\markref{(4)}{\ge} 1-\frac{1}{\log|\overline{\V}|}\bigg(
        \frac{1}{|\overline{\V}|}\sum\limits_{w\in\overline{\V}}
        \text{KL}(\mathbb{P}_{\mathcal{M}_{\bm{v} \stackrel{\imath}{\leftarrow} w},
    \mathfrak{A}},\mathbb{P}_{\mathcal{M}_{\bm{v} \stackrel{\imath}{\leftarrow} 0},\mathfrak{A}})-\log 2
    \bigg),
\end{align*}
where at (3) we have lower bounded the maximum over $\overline{\V}$ with the
average, and at (4) we have applied, similary to \cite{metelli2023towards}, the
Fano's inequality, reported in Theorem \ref{thr:Fano} for simplicity.

\textbf{Identifying the underlying problem: IRL.}~~
For IRL,  it is possible to carry out a similar derivation. However, we remark
that, now, the error is measured based on the expected utilities, and not on the
policies.

Let us consider an $(\epsilon,\delta)$-correct algorithm $\mathfrak{A}$ for IRL,
that outputs, for any reward function $r\in\mathfrak{R}$, a utility
$\widehat{J}_r$. For simplicity, we consider as output of Algorithm
$\mathfrak{A}$ a function $\widehat{J}:\mathfrak{R}\to  \RR$, that takes in
input a reward and outputs a utility.

For any $\imath \in \mathcal{I}$ and $\bm{v}\in\overline{\V}^\I$, we can lower bound the error probability as:
\begin{align*}
    \delta&\ge \sup\limits_{\text{all problem instances }\M}
    \mathbb{P}_{\M,\mathfrak{A}}\bigg(
    \sup\limits_{r\in\mathfrak{R}}    
    \Big|J^*_\M(r)-\widehat{J}_r\Big|\ge\epsilon\bigg)\\
    &\ge \sup\limits_{\popblue{\M\in\mathbb{M}}}
    \mathbb{P}_{\M,\mathfrak{A}}\bigg(
    \sup\limits_{r\in\mathfrak{R}}    
    \Big|J^*_\M(r)-\widehat{J}_r\Big|\ge\epsilon\bigg)\\
    &\ge \max\limits_{\popblue{w\in\overline{\V}}}
    \mathbb{P}_{\popblue{\mathcal{M}_{\bm{v} \stackrel{\imath}{\leftarrow} w}},\mathfrak{A}}\bigg(\sup\limits_{r\in\mathfrak{R}}
    \Big|J^*_{\popblue{\mathcal{M}_{\bm{v} \stackrel{\imath}{\leftarrow} w}}}(r)-\widehat{J}_r\Big|\ge\epsilon\bigg).
\end{align*}
For any $\imath\in\I$ and $\bm{v}\in\overline{\V}^\I$, we define an identification function $\Psi_{\imath,\bm{v}}$ as:
\begin{align*}
    \Psi_{\imath,\bm{v}}\coloneqq \argmin\limits_{w\in\overline{\V}}\sup\limits_{r\in\mathfrak{R}}
    \Big|J^*_{\mathcal{M}_{\bm{v} \stackrel{\imath}{\leftarrow} w}}(r)-\widehat{J}_r\Big|,
\end{align*}
and by a reasoning analogous to that for RFE, we can continue to lower bounding
as:
\begin{align}\label{eq: lower bound temp fano}
    \max\limits_{w\in\overline{\V}}
    \mathbb{P}_{\mathcal{M}_{\bm{v} \stackrel{\imath}{\leftarrow} w},\mathfrak{A}}& \bigg(\sup\limits_{r\in\mathfrak{R}}
    \Big|J^*_{\mathcal{M}_{\bm{v} \stackrel{\imath}{\leftarrow} w}}(r)-\widehat{J}_r\Big|\ge\epsilon\bigg)
     \ge \frac{1}{|\overline{\V}|}\sum\limits_{w\in\overline{\V}}
    \mathbb{P}_{\mathcal{M}_{\bm{v} \stackrel{\imath}{\leftarrow} w},
    \mathfrak{A}}\big(\Psi_{\imath,\bm{v}}\neq w\big)\nonumber\\
    &\ge 1-\frac{1}{\log|\overline{\V}|}\bigg(
        \frac{1}{|\overline{\V}|}\sum\limits_{w\in\overline{\V}}
        \text{KL}(\mathbb{P}_{\mathcal{M}_{\bm{v} \stackrel{\imath}{\leftarrow} w},
    \mathfrak{A}},\mathbb{P}_{\mathcal{M}_{\bm{v} \stackrel{\imath}{\leftarrow} 0},\mathfrak{A}})-\log 2
    \bigg),
\end{align}
which represents the same lower bound obtained also for RFE.

\textbf{KL-divergence computation} The following derivation is analogous to that
of \cite{metelli2023towards}. To bound the KL-divergence term, for any
$\imath\in\I$, we can write:
\begin{align*}
    \text{KL}(\mathbb{P}_{\mathcal{M}_{\bm{v} \stackrel{\imath}{\leftarrow} w},
    \mathfrak{A}},\mathbb{P}_{\mathcal{M}_{\bm{v} \stackrel{\imath}{\leftarrow} 0},\mathfrak{A}})&\markref{(1)}{=}
    \E_{\mathcal{M}_{\bm{v} \stackrel{\imath}{\leftarrow} w},\mathfrak{A}}\big[N^\tau_{h_\imath}(s_\imath,a_\imath)\big]
    \text{KL}(p^{\mathcal{M}_{\bm{v} \stackrel{\imath}{\leftarrow} w}}_{h_\imath}(\cdot|s_\imath,a_\imath),
    p^{\mathcal{M}_{\bm{v} \stackrel{\imath}{\leftarrow} 0}}_{h_\imath}(\cdot|s_\imath,a_\imath))\\
    &\markref{(2)}{\le}
    2(\epsilon')^2 \E_{\mathcal{M}_{\bm{v} \stackrel{\imath}{\leftarrow} w},\mathfrak{A}}\big[N^\tau_{h_\imath}(s_\imath,a_\imath)\big],
\end{align*}
where at (1) we have applied Lemma \ref{lemma: lemma 5 domingues}, and at (2) we
have applied Lemma \ref{lemma: KLBound} (having observed that the transition
models differ in $\imath$ and defined $N^\tau_{h_\imath}(s_\imath,a_\imath) =
\sum_{t=1}^\tau \indic{(s_t,a_t,h_t) = (s_\imath,a_\imath,h_\imath)}$).

Plugging into Equation \eqref{eq: lower bound temp fano}, we get:
\begin{align*}
    \delta \ge \frac{1}{|\mathcal{\overline{V}}|}
    \sum_{w \in \mathcal{\overline{V}}}
    \mathbb{P}_{\mathcal{M}_{\bm{v} \stackrel{\imath}{\leftarrow} w},
    \mathfrak{A}} \left( \Psi_{\imath,\bm{v}} \neq w \right) 
    \implies \frac{1}{|\mathcal{\overline{V}}|}
    \sum_{w \in \overline{\mathcal{V}}}
    \E_{\mathcal{M}_{\bm{v} \stackrel{\imath}{\leftarrow} w},\mathfrak{A}}
    \left[ N^\tau_{h_\imath}(s_\imath,a_\imath) \right] \ge \frac{(1-\delta)
    \log|\overline{\mathcal{V}}| - \log 2}{2(\epsilon')^2}.
    \end{align*}

Notice that, since $|\overline{\mathcal{V}}|=\Theta(e^S)$ and $\epsilon'=\Theta(\epsilon/H)$, then this bound is in the order of $\Omega(\frac{H^2S}{\epsilon^2})$. To get the additional $\Omega(SAH)$ dependence, we can make the same observation as in \cite{metelli2023towards}, i.e., that ince the derivation is carried out for every $\imath \in \mathcal{I}$ and $\bm{v} \in \overline{\mathcal{V}}^{\mathcal{I}}$, we can perform the summation over $\imath$ and the average over $\bm{v}$. By noticing that we get a guarantee on a mean under the uniform distribution of the instances of the sample complexity, we realize that there must exist one $\bm{v}^{\text{hard}} \in \mathcal{\overline{V}}$ for which it holds the desired $\Omega\Big(\frac{H^3S^2A}{\epsilon^2}\Big)$ dependency.

\end{proof}

\subsubsection{Technical Tools}

We report here some results from other works. The notation adopted is the same
as the original works.

\begin{lemma}[Lemma E.6 of \cite{metelli2023towards}]\label{lemma: packing} Let
	$\V = \{v \in \{-1,1\}^D : \sum_{j=1}^D v_j = 0\}$. Then, the
	$\frac{D}{16}$-packing number of $\mathcal{V}$ w.r.t. the metric
	$d(v,v')=\sum_{j=1}^D |v_j-v'_j|$ is lower bounded by $2^\frac{D}{5}$.
\end{lemma}

\begin{thr}(\emph{Theorem E.2 of \cite{metelli2023towards}})\label{thr:Fano} Let
    $\mathbb{P}_0,\mathbb{P}_1,\dots,\mathbb{P}_M$ be probability measures on
    the same measurable space $(\Omega, \mathcal{F})$, and let
    $\mathcal{A}_1,\dots,\mathcal{A}_M \in \mathcal{F}$ be a partition of
    $\Omega$. Then,
    \begin{align*}
        \frac{1}{M} \sum_{i=1}^M \mathbb{P}_i(\mathcal{A}_i^c) \ge 1 - \frac{ \frac{1}{M} \sum_{i=1}^M D_{\text{KL}}(\mathbb{P}_i,\mathbb{P}_0)  - \log 2}{\log M},
    \end{align*}
    where $\mathcal{A}^c = \Omega \setminus \mathcal{A}$ is the complement of $\mathcal{A}$.
\end{thr}

\begin{lemma}[Lemma E.4 of \cite{metelli2023towards}]\label{lemma: KLBound} Let $\epsilon \in [0,1/2]$ and $\mathbf{v}
    \in \{-\epsilon,\epsilon\}^D$ such that $\sum_{i=1}^d v_i = 0$. Consider the
    two categorical distributions  $\mathbb{P} = \left(
    \frac{1}{D},\frac{1}{D},\dots, \frac{1}{D}\right)$ and $\mathbb{P} = \left(
    \frac{1+v_1}{D},\frac{1+v_2}{D},\dots, \frac{1+v_D}{D}\right)$. Then, it
    holds that:
    \begin{align*}
     D_{\text{KL}}(\mathbb{P},\mathbb{Q}) \le 2 \epsilon^2 \qquad \text{and} \qquad 
     D_{\text{KL}}(\mathbb{Q},\mathbb{P}) \le 2 \epsilon^2.
    \end{align*}
    \end{lemma}

\begin{lemma}[Lemma 5 of \cite{domingues2021episodic}]\label{lemma: lemma 5
    domingues} Let $\M$ and $\M'$ be two MDPs that are identical except for
    their transition probabilities, denoted by $p_h$ and $p_h'$, respectively.
    Assume that we have $\forall (s a)$, $p_h(\cdot|s,a) \ll p_h'(\cdot|s,a)$.
    Then, for any stopping time $\tau$ with respect to $(\F_H^t)_{t\ge 1}$ that
    satisfies $\mathbb{P}_{\M}{\tau < \infty} =1$,
    \begin{align*}
    \text{KL}\Big(\P_\M^{I_H^\tau}, \P_{\M'}^{I_H^\tau}\Big)
    = \sum_{s \in \S}\sum_{a \in \A}\sum_{h\in\dsb{H-1}}
    \E_\M \big[N_{h,s,a}^\tau\big] \text{KL}\Big( p_h(\cdot|s,a), p_h'(\cdot|s,a)\Big),
    \end{align*}
    where $N_{h,s,a}^\tau \coloneqq \sum_{t=1}^\tau \indic{(S_h^t, A_h^t) = (s,a)}$
    and $I_H^\tau: \Omega \to \bigcup_{t\geq 1} \I_H^t: \omega \mapsto
    I_H^{\tau(\omega)}(\omega) $ is the random vector representing the history
    up to episode $\tau$.
\end{lemma}

\section{A Use Case for Objective-Free Exploration (OFE)}\label{section: more on ofe}

Consider the following setting. You are given a certain MDP without reward
$\M=\tuple{\S,\A,H, d_0,p}$, in which you do not know neither $ d_0$ nor $p$.
Your job is to explore the environment to collect samples that allow you to
construct estimates $\widehat{d}_0\approx d_0$ and $\widehat{p}\approx p$,
that will be subsequently used to perform a task in a given class $\mathscr{F}$ in an
$(\epsilon,\delta)$-correct manner. Of course the number of samples should be
as small as possible. How do you explore? It depends on which problems are
contained in class $\mathscr{F}$.

A use case for OFE is the following.

\begin{example}
    Assume that we are given a single fixed environment (for instance, a
warehouse), in which there are many tasks to do (e.g., labelling objects,
putting stuff on the shelves, bringing products from one side to the other), and
assume (it is reasonable) that it is desirable to have one robot for each task.
To teach these robots how to behave, we decide to use RL.
Since all the robots work in the same environment (warehouse), then the
(unknown) transition model is the same. For this reason, an efficient
exploration (potentially through RFE) is meaningful. However, we realize that
some tasks are difficult to design (i.e., the rewards of such tasks). For
these tasks, we prefer to use a human expert to exhibit demonstrations, and then
use ReL (in particular, IRL), to learn the reward, that will be subsequently
used for AL. To perform IRL nicely, the samples collected at the beginning shall
be used. To sum up, we might be interested in performing multiple RL and IRL
tasks in the same unknown MDP, and, for efficiency reasons, our exploration of
the environment has to be performed only once (before) being given the tasks to solve.
\end{example}

\end{document}